\documentclass[10pt]{article} % For LaTeX2e
\usepackage[preprint]{tmlr}
\usepackage{microtype}
\usepackage{graphicx}
\usepackage{booktabs} 
\usepackage{hyperref}

\usepackage{amsmath}
\usepackage{amssymb}
\usepackage{mathtools}
\usepackage{amsthm}

\theoremstyle{plain}
\newtheorem{theorem}{Theorem}[section]

\theoremstyle{definition}
\newtheorem{definition}[theorem]{Definition}

\theoremstyle{remark}
\newtheorem{remark}[theorem]{Remark}

\usepackage{csquotes}
\usepackage{bbm}
\usepackage{enumitem}
\usepackage[caption=false,font=footnotesize]{subfig}
\usepackage{multirow}
\usepackage[normalem]{ulem}

\usepackage{tabularx}

\usepackage{natbib}
\usepackage{algorithm}
\usepackage{algorithmicx}
\usepackage{algpseudocode}
\makeatletter
\def\BState{\State\hskip-\ALG@thistlm}
\makeatother
\algnewcommand\algorithmicforeach{\textbf{for each}}
\algdef{S}[FOR]{ForEach}[1]{\algorithmicforeach\ #1\ \algorithmicdo}

\makeatletter
\algnewcommand{\LineComment}[1]{\Statex \hskip\ALG@thistlm #1}
\makeatother
\floatname{algorithm}{Procedure}

\catcode`,\active

\catcode`\,12

\DeclareMathOperator*{\argmax}{\arg\max}

\newcommand{\bs}[1]{\boldsymbol{#1}}

\usepackage{glossaries}
\let\oldgls=\gls\renewcommand{\gls}[1]{{\hypersetup{linkbordercolor=black, linkcolor=black}\oldgls{#1}}}
\let\oldGls=\Gls\renewcommand{\Gls}[1]{{\hypersetup{linkbordercolor=black, linkcolor=black}\oldGls{#1}}}
\let\oldglspl=\glspl\renewcommand{\glspl}[1]{{\hypersetup{linkbordercolor=black, linkcolor=black}\oldglspl{#1}}}
\let\oldGlspl=\Glspl\renewcommand{\Glspl}[1]{{\hypersetup{linkbordercolor=black, linkcolor=black}\oldGlspl{#1}}}

\newacronym{fwer}{FWER}{family-wise error rate}
\newacronym{fdr}{FDR}{false discovery rate}
\newacronym{tdr}{TDR}{true discovery rate}
\newacronym{bh}{BH}{Benjamini-Hochberg}
\newacronym{iid}{iid}{independent and identically distributed}
\newacronym{prds}{PRDS}{positive regression dependent on a subset}
\newacronym{ocsvm}{OCSVM}{one class support vector machines}
\newacronym{svc}{SVC}{support vector classification}
\newacronym{lr}{LR}{logistic regression}
\newacronym{ood}{OOD}{\emph{out-of-distribution}}
\newacronym{cdf}{CDF}{cumulative distribution function}
\newacronym{mcp}{MCP}{\emph{model context protocol}}
\newacronym{mlp}{MLP}{multilayer perceptron}
\newacronym{pu}{PU}{\emph{positive unlabelled}}
\newacronym{gbc}{GBC}{gradient boosting classifier}
\newacronym{auc}{AUC}{area under the curve}
\newacronym{if}{IF}{isolation forest}
\newacronym{ae}{AE}{autoencoder}

\glsdisablehyper

\title{Conformal Data Contamination Tests for In-distribution Data Acquisition}

\author{\name Martin V. Vejling \email mvv@es.aau.dk \\
      \addr Department of Electronic Systems\\
      Aalborg University
      \AND
      \name Shashi Raj Pandey \email srp@es.aau.dk \\
      \addr Department of Electronic Systems\\
      Aalborg University
      \AND
      \name Christophe A. N. Biscio \email christophe@math.aau.dk \\
      \addr Department of Mathematical Sciences\\
      Aalborg University
      \AND
      \name Petar Popovski \email petarp@es.aau.dk \\
      \addr Department of Electronic Systems\\
      Aalborg University
}

\def\month{10}  % Insert correct month for camera-ready version
\def\year{2026} % Insert correct year for camera-ready version
\def\openreview{\url{https://openreview.net/forum?id=7HNbXnhtfR}} % Insert correct link to OpenReview for camera-ready version

\begin{document}

%\begin{bibunit}[tmlr]

\maketitle

\begin{abstract}
The amount of quality data in many machine learning tasks is limited to what is available locally to data owners. The set of quality data can be expanded through trading or sharing with external data agents. However, external data may be contaminated or introduce undesirable sample diversity which can degrade performance of personalized machine learning tasks, as in diagnosis of a rare disease or recommendation systems. Therefore, data buyers need quality guarantees prior to data acquisition. Previous works primarily rely on distributional assumptions about data from different agents, relegating quality checks to post-hoc steps involving costly data valuation procedures. We propose a distribution-free, contamination-aware data {acquisition} framework that, by inspecting only a small volume of data, identifies external data agents whose data is most valuable for model personalization. To achieve this, we introduce novel two-sample testing procedures, preceding full data acquisition, grounded in rigorous theoretical foundations for conformal outlier detection, to determine whether an agent’s data exceeds a contamination threshold. The proposed tests, termed \emph{conformal data contamination tests}, remain valid under arbitrary contamination levels {and the novel Storey-type test provably enables finite-sample} false discovery rate control via the Benjamini-Hochberg procedure. Empirical evaluations across diverse collaborative learning scenarios demonstrate the robustness and effectiveness of our approach. Overall, the conformal data contamination test distinguishes itself as a generic procedure for aggregating data with statistically rigorous quality guarantees.
\end{abstract}

\section{Introduction}\label{sec:intro}
In many real-world machine learning applications the amount of quality data for training is limited to what is locally available to the data owners.
Collaborative learning techniques have shown some potential by allowing multiple participants to contribute data and jointly train a model in a privacy-preserving manner~\citep{Mcmahan:2017:Communication}, however, relying solely on distributed and diverse private data from collaborators in the learning process significantly affects \emph{personalization}~\citep{blum2017collaborative}. Unlike federated learning \citep{Mcmahan:2017:Communication}, several practical scenarios consider a data owner having specific learning goals, and not being interested in having a \emph{well-generalized} model that works for the rest of the participants. An example would be a small, energy-efficient model that is specialized for a certain type of task and there is no interest to perform well on other tasks~\citep{belcak2025small}.
This indicates that \emph{not all} data, but only quality data, meaning data with \emph{specific attributes}, are relevant for the learning agent, and motivates this work to resolve challenges of quality data acquisition.

We distinguish data agents with distribution mismatch into three categories: (i) honest data agents with unintentionally biased data, (ii) dishonest data agents selling artificial data, and (iii) malicious data agents generating adversarial data \citep{Szegedy2014:Intriguing,Zhang2023:Detecting}. For the first case, an example is a hospital interested in training a model to detect a rare disease based on an X-ray image of a patient. Since the disease is rare, the hospital may not have sufficient local data to train an accurate prediction model, but when buying data from other hospitals, or a \emph{data marketplace}~\citep{Fernandez2023:Data}, it is important to ensure similar data conditions, as different X-ray machines and softwares may have been used to take the measurements \citep{Alqaness2024:Chest}. An example of the second case is fake social media accounts sold on a data marketplace as genuine users. For the third case, an example is recommendation systems with fake review injection \citep{Nawara2025:Shilling}.
% These examples outline the need for having principled data trading or sharing protocols tailored to specific learning goals.
These examples outline the need for having principled data acquisition/sharing protocols with quality guarantees.

Quality data acquisition approaches have motivated novel data trading or sharing mechanisms, such as \emph{data markets}, where the data agent can prioritize access to relevant data for better personalization \citep{Fernandez2020:Data}. In a nutshell, we assume that \emph{in-distribution} implies relevant data, similar to the training samples, for improving personalization. In contrast, \gls{ood} indicates outliers and unseen samples, including data with potential contamination, which may hold value for better generalization.
A participant's benefit from the data of other participants hinges on careful data selection per its learning goal, as outlined in the earlier real-world examples, leading to the development of efficient data valuation techniques, such as Shapley value~\citep{Ghorbani2019:Shapley} and its variants, amongst others.
Unfortunately, these approaches are still limited: first, they are model-specific and computationally expensive;
second, they are often used \emph{post-hoc}, i.e., after collecting the data, which might be contaminated and of less value; and third, detecting the optimal set of collaboration partners without making any distributional assumptions on their data is challenging.
This raises some fundamental questions in developing data sharing mechanisms prior to training. Furthermore, in a networked system, any of such mechanisms should be compatible with emerging peer-to-peer collaboration protocols between decentralized data sources and AI-powered tools, such as \gls{mcp} \citep{Hou2025:Model}. From the perspective of \gls{mcp}, the data agents are the local data sources and an \gls{mcp} host then wants to access the data from the local data sources which are relevant for personalization. In this premise, the focus of this work is to develop an agile data trading framework between data agents that identifies valuable training data, \emph{without making any distributional assumption}, through the formalization of guarantees to data contamination tests.

\textbf{Data contamination model:} Consider that a local data agent observes data $\mathcal{D}_0 = \{(X_i, Y_i)\}_{i=1}^n$ of $n$ input-output pairs, where $X_i \in \mathbb{R}^d$ is the observed feature and $Y_i$ is the response with $Z_i = (X_i, Y_i) \sim P_0$ for an unknown local distribution $P_0$. In the data sharing platform, there are $K$ other data agents with each their own local data $\mathcal{D}_k = \{(X_i^k, Y_i^k)\}_{i=1}^{m_k}$ of $m_k$ input-output pairs with $Z_i^k = (X_i^k, Y_i^k) \sim \Tilde{P}_k$ for unknown local distributions $\Tilde{P}_k$. Each of the unknown local distributions can be decomposed as $\Tilde{P}_k = (1-\pi_k)P_0 + \pi_k P_k$ where $P_k$ is a proper outlier distribution and $\pi_k \in [0, 1]$ is the contamination factor \citep{Blanchard2010:Semi}. If the learning goal is better personalization, as is the focus of our work, finding the best data agent(s) for data acquisition naturally boils down to learning who has the least contaminated data.

\textbf{General data sharing procedure and challenges:} \emph{First}, upon the availability of data from $K$ data agents in the data sharing platform, the data requester, say the data agent $0$, must decide from which data agent(s) to acquire more data to improve personalization \citep{Lee2018:Simple}.
Once the first problem of who to pair with is resolved, the data sharing procedure is executed in rounds, where, in each round, a batch of data is acquired and used to decide the data acquisition policy in the next round; the process ends in either of the following cases: (i) when the data has been exhausted, (ii) the data purchasing budget has been depleted, or (iii) the data agent $0$ decides that acquiring more data is not necessary (beneficial). \emph{Second}, after the termination of data sharing, data agent $0$ uses a data subset selection technique, for instance a complex data valuation technique \citep{Koh2017:Understanding,Ghorbani2019:Shapley,Yoon2020:Reinforcement}, or a simpler outlier detection method \citep{Scholkopf2001:Estimating,Hawkins2002:Outlier,Ruff2018:Deep}, to filter out \gls{ood} datapoints and consequently improve personalization.

\begin{figure*}
    \centering
    \includegraphics[page=1, width=1\linewidth]{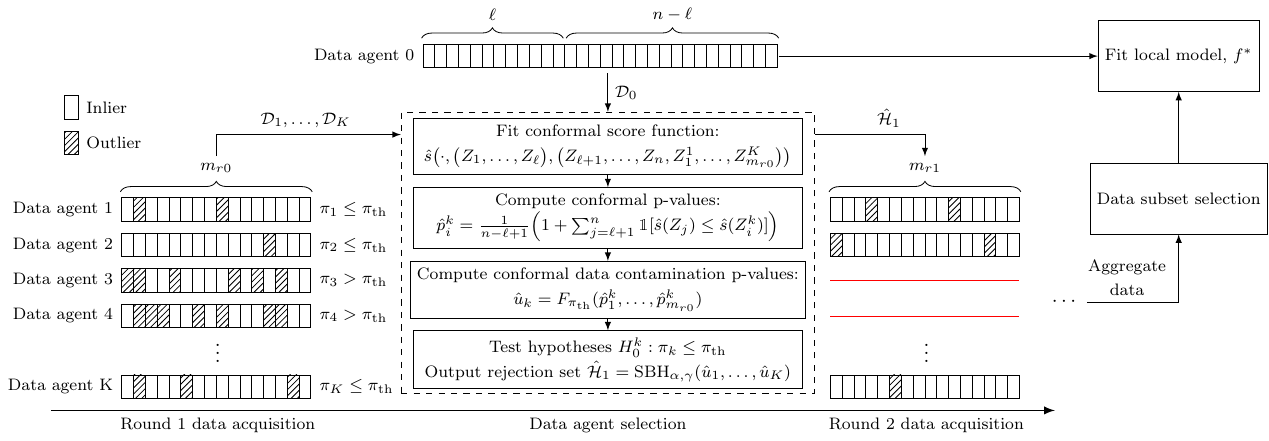}
    \vspace{-0.4cm}
    \caption{{\textit{Proposed data acquisition procedure:}
    In the first round, $m_{r0}$ samples are received at data agent $0$ from each of the other $K$ data agents. Then, a conformal score function $\hat{s}$ is fitted to the data with a split of the local data into $\ell$ training data and $n-\ell$ calibration data, conformal p-values $\hat{p}_i^k$ are computed for the acquired data, the proposed conformal data contamination p-values $\hat{u}_k$ are evaluated as functions $F_{\pi_{\rm th}}$ of the conformal p-values, and these data contamination p-values are used for testing the null hypotheses $H_0^k : \pi_k \leq \pi_{\rm th}$ with Storey's Benjamini-Hochberg (${\rm SBH}_{\alpha,\gamma}$) procedure, resulting in a data agent rejection set $\hat{\mathcal{H}}_1$. In the following round, data agent $0$ only acquires data from those not in the rejection set.} Finally, a data subset selection technique is used to filter away OOD data, and the local model is trained.}
    \label{fig:protocol_figure}
    \vspace{-0.3cm}
\end{figure*}

Based on the data contamination model, the first challenge can be posed as testing null hypotheses $H_0^1,\dots,H_0^K$ where $H_0^k : \pi_k \leq \pi_{\rm th}$ for a user-specified contamination factor threshold $\pi_{\rm th} \in [0, 1)$. The idea is that we reject $H_0^k$ when we have sufficient evidence that the data from the $k$-th data agent is contaminated beyond the threshold we allow. A theoretical (but impractical) solution to this problem has been considered in \cite{Blanchard2010:Semi}, and some existing works are occupied with estimating the contamination factor \citep{Ramaswamy2016:Mixture,Perini2023:Estimating}; however, \cite{Ramaswamy2016:Mixture} provide no distributional guarantees, and \cite{Perini2023:Estimating} uses a complex Bayesian approach that requires distributional and prior assumptions. {Relying on the contamination factor for choosing collaborators does not monotonically relate to improved model performance, but we show that it improves on the uninformed alternative and is therefore a basis for a practical solution.}

\textbf{Main Contributions:} The objective of this work is to improve the local model performance through quality data acquisition during training {when data agents are honest but have unintentionally biased data}. This is done by employing a data selection strategy that identifies agents for collaboration through novel data contamination tests, based on low-complexity outlier detection algorithms. The key contribution of this technique is its practicality and low-complexity relative to expensive data valuation techniques. Additionally, the proposed procedure allows for flexible and interpretable adjustment of the trade-off between the amount of data acquired and the relevance of the acquired data for the personalized learning goal. Specifically, we propose a distribution-free solution for testing $H_0^1,\dots,H_0^K$ with \gls{fdr} control guarantees that builds on the ideas of conformal outlier detection \citep{Bates2023:Testing,Marandon2024:Adaptive}, and apply the developed tools to the data {acquisition} scenario, as outlined in Figure~\ref{fig:protocol_figure}.
The main contributions are summarized as follows:\vspace{-0.5em}
% The summary of our contributions is the following:
\begin{itemize}[leftmargin=1em]
    \setlength\itemsep{0.03em}
    \item {We present a data acquisition procedure developed for personalization that relies on tests of data contamination to initialize collaboration of data agents while providing statistical guarantees on the rate of wrong exclusion.}
    \item We propose novel two-sample testing procedures {based on four types of conformal p-value combination test statistics. This enables us to define four types} of conformal data contamination p-values for testing $H_0^k : \pi_k \leq \pi_{\rm th}$, $k=1,\dots,K$, without any distributional assumptions which provably controls the false rejection probability, and generalize the combination tests of \cite{Bates2023:Testing} from the special case of $\pi_{\rm th} = 0$ to any $\pi_{\rm th} \in [0, 1)$.
    Moreover, the tests provide a lot of flexibility with the choice of the conformal score, and as such are compatible with a wide range of low-complexity outlier detection methods.
    {This contribution has independent interest and can potentially have practical significance for many applications. In this work, we highlight how these tests can be used to improve personalization through in-distribution data acquisition in data contamination scenarios.}
    % \item We show that the p-value sequence for testing $H_0^1,\dots,H_0^K$ is \gls{prds} thereby allowing for \gls{fdr} control using the \gls{bh} procedure.
    \item {We show that the proposed p-value sequence, based on a Storey-type combination test statistic, for testing $H_0^1,\dots,H_0^K$ is \gls{prds} thereby allowing for \gls{fdr} control using the \gls{bh} procedure.}
    % \item Numerical experiments on the MNIST and FEMNIST datasets are conducted to validate the effectiveness of the conformal data contamination tests and the proposed data sharing procedure.
    \item Numerical experiments are conducted on medical image datasets, as well as a label noise scenario on the MNIST dataset, to validate the effectiveness of the conformal data contamination tests and the proposed data sharing procedure.
\end{itemize}

% Although the setting may superficially invite the use of federated learning, our goal and methodology differ fundamentally: we are not concerned with distributed model training, but with quality data acquisition across heterogeneous data agents.
We reiterate the fact that the setup is fundamentally different from the one of federated learning: we are not concerned with privacy-aware distributed model training that would work for all involved parties, but rather acquisition of quality data that contributes to personalized learning goals.
% Some readers may suggest to rather use federated learning in this setting, but...
% Some readers may think to instead use federated learning in this setting, but conformal inference is not valid when applied to federated model updates
% our approach is fundamentally different : through data sharing we obtain statistical guarantees prior to expensive model training and allow the data agent to subsequently use the data 

\textbf{Other related works:}
Several works towards collaborative machine learning ecosystems focus on efficient mechanism design that enables data sharing amongst heterogeneous data agents. 
% Several works towards collaborative machine learning focus on efficient mechanism design for data sharing among heterogeneous agents.
Based on rich literature from economics and game theory, novel value attribution frameworks have been investigated to answer fundamental questions on data collection and sharing, which precede compute-intensive model training. For instance, \cite{Huang2023:Evaluating} discussed interactions between data agents when designing mechanisms for data collection in collaborative learning. Studies on \emph{data marketplaces} started evaluating context and specific goals of data agents to develop data acquisition processes, such as in \cite{Lu2024:Data}, where a buyer's objective is fulfilled through collection of statistically relevant data for training a predictive model, and in \cite{Capitaine2024:Unravelling}, where the total welfare is maximized through collaboration. While \emph{pricing} on data value has been a key incentive for assuring quality data acquisition in most works, \cite{Ananthakrishnan2024:Delegating} extended it to understand how much test data to use for verification, and how to tune hyperparameters once data collection is done for offering machine learning services.

A line of work in conformal inference, that is related to but different from ours, deals with \emph{corruption} in the calibration data, which in the general sense means that the calibration data is not exactly distributed according to the null distribution $P_0$. Such \emph{corruption} violates the fundamental exchangeability assumption that conformal inference relies on, however, approximately valid inference is still possible. The methodologies in this line of work is relevant for downstream uncertainty quantification following (imperfect) in-distribution data acquisition. \cite{Tibshirani2019:Conformal} proposed weighted conformal prediction and applied it to handle covariate shift. Subsequently, weighted conformal prediction was considered for label shift in \cite{Podkopaev2021:Distribution} and for generalized covariate shift and posterior drift in \cite{Wang2025:Conformal}. A detailed theoretical analysis of conformal prediction beyond exchangeability was conducted in \cite{Barber2023:Conformal}. Recently \emph{robust} conformal prediction has been considered: \cite{Bashari2025:Robust} proposed a technique to avoid \emph{corruption} by detecting outliers in the calibration data, and \cite{Feldman2024:Robust} considered scenarios with access to privileged information (not available at test time) that explains the distribution shift.
% If further, data agents have strict privacy constraints, federated conformal methodologies can be used for distributed uncertainty quantification. Conformal inference uses the empirical quantile of the calibration scores, and so \cite{Humbert2023:One} proposed to use the quantile-of-quantiles among data agents, while \cite{Lu2023:Federated} proposed a federated quantile estimator. Concurrently, \cite{Plassier2023:Conformal} developed a technique based on weighted conformal prediction and federated quantile estimation to handle heterogeneous label distributions, which they subsequently adapted to heterogeneity in covariate distributions \cite{Plassier2024:Efficient}. Contrary to this paper, these works deal with uncertainty quantification in the downstream classification task.

\section{Preliminaries}\label{sec:preliminaries}

\textbf{\textit{Notations:}}
Let $[n] = \{1, \dots, n\}$ and $[n]_0 = \{0, \dots, n\}$ denote index sets for $n \geq 1$, and in an abuse of notation let $[n]/m = \{1/m, \dots, n/m\}$ {and $m+[n]=\{m+1,\dots,m+n\}$}. {Denote by ${\rm B}_{\pi_{\rm th}}^{m}(k)$ the probability mass function of a binomial distribution evaluated at $k$ with parameters $\pi_{\rm th}$ and $m$, by $F_{\rm NHG}(x; {n-\ell}+k, {n-\ell}, k-r)$ the \gls{cdf} of the negative hypergeometric distribution evaluated at $x \in [n-\ell]_0$ with population size ${n-\ell}+k$, number of success states ${n-\ell}$, and number of failures $k-r$, by $F_{\chi^2_{2k}}$ the \gls{cdf} of the chi-square distribution with $2k$ degrees of freedom, and by $F_{{\rm IH}_k}$ the \gls{cdf} of the Irwin-Hall distribution with parameter $k$.}

\subsection{Multiple testing}\label{subsec:multiple_testing}
Consider a multiple testing scenario in which $K$ hypotheses $H^{1}_{0}, \dots, H^{K}_{0}$ are tested, and denote by $p_{1}, \dots, p_{K}$ the p-values.
We define $S$ as the number of rejected null hypotheses that are actually false, and $V$ as the number of rejected null hypotheses that are actually true (type I error). The total number of rejected hypotheses is given by $R=V+S$, while $K_0$ represents the total number of true null hypotheses.

A widely studied error rate in multiple hypothesis testing is the \gls{fdr}, introduced by \cite{Benjamini1995:FDR}.
It is defined as $\mathrm{FDR}=\mathbb{E}[V/\max(1, R)]$. Similarly, the power of a test, also known as the \gls{tdr}, is defined as $\mathbb{E}[S/\max(1, K-K_0)]$.
The most well-known method for controlling the \gls{fdr} is the \gls{bh} procedure{. Let $p_{(1)}\leq\cdots\leq p_{(K)}$ denote the ordered p-values. The \gls{bh} procedure at level $q^*\in(0,1)$ rejects hypotheses with indices $j\leq\kappa$ where $\kappa = \max\{j\in[K] : p_{(j)}\leq q^* j/K\}$. Equivalently, the rejection set is $\hat{\mathcal{H}}_1 = \{j \in [K] : p_{j} \leq p_{(\kappa)}\}$.} When the p-values are \gls{prds} {(see Section~\ref{sec:multitest})}, the \gls{bh} procedure controls the \gls{fdr} at level $q^* K_0/K$ \citep{Benjamini2001:Dependency}. There exists adaptive procedures which estimate $K_0$, defined as $\hat{K}_0$, and accordingly adjusts the procedure by running it at level $q^* = \alpha K / \hat{K}_0$, for significance level $\alpha \in (0, 1)$ \citep{Storey2003:Strong,Benjamini2006:Adaptive}. With Storey's \gls{bh} procedure, $\kappa_{\alpha, \gamma} = \max\{j\in[K] : p_{(j)}\leq \alpha j / \hat{K}_0\}$, where $\hat{K}_0/K = \sum_{j=1}^K \mathbbm{1}[p_j > \gamma]/(1-\gamma)$ with the hyperparameter $\gamma \in (0, 1)$ controlling the bias-variance trade-off, and we denote the rejection set of this procedure by ${\rm SBH}_{\alpha, \gamma}(p_1,\dots,p_K) = \{j \in [K] : p_{j} \leq p_{(\kappa_{\alpha, \gamma})}\}$.

% A direct approach was introduced in \cite{Storey2002:Direct} by considering a liberal estimate of the \gls{fdr}
% \begin{equation}\label{eq:FDR_est}
%     \widehat{{\rm FDR}} = {\rm min}\Big\{1, \frac{\delta\sum_{j=1}^K \mathbbm{1}[p_j > \gamma]}{(1-\gamma) \sum_{j=1}^K \mathbbm{1}[p_j \leq \delta]}\Big\},
% \end{equation}
% where $\delta \in [0, 1]$ defines the rejection region $[0, \delta]$.
% Here, one can set $\delta = p_{(\kappa)}$ for a chosen $\kappa$, yielding a rejection set $\hat{\mathcal{H}}_1 = \{j \in [K] : p_{j} \leq p_{(\kappa)}\}$, and then estimate the corresponding \gls{fdr} for this $\kappa$.

\subsection{Conformal outlier detection}\label{subsec:conformal_outlier_detection}
In the conformal outlier detection setting, a number of test points $Z_{n+1}, \dots, Z_{n+m}$, as well as a null sample $Z_1, \dots, Z_n$ are observed. It is assumed that the null sample datapoints are exchangeable and from an unknown distribution $P_0$. Then, a real-valued conformal score function, $\hat{s}$,
% $\hat{s} : \mathcal{Z}\times\mathcal{Z}^{\ell}\times\mathcal{Z}^{n+m-\ell}\to[0, \infty)$
is defined satisfying the following permutation invariance property
\begin{equation*}
    \hat{s}(\cdot) := \hat{s}(\cdot, (Z_1, \dots, Z_\ell), (Z_{\sigma(\ell+1)}, \dots, Z_{\sigma(n+m)})) = \hat{s}(\cdot, (Z_1, \dots, Z_\ell), (Z_{\ell+1}, \dots, Z_{n+m})),
\end{equation*}
for any permutation $\sigma$ of $\{\ell+1,\dots,n+m\}$, and $0 \leq \ell < n$ \citep{Marandon2024:Adaptive}. {Large scores $\hat{s}_{n+i} = \hat{s}(Z_{n+i})$ indicate conformity with $P_0$; small scores indicate outliers.}
% The interpretation here is that a large value of $\hat{s}_{n+i} = \hat{s}(Z_{n+i}, (Z_1, \dots, Z_\ell), (Z_{\ell+1}, \dots, Z_{n+m}))$, $i=1,\dots,m$, is evidence supporting the hypothesis that $Z_{n+i} \sim P_0$, and vice versa.
% \pp{PP: Perhaps an example here would help.} \mv{MV: We can mention OCSVM and give a reference?}
% {Let $\bar{\mathcal{H}}_0 = \{i\in[m]:Z_{n+i} \sim P_0\}$ denote the set of nulls in the test sample, and $\bar{\mathcal{H}}_1 = [m]\setminus\bar{\mathcal{H}}_0$. As in \citep{Marandon2024:Adaptive}, we assume that $(Z_1,\dots,Z_n,(Z_{n+i})_{i\in\bar{\mathcal{H}}_0})$ are exchangeable conditional on $(Z_{n+i})_{i\in\bar{\mathcal{H}}_1}$.}
{As in \cite{Marandon2024:Adaptive}, we assume that the inliers in the test sample are exchangeable with the calibration data $Z_{\ell+1},\dots,Z_n$, conditional on the outliers.}
% We will also assume the following property: $(Z_1,\dots,Z_n,(Z_{n+i})_{i\in\bar{\mathcal{H}}_0})$ are exchangeable conditional on $(Z_{n+i})_{i\in\bar{\mathcal{H}}_1}$, where the set of nulls in the test sample is denoted $\bar{\mathcal{H}}_0 = \{i\in\{1,\dots,m\}\;\mathrm{s.t.}\;\bar{H}_0^{i}\;\mathrm{is}\;\mathrm{true}\}$ with $\bar{H}_0^{i} : Z_{n+i} \sim P_0$, and $\bar{\mathcal{H}}_1 = \{1,\dots,m\}\setminus\bar{\mathcal{H}}_0$.
For {$i\in[m]$}, the conformal p-value associated to the test point $Z_{n+i}$ is
\begin{equation}\label{eq:conformal_pvalue}
    \hat{p}_i = \frac{1}{n-\ell+1} \Big(1 + \sum_{j=\ell+1}^{n} \mathbbm{1}[\hat{s}_j \leq \hat{s}_{n+i}]\Big).
\end{equation}
Assuming that the conformal scores are continuously distributed and that $Z_{n+i}$ is an inlier independent of $Z_1,\dots,Z_\ell$, then $\hat{p}_i$ is uniformly distributed on
$[n-\ell+1]/(n-\ell+1)$
% $\{1/(n-\ell+1), 2/(n-\ell+1), \dots, 1\}$
, and one has that $\lfloor \beta(n-\ell+1)\rfloor/(n-\ell+1) \leq \mathbb{P}(\hat{p}_i \leq \beta) \leq \beta,$
% \begin{equation*}
%     \frac{\lfloor \alpha(n-\ell+1)\rfloor}{n-\ell+1} \leq \mathbb{P}(\hat{p}_i \leq \alpha) \leq \alpha,
% \end{equation*}
for $\beta \in (0, 1)$. This shows that $\hat{p}_i$ is a \textit{marginally} valid p-value with almost exact control of the false rejection probability \citep{Marandon2024:Adaptive}.

\section{Proposed data sharing procedure}\label{sec:procedure}
% In this section, we describe the proposed data sharing procedure in further details and discuss its complexity and limitations.

Taking the perspective of the $0$-th data agent, we consider a scenario where in the first round each data agent $k\in[K]$ {shares a test sample $\{Z_i^k\}_{i=1}^{m_{r0}}$.} We assume for simplicity {that the number of shared datapoints} $m_{r0}$ is the same for all data agents, however this assumption is not required for the proposed methodology.
{We model data coming from the $k$-th data agent by the contamination model \citep{Blanchard2010:Semi}:
\begin{equation*}
    Z_{i}^{k} \sim (1-\pi_k)P_0 + \pi_k P_{k}
\end{equation*}
where $\pi_k \in [0, 1]$ is the contamination factor.} We make the assumption here that $P_{k}$ are proper novelty distributions with respect to $P_0$: there exists no decomposition of the form $P_{k} = (1-\zeta)Q + \zeta P_0$ where $Q$ is a probability distribution and $0 < \zeta \leq 1$, see \cite{Blanchard2010:Semi,Zhu2023:Mixture} for details. {Roughly, this means that at least some data from $P_0$ cannot possibly be confused with data from $P_k$, and vice versa. With this interpretation, we aim to limit the acquisition of data that deviates so substantially from the local data that it cannot possibly be confused.}

As in Section~\ref{subsec:conformal_outlier_detection}, we require that {the calibration data and the inliers in the test sample are exchangeable, conditional on the outliers, for each $k\in[K]$. Outliers $Z_{i}^{k} \not\sim P_0$ are} not of interest for the $0$-th data agent as we focus on improving {personalization, here, by selecting relevant collaborators.}

{We proceed by testing null hypotheses $H_0^k : \pi_k \leq \pi_{\rm th}$ given a threshold $\pi_{\rm th}$ on how much contamination we tolerate. To do this, we begin by computing conformal p-values as in Eq.~\eqref{eq:conformal_pvalue}
\begin{equation*}
    \hat{p}_i^k = \frac{1}{n-\ell+1} \Big(1 + \sum_{j=\ell+1}^{n} \mathbbm{1}[\hat{s}_j \leq \hat{s}_{i}^k]\Big),
\end{equation*}
where $\hat{s}_{i}^k = \hat{s}(Z_{i}^k)$ are the test scores. The conformal p-values $\hat{p}_1^k,\dots,\hat{p}_{m_{r0}}^k$ are then mapped to conformal data contamination p-values $\hat{u}_k$,} {an example of which is the Storey-type contamination p-value
\begin{equation*}
    \hat{u}^{\rm storey}_k = \sum_{j = 0}^{T^{\rm storey}_k} {\rm B}_{\pi_{\rm th}}^{m_{r0}}(j) + \sum_{j = T^{\rm storey}_k+1}^{m_{r0}} {\rm B}_{\pi_{\rm th}}^{m_{r0}}(j) F_{\rm NHG}\big(\lfloor \lambda (n-\ell+1)\rfloor - 1; n-\ell+j, n-\ell, j-T^{\rm storey}_k\big),
\end{equation*}
where $T^{\rm storey}_k = \sum_{i=1}^{m_{r0}} \mathbbm{1}[\hat{p}_i^k > \lambda]$ and $\lambda \in [n-\ell]/(n-\ell+1)$ is a hyperparameter (see Section~\ref{sec:results}).}
{Finally, for a specified significance level, $\alpha \in (0, 1)$, the hypotheses are tested using an adaptive \gls{bh} procedure, e.g., Storey's \gls{bh} procedure described in Section~\ref{subsec:multiple_testing}, yielding a set of non-rejected data agents, $\hat{\mathcal{H}}_0 = [K] \setminus {\rm SBH}_{\alpha, \gamma}(\hat{u}_1, \dots, \hat{u}_K)$ for Storey's hyperparameter $\gamma \in (0, 1)$. The statistical guarantee, provided by the proposed tests, ensures that data agents observing a distribution highly relevant for the personalized learning goal, are rarely neglected in the collaboration.}

After determining which data agents to collaborate with in the following round(s), the $0$-th data agent once again acquires data $Z_{m_{r0}+i}^k$ for $i\in[m_{r1}]$ and $k\in \hat{\mathcal{H}}_0$, and can then aggregate all the received data, followed by data subset selection for instance using techniques from data valuation \citep{Koh2017:Understanding,Ghorbani2019:Shapley,Yoon2020:Reinforcement}. {To remain lightweight, we employ conformal outlier detection for this part, specifically, when fitting the local prediction model, we do not use acquired data that is detected as outliers using Storey's \gls{bh} procedure with significance level $\beta$ and Storey hyperparameter $\zeta$, i.e., the set ${\rm SBH}_{\beta, \zeta}(\{\hat{p}_i^k:i\in[m_{r0}], k \in [K]\}\cup\{\hat{p}_i^k:i\in m_{r0}+[m_{r1}], k\in \hat{\mathcal{H}}_0\})$.}
The selected data is combined with the $n$ local datapoints, and the $0$-th data agent trains its local model.
% The proposed procedure is visualized in Figure~\ref{fig:protocol_figure} and also summarized in Section~\ref{sec:additional_details} of the appendix. An overview of the scenario variables and hyperparameters is given in Section~\ref{sec:interpretations}, followed by a discussion of limitations in Section~\ref{subsec:limitations}{, and in Section~\ref{sec:hyperparameter_selection} a data-driven approach to hyperparameter selection is outlined.}
The proposed procedure is visualized in Figure~\ref{fig:protocol_figure} and also summarized in Procedure~\ref{alg:procedure}. An overview of the scenario variables and hyperparameters is given in Section~\ref{sec:additional_details}, and in Section~\ref{sec:hyperparameter_selection} a data-driven approach to hyperparameter selection is outlined.

\begin{algorithm}[]
  \caption{}
  \label{alg:procedure}
  \begin{algorithmic}[1]
    % \STATE {\bfseries Input:} local data $\{Z_i\}_{i=1}^\ell$ and $\{Z_i\}_{i=\ell+1}^n$, conformal score method $\hat{s}$, model class $f$, incoming data per round per data agent $m$, collaboration budget $K_{\rm budget}$ (or contamination threshold $\pi_{\rm th}$ and significance level $\alpha$), conformal non-contamination statistic $T$.
    \State {\bfseries Input:} local data $\{Z_i\}_{i=1}^\ell$ and $\{Z_i\}_{i=\ell+1}^n$, conformal score method $\hat{s}$, model class $f$, incoming data per round per data agent $m_{r0}$ and $m_{r1}$, contamination threshold $\pi_{\rm th}$, significance level $\alpha$, conformal non-contamination statistic $T$.
    \State {\bfseries Output:} optimized local model, $f^*$.
    \State Receive data (round 1) from other data agents $\{Z_{i}^{k}\}_{i=1}^{m_{r0}}$, $k\in[K]$.
    \State Learn conformal score function {$\hat{s}(\cdot) = \hat{s}(\cdot, (Z_1, \dots, Z_\ell), (Z_{\ell+1}, \dots, Z_n, Z_{1}^1, \dots, Z_{m_{r0}}^K))$}.
    \State {Compute calibration scores $\hat{s}_{\ell+1},\dots,\hat{s}_n$ where $\hat{s}_{i}=\hat{s}(Z_i)$.}
    \State {For each $k\in[K]$, compute test scores, $\hat{s}_{n+1}^k,\dots,\hat{s}_{n+m_{r0}}^k$ where $\hat{s}_{n+j}^k = \hat{s}(Z_j^k)$.}
    \State {For each $k\in[K]$, compute conformal p-values, $\hat{p}_1^k,\dots, \hat{p}_{m_{r0}}^k$.}
    \State {Evaluate conformal data contamination p-values, $\hat{u}_k$ (see Section~\ref{sec:results}).}
    % \IF{Given a fixed collaboration budget}
    % \STATE Select for collaboration in the following round data agents in $\hat{\mathcal{H}}_0 = \{\sigma(i) : i\in[K_{\rm budget}]\}$, where $\sigma$ is a permutation on $[K]$ such that $T_{\sigma(1)} \geq T_{\sigma(2)} \geq \cdots \geq T_{\sigma(K)}$, and estimate the \gls{fdr} for null hypotheses $H_0^k : \pi_k \leq \pi_{\rm th}$ for a $\pi_{\rm th}$ of interest with $\hat{u}_k$, $k \in [K]$, using {Eq.}~\eqref{eq:FDR_est}.
    % \ELSE
    \State Collaborate in the following round with data agents in $\hat{\mathcal{H}}_0 = [K] \setminus {\rm SBH}_{\alpha, \gamma}(\hat{u}_1,  \dots,  \hat{u}_K)$.
    % \ENDIF
    \State Receive data (round 2) from other data agents $\{Z_{m_{r0}+i}^{k}\}_{i=1}^{m_{r1}}$, $k\in\hat{\mathcal{H}}_0$.
    \State Run data subset selection on all the received data{, specifically, discard data
    \Statex ${\rm SBH}_{\beta,\zeta}(\{\hat{p}_i^k:i\in[m_{r0}], k \in [K]\}\cup\{\hat{p}_i^k:i\in m_{r0}+[m_{r1}], k\in \hat{\mathcal{H}}_0\})$.}
    \State Use all local data $\{Z_i\}_{i=1}^n$ together with the selected data to train the model, yielding $f^*$.
  \end{algorithmic}
\end{algorithm}

\textbf{Complexity:} One of the strengths of the proposed data acquisition protocol is its low computational complexity due to preceding the expensive model training procedure. Note that computing a conformal data contamination p-value as defined in Eqs.~\eqref{eq:storey_p_value}-\eqref{eq:Sum_p_value} requires at most $m$ evaluations of known CDFs: for Quantile and Storey the \gls{cdf} of a negative hypergeometric distribution, for Fisher the \gls{cdf} of a chi-square distribution, and for Sum the \gls{cdf} of an Irwin-Hall distribution. The most computationally complex part of the data acquisition protocol is fitting of the conformal score function $\hat{s}$ but notably the methodology complies with a wide range of conformal scores with the choice open to the practitioner. Further, the procedure scales well to a large number of data agents by virtue of the Benjamini-Hochberg procedure which was designed for large scale testing.

\textbf{Limitations:} The proposed data acquisition procedure has some limitations.
First, the conformal data contamination tests require exchangeability of calibration and test data to be theoretically valid, restricting use in settings like time series.
Second, we chose to focus on the conformal data contamination tests, rather than aspects of data subset selection{, however, this could be explored further in other works.}
% , however, in a practical procedure we recommend to include some data subset selection \citep{Ghorbani2019:Shapley}, or data weighting \citep{Ding2022:Collaborative}.
{Third, the procedure does not consider data acquisition budget constraints. With budget constraints, we would proceed differently by acquiring data from the agents with the lowest estimated contamination, until the budget is exhausted or all the available data from agents below the contamination threshold has been acquired.}
Lastly, the proposed data sharing procedure is not privacy-aware, limiting applicability to scenarios with strict privacy constraints. However, a slightly altered procedure, which replaces the first step of data sharing by federated learning of the conformal score function, can avoid privacy-leakage for data agents not selected for subsequent collaboration.

\section{Main results}\label{sec:results}
In the first round of data sharing as in Figure~\ref{fig:protocol_figure}, the data from the $k$-th agent yields conformal p-values $\hat{p}_{1}^{k},\dots,\hat{p}_{m}^{k}$ {(setting $m_{r0}=m$ for notational convenience)}. In the following, we consider a single sequence of conformal p-values, and so to simplify notations we denote this sequence by $\hat{p}_1,\dots,\hat{p}_m$. We are interested in testing the null hypothesis $H_0~:~\pi \leq \pi_{\rm th}$ using this sequence of conformal p-values.
Such tests can be referred to as combination tests as they use a test statistic which combines a sequence of p-values. We also refer to these tests as two-sample tests since the conformal p-values, and in turn the combination test statistics, are computed using two samples of data, i.e., the null sample and the test sample.
% In the following, we introduce different test statistics based on combining the sequence of conformal p-values, and for each of the test statistics, we present a valid p-value for testing $H_0$.

{We start with the motivation of the considered combination statistics. A celebrated estimator of the contamination factor $\pi$, is the Storey estimator defined as
\begin{equation*}
    \hat{\pi}^{\rm storey} = 1 - \frac{\sum_{i=1}^m \mathbbm{1}[\hat{p}_i > \lambda]}{m (1-\lambda)},
\end{equation*}
where $\lambda\in [n-\ell]/(n-\ell+1)$ is a hyperparameter controlling the bias-variance trade-off \citep{Storey2003:Strong}. The data enters through the statistic $\sum_{i=1}^m \mathbbm{1}[\hat{p}_i > \lambda]$, and intuitively, when the p-values are large, the statistic is large meaning the estimate $\hat{\pi}^{\rm storey}$ is small, suggesting a small contamination factor, and vice versa. A competing estimator is the Quantile estimator defined as
\begin{equation*}
    \hat{\pi}^{\rm quantile} = 1 - \frac{i_0 + 1}{m (1-\hat{p}_{(m-i_0)})},
\end{equation*}
where $\hat{p}_{(1)}\leq\cdots\leq\hat{p}_{(m)}$ are the ordered p-values and $i_0 \in [m-1]_0$ is a hyperparameter controlling the bias-variance trade-off \citep{Benjamini2006:Adaptive}. Here, the data enters through the statistic $\hat{p}_{(m-i_0)}$ and the intuition is similar to the Storey estimator. These two statistics are motivated from estimators of $\pi$, however, more general constructions are possible. In testing theory, a classical statistic used to test if $H_{0}^i$ are simultaneously true for all $i$ (equivalently to test if $\pi = 0$), is the Fisher statistic $-2\sum_{i=1}^m{\rm log}(\hat{p}_i)$ \citep{Fisher1925:Statistical}. The motivation behind this choice is that small p-values are emphasized in the aggregation --- if even a single p-value is very small the statistic takes a large value --- as well as the fact that it has a relation to the chi-square distribution. Naturally, other statistics aggregating p-values have also been studied in the testing literature, and we mention here the Sum statistic $\sum_{i=1}^m \hat{p}_i$ which gives no particular emphasis on smaller or larger p-values \citep{Vovk2020:Combining}. These statistics can also be used in our setting, although they have historically been developed for the special case $\pi_{\rm th}=0$. For the Fisher statistic, a small value is indicative of a small contamination factor, while for the Sum statistic, a large value is indicative of a small contamination factor. These observations are exploited in the following where we derive conformal data contamination p-values.}
% Denoting by $m_0$ the number of true nulls, the motivation behind this estimator is that, using the superuniformity of the true null p-values, $\mathbb{E}[\sum_{i=1}^m \mathbbm{1}[\hat{p}_i > \lambda]]\geq m_0(1-\lambda)$, and so $\mathbb{E}[\hat{\pi}^{\rm storey}] \leq \frac{m-m_0}{m}$.
% \begin{align*}
%     \hat{\pi}^{\rm quantile} = 1 - \frac{i_0 + 1}{m (1-\hat{p}_{(m-i_0)})}.
% \end{align*}

{Let $\bar{\mathcal{H}}_1 = \{i\in[m] : Z_{n+i}\not\sim P_0\}$ denote the set of outliers.} We now state one of our main results, defining valid p-values for {the Storey and Quantile} test statistics.

\begin{theorem}\label{thm:nonasymptotic_tests}
    Let $H_0 : \pi\leq\pi_{\rm th}$, and let $\hat{p}_i$ for $i\in[m]$ denote the conformal p-values as introduced in Section~\ref{subsec:conformal_outlier_detection}.
    The following are valid p-values for $H_0$:
    \begin{equation}\label{eq:storey_p_value}
        \hat{u}^{\rm storey} = \sum_{k = 0}^{T^{\rm storey}} {\rm B}_{\pi_{\rm th}}^{m}(k) + \sum_{k = T^{\rm storey}+1}^m {\rm B}_{\pi_{\rm th}}^{m}(k) F_{\rm NHG}\big(\lfloor \lambda ({n-\ell}+1)\rfloor - 1; {n-\ell}+k, {n-\ell}, k-T^{\rm storey}\big),
    \end{equation}
    where $T^{\rm storey} = \sum_{i=1}^m \mathbbm{1}[\hat{p}_i > \lambda]$ and $\lambda \in [{n-\ell}]/({n-\ell}+1)$ is a hyperparameter, moreover, if $\mathbb{P}(\hat{p}_i \leq \lambda, \forall i \in \bar{\mathcal{H}}_1) = 1$ and $\pi = \pi_{\rm th}$ the p-value is exact;
    \begin{equation}\label{eq:quantile_p_value}
        \hat{u}^{\rm quantile} = \sum_{k = 0}^{i_0} {\rm B}_{\pi_{\rm th}}^{m}(k) + \sum_{k = i_0+1}^m {\rm B}_{\pi_{\rm th}}^{m}(k) F_{\rm NHG}\big(T^{\rm quantile} - 1; {n-\ell}+k, {n-\ell}, k-i_0\big),
    \end{equation}
    where $T^{\rm quantile} = ({n-\ell}+1) \hat{p}_{(m-i_0)}$ with $\hat{p}_{(m-i_0)}$ denoting the $(m-i_0)$-th smallest conformal p-value with hyperparameter $i_0 \in [m-1]_0$, moreover, if $\mathbb{P}(\hat{p}_i \leq \hat{p}_{(m-i_0)}, \forall i \in \bar{\mathcal{H}}_1) = 1$ and $\pi = \pi_{\rm th}$ the p-value is exact.
\end{theorem}
\begin{proof}
    The proof is given in Section~\ref{sec:CCTests} of the appendix.
\end{proof}
\begin{remark}
    The strategy of the proof is to use the law of total probability to decompose the rejection probability under the null into a sum over the number of inliers. Subsequently, an inequality is made by discarding contributions from the conformal p-values of outliers. Finally, using the marginal distribution of the order statistics of the conformal p-values, which specifically is a negative hypergeometric distribution, see \cite{Gazin2024:Transductive} and \cite{Biscio2025:Conformal}, allows for an explicit expression for the rejection probability under the null, which in turn defines a p-value.

    % The test statistic $T^{\rm storey}$ is motivated by the Storey estimator of $\pi$ by \cite{Storey2002:Direct}, while $T^{\rm quantile}$ is motivated by the quantile estimator of $\pi$ by \cite{Benjamini2006:Adaptive}.
\end{remark}

More general test statistics are also allowed, however, constructing valid p-values require knowledge of the distribution of the test statistic under the null. In the following result, we present {the Fisher and Sum} test statistics and their associated p-values which are asymptotically valid.
\begin{theorem}\label{thm:asymptotic_fisher_test}
    Let $H_0 : \pi\leq\pi_{\rm th}$ and let $\hat{p}_i$ for $i\in[m]$ denote the conformal p-values as introduced in Section~\ref{subsec:conformal_outlier_detection}. Then, asymptotically as ${n-\ell}\to\infty$ the following are valid p-values for $H_0$:
    \begin{equation}\label{eq:fisher_p_value}
        \hat{u}^{\rm fisher} = \pi_{\rm th}^{m} + \sum_{k = 1}^m {\rm B}_{\pi_{\rm th}}^{m}(k) F_{\chi^2_{2k}}\bigg(2k + \frac{-T^{\rm fisher} + 2k\big({\rm log}({n-\ell}+1)-1\big)}{\sqrt{1+k/{(n-\ell)}}}\bigg),
    \end{equation}
    where $T^{\rm fisher} = -2\sum_{i=1}^m \Big({\rm log}\big(\frac{1}{{n-\ell}+1}\big) - {\rm log}(\hat{p}_i)\Big)$;
    \begin{equation}\label{eq:Sum_p_value}
        \hat{u}^{\rm sum} = \pi_{\rm th}^{m} + \sum_{k = 1}^m {\rm B}_{\pi_{\rm th}}^{m}(k) F_{{\rm IH}_k}\bigg(\frac{T^{\rm sum} + k(\sqrt{1+k/{(n-\ell)}} - 1)/2}{\sqrt{1+k/{(n-\ell)}}}\bigg),
    \end{equation}
    where $T^{\rm sum} = \sum_{i=1}^m \hat{p}_i$.
\end{theorem}
\begin{proof}
    The proof is found in Section~\ref{sec:CCTests} of the appendix.
\end{proof}
\begin{remark}
    The strategy of the proof is similar as that of Theorem~\ref{thm:nonasymptotic_tests}, however, since the distribution of $T^{\rm fisher}$ or $T^{\rm sum}$ where $\hat{p}_i$ are all inliers is not explicitly known, we employ the asymptotic result Theorem S1 of \cite{Bates2023:Testing}. The general theorem is actually more general than presented here, as one can consider a general class of test statistics of the form $\sum_{i=1}^m G(\hat{p}_i)$ requiring that $G(U)$ has finite moments for $U \sim {\rm Unif}([0, 1])$ \citep{Bates2023:Testing}. The general statement with precise conditions required on $G$ is given in the appendix.

    % The test statistic $T^{\rm fisher}$ is motivated by Fisher's combination test \citep{Fisher1925:Statistical}, and $T^{\rm sum}$ is another classical combination test statistic \citep{Vovk2020:Combining}.
\end{remark}

With the p-values of {Eqs.~\eqref{eq:storey_p_value}-\eqref{eq:Sum_p_value}}, tests for $H_0$ can be conducted without requiring any assumptions on the null distribution, $P_0$, and outlier distribution, $P_1$.
Effective conformal data contamination testing relies on conformal scores that clearly separate inliers from outliers, leading to small conformal p-values for outliers and enabling exact control when $\pi = \pi_{\rm th}$, as seen in the details in the appendix. The key parameter is $\ell$, which balances test power: {larger} $\ell$ improves score quality, while larger $n-\ell$ enhances distribution approximation. Finally, having a large $m$ in general also improves the power of the conformal data contamination tests as more evidence against the null can be aggregated. This is opposite to the case of conformal outlier detection for which \cite{Mary2022:Semi} studied how larger $m$ results in a loss of power.
% As a shorthand, we denote in the following by $\hat{p}^{\rm con}$ any of such p-values and refer to them as conformal data contamination p-values. 
% Next we consider a scenario where $H_0$ is tested for multiple test data sets, thereby requiring a multiplicity correction.

% \subsection{Multiple testing with conformal data contamination p-values}
In case multiple test sets are observed, we have multiple sets of conformal p-values, each of which gives a conformal data contamination p-value. We consider here how to do multiple testing with this sequence of conformal data contamination p-values while proving \gls{fdr} control guarantees. Motivated by the result of \cite{Bates2023:Testing} that conformal p-values are \gls{prds}, we have shown that the Storey conformal data contamination p-values are indeed also \gls{prds}. Hence, we can simply apply the \gls{bh} procedure on the conformal data contamination p-values to guarantee \gls{fdr} control.
\begin{theorem}\label{thm:Storey_PRDS}
    Let $\hat{u}_1^{\rm storey}, \dots, \hat{u}_K^{\rm storey}$ be $K$ Storey conformal data contamination p-values as in {Eq.}~\eqref{eq:storey_p_value}, with the $k$-th derived from conformal p-values ${\hat{p}_{1}^k, \dots, \hat{p}_{m}^k}$ assuming independence of the $K$ test sets $\mathcal{D}_k = \{(X_i^k, Y_i^k)\}_{i=1}^{m}$. Then, the Storey conformal data contamination p-values are \gls{prds}.
\end{theorem}
\begin{proof}
    The proof is found in Section~\ref{sec:CCTests} of the appendix.
\end{proof}
\begin{remark}
    While Theorem \ref{thm:Storey_PRDS} only refers to the Storey conformal data contamination p-values, we expect the result can be generalized to a broad class of conformal data contamination p-values, herein the ones presented in this paper, and we leave this as a line of future investigation.
    % Notice that Theorem \ref{thm:Storey_PRDS} only refers to the Storey conformal data contamination p-values, and as such we do not provide the same theoretical \gls{fdr} guarantees when using the other conformal data contamination p-values. However, we expect the result can be generalized to a broad class of conformal data contamination p-values, herein the ones presented in this paper, and we leave this as a line of future investigation.
\end{remark}

\textbf{Selecting the test:} Based on theoretical considerations and numerical experiments, we give the following recommendations: For small calibration sets (e.g.~$n-\ell < 100$), the Storey and Quantile statistics are generally preferred as they offer finite-sample guarantees; however, both require the selection of a hyperparameter, which can be challenging in practice.
% We emphasize the parallel to the Storey and Quantile estimators of the true null proportion, see \cite{Storey2003:Strong} and \cite{Benjamini2006:Adaptive}, of which neither one is consistently better.
Since the Storey and Quantile statistics can take $m$ and $n-\ell$ distinct values, respectively, we recommend the Storey statistic when $m > n-\ell$ and vice versa. When ample calibration data are available, the Fisher or Sum statistics are recommended as they are hyperparameter free. The Fisher statistic is suited for low contamination thresholds, whereas the Sum statistic is appropriate for higher contamination thresholds.

% \subsection{Analysis of complexity and hyperparameter influence}\label{subsec:analysis}
% \mv{Complexity, hyperparameters...}
% \mv{The complexity of the conformal contamination test is $\mathcal{O}(mC_1(\ell)C_2(n, m))$ where $C_1(\ell)$ is the complexity of fitting the conformal score and $C_2(n,m)$ is the complexity of evaluating the cumulative distribution functions arising in the conformal contamination test p-values. As an example, for the Storey conformal contamination test p-value, $C_2(n,m) = \mathcal{O}(n ~ {\rm max}(n,m) ~ {\rm log}({\rm max}(n,m)))$. For the Fisher and sum conformal contamination test p-values, $C_2(n,m)=C_2(m)$ does not depend on $n$. When doing multiple conformal contamination tests, the complexity is just $\mathcal{O}(KmC_1(\ell)C_2(n, m))$.}

% In Section~\ref{sec:interpretations} we provide an overview of the relevant scenario parameters and hyperparameters, and discuss their influence on the conformal contamination tests in isolation, as well as their influence in the proposed collaborative data sharing scheme. This is also accompanied by an ablation study in Section~\ref{sec:Gaussian_data}.

\section{Numerical experiments and analysis}\label{sec:numerical_experiments}
In this section, we numerically evaluate the proposed methodology on two data contamination scenarios. In the first scenario, retinal fundus images are used to classify severity of diabetic retinopathy and the inliers are drawn from RetinaMNIST \citep{Liu2022:RetinaMNIST} with contaminated data drawn from EyePACS \citep{Dugas2015:EyePACS,Gulshan2016:EyePACS}.
% In the second scenario, microscopic blood smear images are used to classify among four cell types (\emph{basophile}, \emph{eosinophile}, \emph{lymphocyte}, and \emph{monocyte}) and the inliers are drawn from BloodMNIST \cite{Acevedo2020:BloodMNIST} with contaminated data drawn from a different blood cell dataset \cite{Bodzas2023:WBC}.
In the second scenario, images of handwritten digits are used to classify the digits ($1$, $4$, and $7$) and the inliers are drawn noise-free from MNIST \citep{Lecun1998:Gradient} with contaminated data having incorrect labels. First we introduce the data setup, and present the baseline approaches. Then, the conformal data contamination tests are analyzed numerically. Finally, the performance of the proposed data sharing procedure is evaluated on the classification tasks. The code used for the study can be found in the \href{https://github.com/Martin497/ConformalDataContaminationTest}{github repository}\footnote{\url{https://github.com/Martin497/ConformalDataContaminationTest}}.

\subsection{Data setup and baselines}\label{subsec:data_setup}
The RetinaMNIST and EyePACS datasets consists of $1600$ and $35126$ retinal fundus images, respectively. Each image in RetinaMNIST is $28\times28\times3$ with each pixel taking values in $[255]_0$ and has an associated label $y \in [4]_0$ indicating the severity of diabetic retinopathy ranging from no diabetic retinopathy to proliferate retinopathy. The EyePACS data has varying resolutions, so we crop and downsample the images to the $28\times28\times3$ resolution.
% We let the null distribution be to sample independently from the RetinaMNIST dataset, while data contamination is simulated by sampling independently from the EyePACS dataset.

% {The BloodMNIST and the other blood cell dataset consists of $6969$ and $7126$ microscopic blood smear images, respectively. Each image in BloodMNIST is $28\times28\times3$ with each pixel taking values in $[255]_0$ and has an associated label $y \in [3]_0$ indicating the cell type. The other blood cell dataset has varying resolutions, so we crop and downsample the images to the $28\times28\times3$ resolution. We let the null distribution be to sample independently from the BloodMNIST dataset, while data contamination is simulated by sampling independently from the other blood cell dataset.}

The MNIST dataset restricted to digits $1$, $4$, and $7$ consists of $21994$ images of handwritten digits. Each image is $28\times28$ with each pixel taking values in $[255]_0$ and has an associated label $y \in [2]_0$ indicating the digit.
% We let the null distribution be to sample independently from the MNIST dataset, while data contamination is simulated by assigning an incorrect label to an image from the MNIST dataset.

We normalize the images to the interval $[0, 1]$ by dividing pixel-wise with $255$.
%We sample the contamination factors \gls{iid} according to a standard uniform distribution.
For all the numerical experiments we set $n=100$, $\ell=60$, $m=m_{r0}=m_{r1}=40$, $K = 10$, $\lambda=\lfloor {(n-\ell)}/8\rfloor/({n-\ell}+1)$, $i_0 = \lfloor m/3 \rfloor$, and $\gamma=0.5$. Unless otherwise specified, we simulate the data contamination factors \gls{iid} according to $\pi_k \sim {\rm Uniform}([0,1])$. In all cases we sample randomly $n$ in-distribution datapoints for the $0$-th data agent. Meanwhile, the $k$-th data agent observes $2m$ datapoints and the number of \gls{ood} datapoints is sampled from ${\rm B}_{\pi_k}^{2m}$.

As a conformal score, we will use the \gls{pu} classification approach advocated in \cite{Marandon2024:Adaptive}. For simplicity, the \gls{pu} classifier is chosen as \gls{lr}, and we use the standard implementation in \textit{scikit-learn} with balanced class weights. We define the score to be class-wise in the sense that
% \begin{equation*}
%     \begin{split}
%         &\hat{s}((\bs{X}, Y), (Z_1, \dots, Z_\ell), (Z_{\ell+1}, \dots, Z_{n+m})) \\
%         &\quad = \sum_{i \in \mathcal{Y}} \mathbbm{1}[Y = i] ~ \hat{s}_i(\bs{X}, (\bs{X}_j : Y_j = i, j \in \{1,\dots,\ell\}), (\bs{X}_j : Y_j = i, j \in \{\ell+1,\cdots n+m\})),
%     \end{split}
% \end{equation*}
% \begin{equation*}
%     \hat{s}((\bs{X}, Y), (Z_1, \dots, Z_\ell), (Z_{\ell+1}, \dots, Z_{n+m})) = \sum_{i \in \mathcal{Y}} \mathbbm{1}[Y = i] ~ \hat{s}_i(\bs{X}, \bs{X}_{1:\ell}^{Y=i}, \bs{X}_{\ell+1:n+m}^{Y=i}),
% \end{equation*}
% where $\mathcal{Y}$ denotes the set of labels, $\bs{X}_{1:\ell}^{Y=i} = (\bs{X}_j : Y_j = i, j\in\{1,\dots,\ell\})$ denotes the subset of $\bs{X}_1, \dots, \bs{X}_\ell$ for which $Y_j = i$ for $j\in\{1,\dots,\ell\}$.
\begin{equation*}
    {\hat{s}((X, Y), (Z_1, \dots, Z_\ell), (Z_{\ell+1},\dots,Z_{n}, Z_1^1, \dots, Z_{m}^K)) = \sum_{i \in \mathcal{Y}} \mathbbm{1}[Y = i] ~ \hat{s}_i(X, X_{\rm train}^{Y=i}, X_{\rm cal + test}^{Y=i}),}
\end{equation*}
{where $\mathcal{Y}$ denotes the set of labels, $X_{\rm train}^{Y=i} = (X_j : Y_j = i, j\in[\ell])$ denotes the subset of $X_1, \dots, X_\ell$ for which $Y_j = i$ for $j\in[\ell]$, and likewise $X_{\rm cal+test}^{Y=i} = (X_{\ell+j} : Y_{\ell+j} = i, j\in[n-\ell])\cup(X_{j}^k : Y_{j}^k = i, j\in[m], k\in[K])$.}
Specifically, $\hat{s}_i(X, X_{\rm train}^{Y=i}, X_{\rm cal+test}^{Y=i})$ denotes the predicted probability that $(X, i)$ has label $1$, i.e., is an inlier, using \gls{lr} fitted to $X_{\rm train}^{Y=i}$ with labels $1$ and $X_{\rm cal+test}^{Y=i}$ with labels $-1$.

As a classifier we consider \gls{svc} \citep{Boser1992:SVC,Cortes1995:SVC} for simplicity and robustness, again using the standard implementation in \textit{scikit-learn} \citep{sklearn2011}.

% We consider the following baselines: (i) \emph{oracle 1:} the complete oracle knows the true contamination factors and exactly which datapoints are contaminated; (ii) \emph{oracle 2:} the partial oracle knows the true contamination factors but not which datapoints are contaminated; and (iii) \emph{random:} we randomly select which agents to collaborate with after the first round.

We consider the following baselines: (i) \emph{Oracle 1:} The complete oracle knows the true contamination factors and exactly which datapoints are contaminated. Accordingly, it selects for collaboration in the second round the data agents with data contamination factor $\pi_{k}$ below a specified threshold $\pi_{\rm th}$, and then uses only the in-distribution data for subsequent model training. (ii) \emph{Oracle 2:} The partial oracle knows the true contamination factors but not which datapoints are contaminated. Accordingly, it selects for collaboration in the second round the data agents with data contamination factor $\pi_{k}$ below a specified threshold $\pi_{\rm th}$, and then uses all collected data for subsequent model training. (iii) \emph{Random:} The naïve baseline randomly selects $K_{\rm budget}$ agents to collaborate with in the second round, and then uses all collected data for subsequent model training. {The proposed method is benchmarked either without data subset selection (\emph{Storey}/\emph{Quantile}/\emph{Fisher}/\emph{Sum}), or using conformal outlier detection with Storey's BH procedure (\emph{Storey}$^*$/\emph{Quantile}$^*$/\emph{Fisher}$^*$/\emph{Sum}$^*$), setting the significance level to $\beta=0.1$ and Storey's hyperparameter to $\zeta=0.5$.}

% We consider the following baselines:
% \begin{enumerate}[label=({\roman*})]
%     \item \emph{Oracle 1:} The complete oracle knows the true contamination factors and exactly which datapoints are contaminated. {Accordingly, it selects for collaboration in the second round the data agents with data contamination factor $\pi_{k}$ below a specified threshold $\pi_{\rm th}$, and then uses only the in-distribution data for subsequent model training.}
%     \item \emph{Oracle 2:} The partial oracle knows the true contamination factors but not which datapoints are contaminated. {Accordingly, it selects for collaboration in the second round the data agents with data contamination factor $\pi_{k}$ below a specified threshold $\pi_{\rm th}$, and then uses all collected data for subsequent model training.}
%     \item \emph{Random:} {The naïve baseline randomly selects $K_{\rm budget}$ agents to collaborate with in the second round, and then uses all collected data for subsequent model training.}
% \end{enumerate}

In Section~\ref{sec:Gaussian_data} of the appendix we provide an ablation study with synthetic Gaussian data. We also provide additional numerical experiments with a variety of conformal scores for the data contamination scenario presented here, as well as seven other data contamination scenarios, in Section~\ref{sec:real_datasets}, herein, also validating scalability to larger dataset sizes in Section~\ref{subsecSupp:scalability}.

\subsection{Conformal data contamination tests}
In this section the conformal data contamination tests are numerically analyzed on the retinal fundus images. We showcase in Figure~\ref{fig:retina_AdaDetect_SVC_CDCT} the (empirical) \gls{cdf} of the conformal data contamination p-values with $\pi_{\rm th} = 0.1$ and for $\pi=0.1,0.3$. From Figure~\ref{subfig:retina_AdaDetect_SVC_CDCT_pi0.1} we observe that the conformal data contamination tests are superuniform under the null, as dictated by the theory in Section~\ref{sec:results}, and notice that the Quantile p-value is nearly standard uniform, meanwhile the Fisher p-value is the most conservative in this case.
Figure~\ref{subfig:retina_AdaDetect_SVC_CDCT_pi0.3} shows that in this case the Storey test is the most powerful followed by the Sum and Fisher tests, and finally the weakest is the Quantile test.
% {We report in Table~\ref{tab:retina_AdaDetect_SVC_FDR_TDR} the (empirical) \gls{tdr} and \gls{fdr} for the conformal data contamination with varying data contamination thresholds $\pi_{\rm th} \in \{0, 0.1, 0.2\}$ at significance level $\alpha=0.1$ when using Storey's \gls{bh} procedure with Storey's hyperparameter $\gamma=0.5$. We observe the highest \gls{tdr} with the Storey test and a generally low \gls{tdr} with the Quantile test, while on the other hand, the Quantile test tends to be the least conservative. Notably, the Storey and Fisher tests perform well at $\pi_{\rm th}$ but tend to decrease rapidly in both \gls{tdr} and \gls{fdr} as the data contamination threshold increases, meanwhile the decrease is less pronounced with the Quantile and Sum tests.}

%
\begin{figure}[t]
    \centering
    \begin{minipage}{0.45\linewidth}
        \centering
        \subfloat[$\pi=0.1$.]{\includegraphics[width=1\linewidth, page=15]{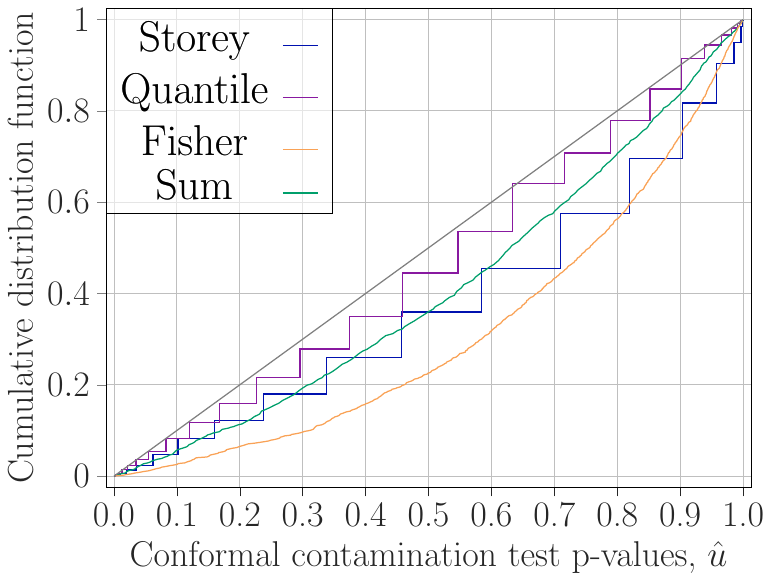}\label{subfig:retina_AdaDetect_SVC_CDCT_pi0.1}}
    \end{minipage}\hspace{2pt}
    \begin{minipage}{0.45\linewidth}
        \centering
        \subfloat[$\pi=0.3$.]{\includegraphics[width=1\linewidth, page=16]{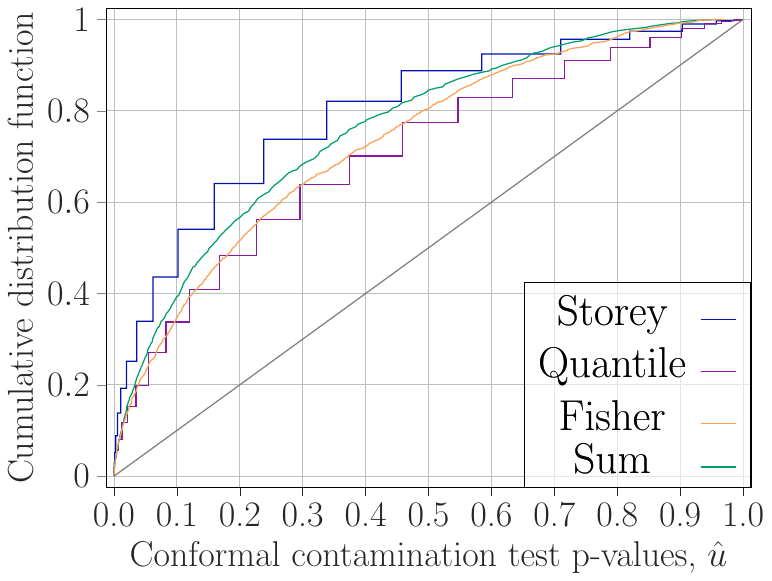}\label{subfig:retina_AdaDetect_SVC_CDCT_pi0.3}}
    \end{minipage}
    \caption{Empirical CDF of conformal data contamination p-values with $\pi_{\rm th} = 0.1$ for (a) $\pi=0.1$ and (b) $\pi=0.3$. The gray line is the CDF of a standard uniform.}
    \label{fig:retina_AdaDetect_SVC_CDCT}
    \vspace{-0.4cm}
\end{figure}

We report in Table~\ref{tab:retina_AdaDetect_SVC}(a) the (empirical) \gls{tdr} for the conformal data contamination with varying data contamination thresholds $\pi_{\rm th} \in \{0, 0.1, 0.2\}$ at significance level $\alpha=0.1$ when using Storey's \gls{bh} procedure with Storey's hyperparameter $\gamma=0.5$. We observe the highest \gls{tdr} with the Storey test and a generally low \gls{tdr} with the Quantile test, and notably, the Storey and Fisher tests perform well at low $\pi_{\rm th}$ but tend to decrease rapidly in \gls{tdr} as the data contamination threshold increases, while the decrease is less pronounced with the Quantile and Sum tests.

In Table~\ref{tab:retina_AdaDetect_SVC}(b) we numerically validate the \gls{fdr} control provided by the conformal data contamination tests on the boundary of the null hypothesis. Here, $K_0=K/2$ data agents have data contamination factor $\pi_k = \pi_{\rm th}$, and the remaining data agents have $\pi_k = 1-\pi_{\rm th}$ so approximately half of the test data is contaminated. Through this construction, we are on the boundary of the null hypothesis, i.e., there are no $\pi_k < \pi_{\rm th}$ contributing to excess conservativeness from being in the interior of the null hypothesis. Once again, we used Storey's \gls{bh} procedure at significance level $\alpha=0.1$ and with Storey's hyperparameter $\gamma=0.5$. The results show that the tests effectively controls the \gls{fdr} at the nominal level, and we notice that the Quantile test is the least conservative among the four tests.

\setlength{\tabcolsep}{0.5em} % for the horizontal padding
\begin{table}
    \centering
    \caption{${\rm TDR}$ and ${\rm FDR}$ computed from $2000$ simulations for varying data contamination thresholds, $\pi_{\rm th}$, at significance level $\alpha=0.1$ using Storey's BH procedure with parameter $\gamma=0.5$. (a) ${\rm TDR}$ with data contamination factor simulated as $\pi_k\sim {\rm Uniform}([0,1])$, and (b) ${\rm FDR}$ where for $K/2$ of the data agents $\pi_k\!=\!\pi_{\rm th}$ while for the others $\pi_k\!=\!1 - \pi_{\rm th}$.\vspace{5pt}}
    \footnotesize
    \begin{tabular}{@{}ll|rrr|ll|rrr@{}}
        \toprule
               &   & $\pi_{\rm th}\!=\!0$ & $\pi_{\rm th}\!=\!0.1$ & $\pi_{\rm th}\!=\!0.2$ &  &  & $\pi_{\rm th}\!=\!0$ & $\pi_{\rm th}\!=\!0.1$ & $\pi_{\rm th}\!=\!0.2$ \\\midrule
        (a) ${\rm TDR}$ & Storey    & 0.9128 & 0.8147 & 0.7033 & (b) ${\rm FDR}$ & Storey    & 0.1064 & 0.0708 & 0.0489 \\
        &Quantile  & 0.8104 & 0.7851 & 0.7537 & &Quantile  & 0.1120 & 0.1074 & 0.1050 \\
        &Fisher    & 0.8890 & 0.7576 & 0.6077 & &Fisher    & 0.0641 & 0.0281 & 0.0153\\
        &Sum       & 0.8778 & 0.7994 & 0.7142 & &Sum       & 0.0890 & 0.0658 & 0.0458\\
    \end{tabular}
    \label{tab:retina_AdaDetect_SVC}
    \vspace{-0.1cm}
\end{table}
%

% Based on the numerical experiments here, and in the appendix, we give the following recommendations: For small calibration sets (e.g. $n-\ell<100$), Storey or Quantile statistics are preferable due to finite-sample guarantees. However, those two statistics rely on choosing a hyperparameter which is not easily done in practice. With sufficient calibration data, Fisher or Sum statistics are recommended as they are hyperparameter free. Finally, the Storey and Fisher statistics suit low contamination thresholds, while the Quantile and Sum statistic are more appropriate for higher thresholds.

\subsection{Collaborative data sharing}
We show in Figure~\ref{fig:AccuracyMain} the relation between the average number of collaborating data agents, $\mathbb{E}[|\hat{\mathcal{H}}_0|]$, and the average classification accuracy across $500$ data simulations for both data contamination scenarios, with various hyperparameter settings $(\pi_{\rm th}, \alpha) \in \{0, 0.1, 0.2, 0.3, 0.4, 0.5, 0.6, 0.7\} \times \{0.05, 0.2, 0.5, 0.7\}$. The proposed data sharing procedures {without data subset selection} outperform no data sharing and the \emph{random} baseline, even reaching the performance of \emph{oracle 2}, showing the efficiency of the proposed methodology, and highlighting the accuracy improvements achieveable by carefully selecting data agents for collaboration.

% Further, the accuracy gap between \emph{oracle 1} and \emph{oracle 2} indicates that additional accuracy improvements can be achieved by using data subset selection techniques, although this has not been considered in this work, as our focus is the proposed conformal data contamination tests.

%
\begin{figure}[t]
    \centering
    \begin{minipage}{0.45\linewidth}
        \centering
        \subfloat[Retinal fundus images. Accuracy without data sharing: $47.77~\%$.]{\includegraphics[width=1\linewidth, page=40]{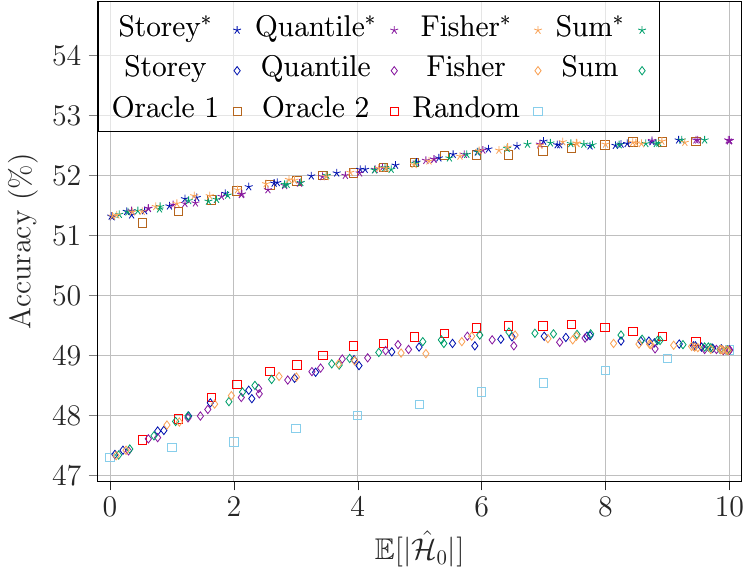}\label{subfig:AccuracyMain_1}}
    \end{minipage}\hspace{2pt}
    % \begin{minipage}{0.32\linewidth}
    %     \centering
    %     \subfloat[{Blood smear images. Accuracy without data sharing: $69.47~\%$.}]{\includegraphics[width=1\linewidth, page=19]{figures/__P7__Figures.pdf}\label{subfig:AccuracyMain_2}}
    % \end{minipage}
    \begin{minipage}{0.45\linewidth}
        \centering
        \subfloat[Handwritten digits. Accuracy without data sharing: $90.75~\%$.]{\includegraphics[width=1\linewidth, page=41]{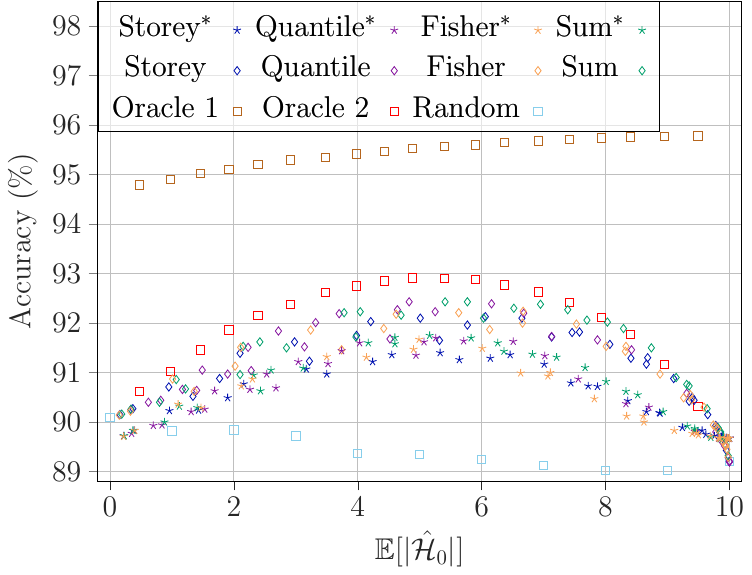}\label{subfig:AccuracyMain_3}}
    \end{minipage}
    \caption{Plot of the classification accuracy against the average number of collaborating data agents. Results with seven other datasets are available in Section~\ref{sec:real_datasets}.}
    \label{fig:AccuracyMain}
    \vspace{-0.3cm}
\end{figure}

The biggest gains in accuracy {without data subset selection} can be observed in Figure~\ref{subfig:AccuracyMain_3} showing the case of label noise on the MNIST data. Specifically, at $\mathbb{E}[|\hat{\mathcal{H}}_0|] \approx 6$ we achieve an increase in accuracy of up to $3$ percentage points with the proposed methods compared to the \emph{random} baseline. For the other scenario, at $\mathbb{E}[|\hat{\mathcal{H}}_0|] \approx 6$ we achieve an increase in accuracy just below $1$ percentage point.

{When including data subset selection via the lightweight conformal outlier detection, for the retinal fundus image scenario, the performance of the proposal reaches even that of \emph{oracle 1}, however, the same is not observed for the label noise scenario, since the outlier detection problem proved more difficult for that data scenario. This shows that if the conformal score function is sufficiently powerful, conformal outlier detection can practically be used as a lightweight data subset selection technique.}

\section{Conclusion}\label{sec:conclusion}
We have presented a data sharing framework which selects collaboration partners using novel distribution-free testing procedures, named \emph{conformal data contamination tests}. Leveraging conformal p-values, our method detects agents whose data exceed a contamination threshold while providing false discovery rate guarantees. This enables strategic acquisition of high-quality and personalized data while avoiding irrelevant or harmful data sources. Experiments validate the effectiveness and practicality of our approach.

Our framework opens up several promising directions for future work, including integration with incentive mechanisms in real-world data markets, {study of advanced adversarial attacks,} and extension to continual learning with online data acquisition policies.

\section*{Broader Impact Statement}
This paper presents work whose goal is to advance the field of Machine Learning. There are many potential societal consequences of our work, none which we feel must be specifically highlighted here.

% Only add this information once your submission is accepted and deanonymized. 

% \putbib[bib]
% \end{bibunit}

% \newpage
\setcounter{figure}{0}
\renewcommand{\thefigure}{A\arabic{figure}}%
\setcounter{theorem}{0}
\renewcommand{\thetheorem}{A\arabic{theorem}}%
\setcounter{table}{0}
\renewcommand{\thetable}{A\arabic{table}}%
\renewcommand{\thesection}{A\arabic{section}}
\setcounter{equation}{0}
\renewcommand{\theequation}{A.\arabic{equation}}%

\newpage
\appendix
% \onecolumn
% \begin{bibunit}[tmlr]

\section{Conformal data contamination tests: theoretical details and proofs}\label{sec:CCTests}

We consider the scenario where we have observed a null sample $\mathcal{D}_{\rm null} = \{Z_i\}_{i=1}^n$, $Z_i \sim P_0$, as well as a test sample $\mathcal{D}_{\rm test} = \{Z_{n+j}\}_{j=1}^m$, $Z_{n+j} \sim P_j$. We call data from $P_0$ inliers and data from some other distribution $P_1$ outliers.
We define $\bar{\mathcal{H}}_0 = \{j \in [m] : Z_{n+j} \sim P_0\}$, and $\bar{\mathcal{H}}_1 = [m]\setminus\bar{\mathcal{H}}_0$.
Following the conformal outlier detection procedure as described in Section~\ref{subsec:conformal_outlier_detection}, conformal p-values are computed $\hat{p}_1,\dots,\hat{p}_m$. As in Section~\ref{subsec:conformal_outlier_detection} we shall assume that the conformal score is continuously distributed, or almost surely has no ties, and that the inliers are exchangeable conditioned on the outliers \citep{Bates2023:Testing}.

Two classical estimators of the contamination factor $\pi$ are the Storey estimator of \cite{Storey2003:Strong} and the quantile estimator of \cite{Benjamini2006:Adaptive} defined respectively as
\begin{align*}
    \hat{\pi}^{\rm storey} = 1 - \frac{\sum_{i=1}^m \mathbbm{1}[\hat{p}_i > \lambda]}{m (1-\lambda)},\\
    \hat{\pi}^{\rm quantile} = 1 - \frac{i_0 + 1}{m (1-\hat{p}_{(m-i_0)})},
\end{align*}
where $\hat{p}_{(1)}\leq\cdots \leq \hat{p}_{(m)}$ are the ordered conformal p-values. Here $\lambda \in [{n-\ell}]/({n-\ell}+1)$ and $i_0 \in [m-1]_0$ are hyperparameters controlling the bias-variance trade-off. For these two estimators, the parts related to the data are $\sum_{i=1}^m \mathbbm{1}[\hat{p}_i > \lambda]$ and $\hat{p}_{(m-i_0)}$, respectively, and so in the following we will use these as test statistics, noting that small values of these tests statistics will yield large estimates of $\pi$ and so small values are evidence against the null hypothesis $H_0 : \pi \leq \pi_{\rm th}$.

In the following, we give a detailed proof of Theorem~\ref{thm:nonasymptotic_tests}.

\begin{proof}[Proof of Theorem~\ref{thm:nonasymptotic_tests}]
    Since the test statistic $T^{\rm storey}$ is left-sided, its exact p-value is $\mathbb{P}_{H_0}(T^{\rm storey} \leq t^{\rm storey})$ where $t^{\rm storey}$ denotes the observed value of the random variable $T^{\rm storey}$. By the law of total probability
    \begin{align}
        \mathbb{P}_{H_0}(T^{\rm storey} \leq t^{\rm storey}) &= \sum_{k = 0}^m \mathbb{P}_{H_0}(k) \mathbb{P}\Big(\sum_{i=1}^{k} \mathbbm{1}[\hat{p}_i > \lambda] + \sum_{i=k+1}^m \mathbbm{1}[\hat{p}_i > \lambda] \leq t^{\rm storey}~\big|~\bar{\mathcal{H}}_0 = [k]\Big) \nonumber\\
        &\leq \sup_{\pi_0 \in [0, \pi_{\rm th}]} \sum_{k = 0}^m {\rm B}_{\pi_0}^{m}(k) \mathbb{P}\Big(\sum_{i=1}^{k} \mathbbm{1}[\hat{p}_i > \lambda] \leq t^{\rm storey}~\big|~\bar{\mathcal{H}}_0 = [k]\Big). \label{eq:pi0sup}
    \end{align}
    The inequality is tight when $\hat{p}_i \leq \lambda$ for all $i \in \mathcal{H}_1$ which with a well chosen $\lambda$ can occur if the conformal score separates well the inliers from the outliers.
    The distribution of the null data points is binomial with success probability $\pi_0 \in [0, \pi_{\rm th}]$,
    \begin{equation*}
        {\rm B}_{\pi_0}^{m}(k) = \binom{m}{k}(1-\pi_0)^k\pi_0^{m-k}.
    \end{equation*}
    The probability of the sum can be expressed in terms of the marginal \gls{cdf} of the ordered conformal p-values as
    \begin{equation*}
        \mathbb{P}\Big(\sum_{i=1}^{k} \mathbbm{1}[\hat{p}_i > \lambda] \leq t^{\rm storey}~\big|~\bar{\mathcal{H}}_0 = [k]\Big) = 
        \begin{cases}
            1 & \text{if } k \leq t^{\rm storey},\\
            \mathbb{P}\big(\hat{p}_{(k - t^{\rm storey})} \leq \lambda\big) & \text{otherwise},
        \end{cases}
    \end{equation*}
    where $\hat{p}_{(1)}\leq \hat{p}_{(2)} \leq \cdots \leq \hat{p}_{(k)}$ are the ordered conformal p-values among the $k$ p-values from the null. Notably, $\hat{p}_{(k - t^{\rm storey})}$ is distributed according to a negative hypergeometric distribution with population size ${n-\ell}+k$, number of success states ${n-\ell}$, and number of failures $k-r$ supported on $\{1, 2, \dots, {n-\ell}+1\}${, see Corollary A.1 of} \cite{Biscio2025:Conformal}, and the \gls{cdf} can be expressed as
    \begin{align*}
        \mathbb{P}\big(\hat{p}_{(k - t^{\rm storey})} \leq \lambda\big) &= F_{\rm NHG}\big(\lfloor \lambda ({n-\ell}+1)\rfloor - 1; {n-\ell}+k, {n-\ell}, k-t^{\rm storey}\big).
        % &=\frac{n!k!}{(n+k)!}\frac{(\lfloor \lambda (n+1)\rfloor+k-t^{\rm storey}-2)!(n-\lfloor \lambda (n+1)\rfloor+t^{\rm storey}+1)!}{(\lfloor \lambda (n+1)\rfloor-1)!(n-\lfloor \lambda (n+1)\rfloor+1)!(k-t^{\rm storey}-1)!t^{\rm storey}!}
        % \pFq{3}{2}{1,-x,n-\lfloor \lambda (n+1)\rfloor+t^{\rm storey}+2}{n-\lfloor \lambda (n+1)\rfloor+2,2-\lfloor \lambda (n+1)\rfloor-k+t^{\rm storey}}{1},
    \end{align*}
    % where
    % $x := \lfloor \lambda (n+1)\rfloor-1$, and
    % ${}_3 F_2$ is the generalized hypergeometric function.
    Finally, observe that the supremum in Eq.~\eqref{eq:pi0sup} occurs at $\pi_{\rm th}$, since $\sum_{k=0}^l {\rm B}_{\pi_0}^{m}(k)$, $l \in \{0,\dots,m\}$ is a non-decreasing function in $\pi_0$, and $\mathbb{P}(\sum_{i=1}^{k} \mathbbm{1}[\hat{p}_i > \lambda] \leq t^{\rm storey}~|~\bar{\mathcal{H}}_0 = [k])$ is a decreasing function of $k$ as the indicator function is non-negative. This concludes the derivation of the p-value $\hat{u}^{\rm storey}$

    The proof of the validity of $\hat{u}^{\rm quantile}$ follows immediately by noticing that for $x \in [{n-\ell}+1]/({n-\ell}+1)$
    \begin{equation*}
        \mathbb{P}\big(\hat{p}_{(m-i_0)} \leq x\big) = \mathbb{P}\Big(\sum_{i=1}^m \mathbbm{1}\big[\hat{p}_i > x\big] \leq i_0\Big),
    \end{equation*}
    and then invoking the same arguments as above.
\end{proof}

\begin{remark}
    A possible generalization would be to consider a rejection region on a vector of the ordered p-values, for instance considering two of the ordered p-values rather than just one. In such a case, deriving the upper bound would require evaluating the pairwise distribution of the ordered conformal p-values, and the p-value would depend on two hyperparameters rather than one.
\end{remark}

\cite{Bates2023:Testing} showed a general result regarding the asymptotic distribution of test statistics of the form $\sum_{i=1}^m G(p_i)$, for some general class of functions $G$. In their work, this was used to formulate a correction to the Fisher combination test yielding a valid testing procedure using conformal p-values for the special case of $\pi_{\rm th} = 0$. In the following theorem, we {generalize} their result to the setting of this paper, thereby paving the way for constructing more general test statistics which are valid asymptotically.
\begin{theorem}\label{thm:appendix_asymp_test}
    Let $T_G=\sum_{i=1}^m G(\hat{p}_i)$ be a test statistic for conformal p-values $\hat{p}_1,\dots,\hat{p}_m$ and an increasing function $G: [0, 1] \to [0, \infty)$ satisfying
    \begin{enumerate}
        \item[(i)] $\int_0^1 G^{2+\eta}(u)du < \infty$;
        \item[(ii)] $|\frac{1}{{n-\ell}+1} \sum_{j=1}^{{n-\ell}+1} G^k(j/({n-\ell}+1)) - \int_0^1 G^k(u)du| = o(1/\sqrt{{n-\ell}})$, for $k \in \{1, 2\}$;
        \item[(iii)] $\max_{j\in\{1,\dots,{n-\ell}+1\}}~G(j/({n-\ell}+1)) = o(\sqrt{{n-\ell}})$.
    \end{enumerate}
    Then, under the null hypothesis $H_0 : \pi \leq \pi_{\rm th}$, $\pi_{\rm th} \in [0, 1)$, if $m = \lfloor\gamma {(n-\ell)}\rfloor$ for some $\gamma > 0$, as ${n-\ell} \to \infty$
    \begin{equation}\label{eq:upper_bound_asymptotic}
        \begin{split}
            \mathbb{P}_{H_0}(T_G \leq r) &\leq \pi_{\rm th}^{m} + \sum_{k = 1}^m {\rm B}_{\pi_{\rm th}}^{m}(k) F_{G^k}\Bigg(\frac{r + k(\sqrt{1+\gamma_k} - 1)\int_0^1 G(u){\rm d}u}{\sqrt{1+\gamma_k}}\Bigg),
        \end{split}
    \end{equation}
    where $F_{G^k}$ is the \gls{cdf} of $\sum_{i=1}^k G(U_i)$, $U_i \stackrel{iid.}{\sim} {\rm Unif}([0, 1])$, and $\gamma_k = k/{(n-\ell)}$.
    Moreover, if $\mathbb{P}(G(\hat{p}_i) = 0, \forall i \in \bar{\mathcal{H}}_1) = 1$ and $\pi = \pi_{\rm th}$ the inequality is exact.
\end{theorem}
\begin{proof}
    By the law of total probability
    \begin{align}
        \mathbb{P}_{H_0}\Big(\sum_{i=1}^m G(\hat{p}_i) \leq r\Big) &= \sum_{k = 0}^m \mathbb{P}_{H_0}(k) \mathbb{P}\Big(\sum_{i=1}^{k} G(\hat{p}_i) + \sum_{i=k+1}^m G(\hat{p}_i) \leq r~\big|~\bar{\mathcal{H}}_0 = [k]\Big) \nonumber\\
        &\leq \sup_{\pi_0 \in [0, \pi_{\rm th}]} \sum_{k = 0}^m {\rm B}_{\pi_0}^{m}(k) \mathbb{P}\Big(\sum_{i=1}^{k} G(\hat{p}_i) \leq r~\big|~\bar{\mathcal{H}}_0 = [k]\Big),\label{eq:fisher_ineq}
    \end{align}
    Now, by Theorem S1 of \cite{Bates2023:Testing}, when $k = \lfloor \gamma_k {(n-\ell)}\rfloor$ we have that as ${n-\ell} \to \infty$
    \begin{equation*}
        \mathbb{P}\Big(\sum_{i=1}^{k} G(\hat{p}_i) \leq r~\big|~\bar{\mathcal{H}}_0 = [k]\Big) \to q_k^*,
    \end{equation*}
    where $q_k^*$ is such that $r = \sqrt{1+\gamma_k} Q_{G^k}(q_k^*) - k(\sqrt{1+\gamma_k}-1)\int_0^1 G(u){\rm d}u$, and $Q_{G^k}(q_k^*)$ is the $q_k^*$ quantile of $\sum_{i=1}^k G(U_i)$. Isolating $q_k^*$ yields
    \begin{equation*}
        q_k^* = F_{G^k}\Bigg(\frac{r + k(\sqrt{1+\gamma_k} - 1)\int_0^1 G(u){\rm d}u}{\sqrt{1+\gamma_k}}\Bigg).
    \end{equation*}
    It follows that
    \begin{align*}
        \lim_{{n-\ell}\to\infty} \mathbb{P}_{H_0}\Big(\sum_{i=1}^m G(\hat{p}_i) \leq r\Big)
        %&\leq \sup_{\pi_0 \in [0, \pi_{\rm th}]} \Bigg(\sum_{k = 1}^m {\rm B}_{\pi_0}^{m}(k) F_{G^k}\Bigg(\frac{r + k(\sqrt{1+\gamma_k} - 1)\int_0^1 G(u){\rm d}u}{\sqrt{1+\gamma_k}}\Bigg).\\
        %&\qquad\qquad\qquad + {\rm B}_{\pi_0}^{m}(0)\Bigg)\\
        &\leq \pi_{\rm th}^{m} + \sum_{k = 1}^m {\rm B}_{\pi_{\rm th}}^{m}(k) F_{G^k}\Bigg(\frac{r + k(\sqrt{1+\gamma_k} - 1)\int_0^1 G(u){\rm d}u}{\sqrt{1+\gamma_k}}\Bigg),
    \end{align*}
    where we use that $q_k^*$ is a decreasing sequence.
    % \mv{Actually, we can only use the asymptotic regime when $k \geq k_0$ for some $k_0 = \lfloor \gamma_{k_0} n\rfloor$. The result will still work when ${\rm B}(k, \pi_0) \approx 0$ for $k < k_0$. When this is not the case, one may approximate $F_{G^k}$ numerically?}
\end{proof}
It follows from the theorem that an asymptotically valid p-value is
\begin{equation}
    \hat{u}_G = \pi_{\rm th}^{m} + \sum_{k = 1}^m {\rm B}_{\pi_{\rm th}}^{m}(k) F_{G^k}\Bigg(\frac{T_G + k(\sqrt{1+\gamma_k} - 1)\int_0^1 G(u){\rm d}u}{\sqrt{1+\gamma_k}}\Bigg).
\end{equation}
\begin{remark}
    We have altered the formulation slightly compared to \cite{Bates2023:Testing} since we want the function $G$ to map into non-negative numbers and being an increasing function such that a small value of the test statistic is evidence against the null, thereby allowing the inequality of Eq.~\eqref{eq:fisher_ineq}. Moreover, we want $G$ to be such that $\mathbb{P}(G(\hat{p}_i) \leq x, \forall i \in \bar{\mathcal{H}}_1)$ is large for small $x$, which is an important property to have a tight approximation in Eq.~\eqref{eq:fisher_ineq}.
\end{remark}
The following remarks, in combination with Theorem~\ref{thm:appendix_asymp_test}, provides the proof of Theorem~\ref{thm:asymptotic_fisher_test}.
\begin{remark}
    A variant of the Fisher combination test uses $G(\hat{p}) = -2{\rm log}\big(\frac{1}{{n-\ell}+1}\big) + 2{\rm log}(\hat{p})$, thereby maintaining the shape of the typical Fisher combination function, while making it an increasing function. The result of \cite{Bates2023:Testing} can be exploited in this case by noticing that
    \begin{equation*}
        \mathbb{P}\Big(-2k{\rm log}\Big(\frac{1}{{n-\ell}+1}\Big) + 2 \sum_{i=1}^k{\rm log}(\hat{p}_i) \leq r\Big) = 1 - \mathbb{P}\Big(-2\sum_{i=1}^k {\rm log}(\hat{p}_i) \leq -2k{\rm log}\Big(\frac{1}{{n-\ell}+1}\Big) - r\Big).
    \end{equation*}
    Particularly, it holds that using this test statistic is equivalent to using the typical Fisher test statistic when $\pi_{\rm th} = 0$. The technical conditions (i)-(iii) regarding the function $G$ can be verified, see Remark S1 of \cite{Bates2023:Testing} for details.
\end{remark}
\begin{remark}
    A Sum combination test uses $G(\hat{p}) = \hat{p}$ for which $\int_0^1 G(u) {\rm d}u = 1/2$ and $G(U) \stackrel{d}{=} {\rm Unif}([0, 1])$, resulting in $\sum_{i=1}^k G(U_i) \stackrel{d}{=} {\rm IH}(k)$ where $U_i \stackrel{iid.}{\sim} {\rm Unif}([0, 1])$ and ${\rm IH}(k)$ is the Irwin-Hall distribution with parameter $k$. The technical conditions (i)-(iii) regarding the function $G$ can be verified: condition (i) trivially holds since a polynomial of finite order is finite; for condition (ii) we note that $G(u)$ is increasing and $G'(u) = 1$, thus for $k \in \{1, 2\}$
    \begin{align*}
        \Big|\frac{1}{{n-\ell}+1} \sum_{j=1}^{{n-\ell}+1} \frac{j^k}{({n-\ell}+1)^k} - \int_0^1 u^k {\rm d}u\Big| &\leq \sum_{j=1}^{{n-\ell}+1} \Big|\frac{1}{{n-\ell}+1}\frac{j^k}{({n-\ell}+1)^k} - \int_{(j-1)/({n-\ell}+1)}^{j/({n-\ell}+1)} u^k {\rm d}u\Big|\\
        &\leq \sum_{j=1}^{{n-\ell}+1} \int_{(j-1)/({n-\ell}+1)}^{j/({n-\ell}+1)} \Big|\frac{j^k}{({n-\ell}+1)^{k+1}} - u^k\Big|{\rm d}u\\
        &\leq \sum_{j=1}^{{n-\ell}+1} \frac{kj^{k-1}}{({n-\ell}+1)^{k+1}} = {\rm O}\Big(\frac{1}{{n-\ell}}\Big)
    \end{align*}
    where the first two inequalities are due to the triangle inequality, and the third inequality follows by the mean value theorem and the chain rule; finally condition (iii) immediately holds as $\max_{j\in\{1,\dots,{n-\ell}+1\}}~G(j/({n-\ell}+1)) = 1$.
\end{remark}

\subsection{Multiple testing with combination test p-values}\label{sec:multitest}
For completeness of presentation we begin by defining the \gls{prds} property \citep{Benjamini2001:Dependency}.
% %
\begin{definition}
    A set $\mathcal{D}$ is called non-decreasing if $x \in \mathcal{D}$ and $x \leq y$ implies $y \in \mathcal{D}$.
\end{definition}
\begin{definition}
    For any non-decreasing set $\mathcal{D}$, and for each $i \in I_0$ such that $\mathbb{P}(X \in \mathcal{D} | X_i = x)$ is non-decreasing in $x$, then $X$ is \gls{prds} on $I_0$.
\end{definition}
Theorem~\ref{thm:Storey_PRDS} states that the Storey conformal data contamination p-values are \gls{prds}. This is a result motivated by Theorem 2.4 of \cite{Bates2023:Testing} stating that conformal p-values are \gls{prds}. The implication of Theorem~\ref{thm:Storey_PRDS} is that we can use the \gls{bh} procedure with the conformal data contamination p-values and maintain \gls{fdr} control.
%
% \begin{theorem}\label{thm:Storey_PRDS_supp}
%     Let $\hat{p}_1^{\rm storey}, \dots, \hat{p}_K^{\rm storey}$ be $K$ Storey conformal data contamination p-values as in \eqref{eq:storey_p_value}, with the $k$-th derived from conformal p-values $\hat{p}_{k, 1}, \dots, \hat{p}_{k, m_k}$. The Storey conformal data contamination p-values are \gls{prds}.
% \end{theorem}
\begin{proof}[Proof of Theorem \ref{thm:Storey_PRDS}]
    % Let $Z = (\hat{s}_{(1)}, \dots, \hat{s}_{(n)})$ be the order statistics of the conformal scores on the calibration data.
    Let $F$ be a shorthand for the function in {Eq.}~\eqref{eq:storey_p_value}, i.e., $\hat{u}_k^{\rm storey} = F(\hat{p}_{1}^k, \dots, \hat{p}_{m_k}^k)$. It follows that $F$ is entry-wise monotone increasing (not strictly), i.e., $F\big(\hat{p}_{1}^k, \dots, (\hat{p}_{j}^k)', \dots, \hat{p}_{m_k}^k\big) \geq F\big(\hat{p}_{1}^k, \dots, \hat{p}_{j}^k, \dots, \hat{p}_{m_k}^k\big)$ where $(\hat{p}_{j}^k)' \geq \hat{p}_{j}^k$ for any $j\in\{1,\dots,m_k\}$. Let $Y = \big(\hat{u}_1^{\rm storey}, \dots, \hat{u}_K^{\rm storey}\big)$ be the Storey conformal data contamination p-values on the test sets. Denote by $X_k = \big(\hat{p}_{1}^k, \dots, \hat{p}_{m_k}^k\big)$ the conformal p-values on the $k$-th test set, and by $X = (X_1, \dots, X_k)$ the total conformal p-values on the test sets.

    Let $y \geq y'$ and let $A$ be an increasing set. Conditioning on the $k$-th conformal data contamination p-value and using the law of total probability
    \begin{equation*}
        \mathbb{P}(Y \in A | Y_k = y) = \sum_{x \in \mathcal{X}} \frac{\mathbb{P}(Y \in A, Y_k = y, X_k=x)}{\mathbb{P}(Y_k = y)},
    \end{equation*}
    where $\mathcal{X} = \{1/({n-\ell}+1),\dots, 1\}^{m}$. Now, by conditional independence of $Y_{-k}=(Y_1,\dots,Y_{k-1},Y_{k+1},\dots,Y_K)$ and $Y_k$
    \begin{equation*}
        \mathbb{P}(Y \in A | Y_k = y) = \sum_{x \in \mathcal{X}} \mathbb{P}(Y \in A | X_k=x) \mathbb{P}(X_k = x | Y_k = y).
    \end{equation*}
    We know that $\mathbb{P}(Y_k | X_k=x) = \mathbbm{1}[Y_k = F(x)]$, and so $\mathbb{P}(X_k | Y_k = y)$ is only non-zero when $y = F(x)$. Define $S(y) = \{x \in \mathcal{X} : y = F(x)\}$.
    \begin{equation*}
        \mathbb{P}(Y \in A | Y_k = y) = \frac{\sum_{x \in S(y)} \mathbb{P}(X \in B | X_k=x)}{|S(y)|},
    \end{equation*}
    where $B = \{\bar{x}\in \mathcal{X}^K : [F(\bar{x}_1), \dots, F(\bar{x}_K)] \in A\}$. Since $A$ is an increasing set, and $F$ is an increasing function, $B$ is also an increasing set.
    By Theorem 2.4 of \cite{Bates2023:Testing}, conformal p-values are \gls{prds}, and as a corollary to this, we have that $\mathbb{P}(X \in B | X_k = x) \geq \mathbb{P}(X \in B | X_k = x')$ where $x \succeq x'$ ($\succeq$ denotes entry-wise inequality).

    We define some notations: let $S^{\rm vec}(y)$ denote a vectorization of the set $S(y)$ such that $S^{\rm vec}(y) = (\Tilde{x}^{S(y)}_{1}, \dots, \Tilde{x}^{S(y)}_{|S(y)|})$ with $\Tilde{x}^{S(y)}_i \in S(y)$, $\Tilde{x}^{S(y)}_i \neq \Tilde{x}^{S(y)}_j$, $i\neq j$, $i, j \in [|S(y))|]$; let also $\bar{S}^{\rm vec}(y)$ denote the vectorization of $\bar{S}(y) \subset S(y)$; finally let $S^{\rm c}(y)$ denote the complement of $S(y)$.

    Assume initially that $|S(y')| = |S(y)|$. Then, $S(y)$ dominates $S(y')$ since $F$ is an entry-wise increasing function, i.e., there exists a permutation $\sigma$ on $[|S(y))|]$ such that $(\Tilde{x}^{S(y)}_{\sigma(1)}, \dots, \Tilde{x}^{S(y)}_{\sigma(|S(y)|)}) \succeq (\Tilde{x}^{S(y')}_{1}, \dots, \Tilde{x}^{S(y')}_{|S(y)|})$, and hence,
    \begin{equation*}
        \mathbb{P}(Y \in A | Y_k = y) \geq \frac{\sum_{x' \in S(y')} \mathbb{P}(X \in B | X_k=x')}{|S(y')|} = \mathbb{P}(Y \in A | Y_k = y').
    \end{equation*}

    Now for the case $|S(y)| > |S(y')|$. Let $\bar{S}(y)$ be the subset of $S(y)$ with cardinality $|S(y')|$ which dominates $S(y')$ while minimizing $\sum_{x \in \bar{S}(y)} \mathbb{P}(X \in B | X_k = x)$, i.e., there exists a permutation $\sigma$ on $[|S(y')|]$ such that $(\Tilde{x}^{\bar{S}(y)}_{\sigma(1)}, \dots, \Tilde{x}^{\bar{S}(y)}_{\sigma(|S(y')|)}) \succeq (\Tilde{x}^{S(y')}_{1}, \dots, \Tilde{x}^{S(y')}_{|S(y')|})$, and hence,
    % ${\rm perm}[{\rm subset}[{\rm vec}[S(y)]]] \succeq {\rm vec}[S(y')]$ such that $\bar{S}(y)$ is the subset of $S(y)$ dominating $S(y')$ while minimizing $\sum_{x \in \bar{S}(y)} \mathbb{P}(X \in B | X_k = x)$. Then,
    \begin{align*}
        \mathbb{P}(Y \in A | Y_k = y) &= \frac{\sum_{x \in \bar{S}(y)} \mathbb{P}(X \in B | X_k=x) + \sum_{x \in \bar{S}^c(y)} \mathbb{P}(X \in B | X_k=x)}{|S(y')| + |\bar{S}^c(y)|}\\
        &\geq \frac{\sum_{x \in \bar{S}(y)} \mathbb{P}(X \in B | X_k=x)}{|S(y')|}\\
        & \geq \frac{\sum_{x' \in S(y')} \mathbb{P}(X \in B | X_k=x)}{|S(y')|}\\
        &= \mathbb{P}(Y \in A | Y_k = y').
    \end{align*}

    Finally for the case  $|S(y')| > |S(y)|$. Let $\bar{S}(y')$ be the subset of $S(y')$ with cardinality $S(y)$ which is dominated by $S(y)$ while maximizing $\sum_{x' \in \bar{S}(y')} \mathbb{P}(X \in B | X_k = x')$, i.e., there exists a permutation $\sigma$ on $[|S(y)|]$ such that $(\Tilde{x}^{\bar{S}(y')}_{\sigma(1)}, \dots, \Tilde{x}^{\bar{S}(y')}_{\sigma(|S(y)|)}) \preceq (\Tilde{x}^{S(y)}_{1}, \dots, \Tilde{x}^{S(y)}_{|S(y)|})$.
    % ${\rm perm}[{\rm subset}[{\rm vec}[S(y')]]] \preceq {\rm vec}[S(y)]$ such that $\bar{S}(y')$ is the subset of $S(y')$ dominated by $S(y)$ while maximizing $\sum_{x' \in \bar{S}(y')} \mathbb{P}(X \in B | X_k = x')$.
    Now observe that
    \begin{align*}
        \mathbb{P}(Y \in A | Y_k = y) &= \frac{\sum_{x \in S(y)} \mathbb{P}(X \in B | X_k=x)}{|S(y)|}\\
        &\geq \frac{\sum_{x' \in \bar{S}(y')} \mathbb{P}(X \in B | X_k=x')}{|S(y)|}\\
        &\geq \frac{\sum_{x' \in \bar{S}(y')} \mathbb{P}(X \in B | X_k=x') + \sum_{x' \in \bar{S}^c(y')} \mathbb{P}(X \in B | X_k=x')}{|S(y')|}\\
        &= \mathbb{P}(Y \in A | Y_k = y').
    \end{align*}
    % Now, since $\bar{S}(y')$ contains the highest probabilities, we have for $x^{\rm c} \in \bar{S}^{\rm c}(y')$ that
    % \begin{equation}
    %     \mathbb{P}(X \in B | X_k = x^{\rm c}) \leq \frac{\sum_{x' \in \bar{S}(y')} \mathbb{P}(X \in B | X_k = x')}{|S(y)|},
    % \end{equation}
    % i.e., the probabilities of $x^{\rm c} \in \bar{S}^{\rm c}(y')$ is less than the average among the highest probabilities.
    The existence of such $\bar{S}(y)$ and $\bar{S}(y')$ is confirmed for the Storey conformal data contamination p-value. Consider initially that the Storey test statistic, denoted by $T$ here, is $T = t+1$ and $T' = t$, i.e., just a difference of one. Now, $S(y) = \{x \in \mathcal{X} : x_{(1)} \leq \cdots \leq x_{(K-t-1)} \leq \lambda < x_{(K-t)} \leq \cdots \leq x_{(K)}\}$, and $S(y') = \{x' \in \mathcal{X} : x'_{(1)} \leq \cdots \leq x'_{(K-t)} \leq \lambda < x'_{(K-t+1)} \leq \cdots \leq x'_{(K)}\}$, and so the only difference occurs at $x_{(K-t)}$ which for $y$ is greater than or equal to $\lambda$ while for $y'$ it is less than or equal to $\lambda$. Now, for the case of $|S(y)| > |S(y')|$, for any $x' \in S(y')$ we can find an $x \in S(y)$ such that $x \succeq x'$. Taking for each $x' \in S(y')$ a unique $x \in S(y)$ with smallest $\mathbb{P}(X \in B | X_k = x)$ (avoiding using the same $x$ more than once), we have constructed $\bar{S}(y)$. In the same way we can construct $\bar{S}(y')$, taking for each $x \in S(y)$ a unique $x' \in S(y')$ with the largest $\mathbb{P}(X \in B | X_k = x')$. Extending the argument to differences in the Storey test statistic of more than one is immediate.
    % From this the result follows.
\end{proof}
We have only proven the \gls{prds} property for the Storey conformal data contamination p-values in Theorem \ref{thm:Storey_PRDS}, however, we expect the result can also be shown for the other conformal data contamination p-values.

\section{Additional discussion of the proposed data sharing procedure}\label{sec:additional_details}
% In this section, we provide additional discussions of the proposed data sharing procedure. We summarize the data sharing procedure described in Section~\ref{sec:procedure} in Procedure~\ref{alg:procedure}. We present an overview of the scenario variables and hyperparameters with some important interpretations in \ref{sec:interpretations}.
% The limitations of the data sharing procedure is discussed in \ref{subsec:limitations}.

% \subsection{Interpretations of scenario variables and hyperparameters}\label{sec:interpretations}
We outline in Tables~\ref{tab:scenario_params} and \ref{tab:hyperparams} all scenario variables and hyperparameters. Here we recall the notation for each variable, and discuss how it influences the collaborative data sharing method proposed in this work.
\begin{table}[t]
    \centering
    \caption{Scenario variables.}\label{tab:scenario_params}
    \begin{tabularx}{\textwidth}{lX}
        \toprule
        \textbf{Parameter}  & \textbf{Description \& Influence} \\\midrule
        $n$
        & The size of the null sample.
        As $n-\ell$ increases, the empirical distribution better approximates the true distribution, yielding more accurate conformal p-values. Moreover, this variable controls the smallest possible p-value, thereby controlling how much evidence one outlier can yield towards rejection in the conformal data contamination tests. \\
        $m$
        & The number of test data points.
        As this variable increases, more evidence against the null can be aggregated yielding a more powerful conformal data contamination test. \\
        $\pi$
        & The true contamination factor. As the difference $\pi - \pi_{\rm th}$ increases, the power of the conformal data contamination tests increases. \\
        $K$
        & The number of other data agents. \\
        $K_0$
        & The number of other data agents satisfying the null hypothesis $H_0^k : \pi_k \leq \pi_{\rm th}$. As this grows, the power of the multiple testing procedure decreases. \\
        \bottomrule
    \end{tabularx}
\end{table}

\begin{table}[t]
    \centering
    \caption{Hyperparameters.}\label{tab:hyperparams}
    \begin{tabularx}{\textwidth}{lX}
        \toprule
        $\ell$
        & The number of null data points used for learning the conformal score.
        As this variable increases the conformal score will tend to better separate outliers from inliers thereby improving outlier detection and increasing power of the conformal data contamination tests. \\
        $\hat{s}$
        & The conformal score function. Choosing a good conformal score is key to achieving good separation between outliers and inliers, which also directly impacts the power {and tightness} of the conformal data contamination tests.\\
        % $K_{\rm budget}$
        % & The collaboration budget, i.e., the number of other data agents to communicate with in the second round (in case a fixed collaboration budget is used). \\
        $\pi_{\rm th}$
        & The chosen threshold on the contamination factor defining the null hypotheses. This hyperparameter controls the amount of contamination that is tolerated. \\
        $\alpha$
        & The significance level for the conformal data contamination tests. This hyperparameter controls the amount of evidence needed before concluding that the data from another agent is more contaminated than the threshold $\pi_{\rm th}$ allows. \\
        $\gamma$
        & The hyperparameter in Storey's \gls{bh} procedure. Controls the bias-variance trade-off for estimating $K_0$. \\
        {$\beta$}
        & {The significance level in the conformal outlier detection. This hyperparameter controls how much evidence is needed to conclude that a data point is an outlier. As this increases, the number of data points which is subsequently used for model fitting increases, with the risk of admitting more outliers.} \\
        {$\zeta$}
        & {The hyperparameter in Storey's \gls{bh} procedure for conformal outlier detection. Controls the bias-variance trade-off for estimating the number of inliers among the acquired data.}\\
        $\lambda$
        & The hyperparameter in the Storey test statistic.\\
        $i_0$
        & The hyperparameter in the Quantile test statistic. \\
        \bottomrule
    \end{tabularx}
\end{table}

% \subsection{Limitations}\label{subsec:limitations}
% The proposed data sharing procedure has some limitations. First, the novel conformal data contamination tests rely on conformal p-values, and herein are some assumptions as briefly mentioned in Section~\ref{subsec:conformal_outlier_detection}. Namely, we require exchangeability of the calibration and test data. This limits the applicability to for instance time series data. Conformal prediction has been studied beyond the exchangeability assumption: the downside of losing exchangeability is a coverage gap, which notably can be lessened through weighting \citep{Barber2023:Conformal}. Second, we did not consider the aspect of data subset selection in this work. This was a deliberate choice to focus on the novel conformal data contamination tests, however, in a practical procedure we recommend to include some data subset selection \citep{Ghorbani2019:Shapley}, or data weighting \citep{Ding2022:Collaborative}. {Lastly, the proposed data sharing procedure is not privacy-aware, limiting the applicability in scenarios where data agents have strict privacy constraints.}

\section{Additional numerical experiments: Gaussian data}\label{sec:Gaussian_data}
In this section, we present a simulation study with null distribution $P_0 \equiv \mathcal{N}(\bs{0}_{2}, \bs{I}_{2\times 2})$ and alternative $P_1 \equiv \mathcal{N}(\mu_1 \bs{1}_{2}, \bs{I}_{2 \times 2})$. As conformal score we use $\hat{s}(X) = -\Vert X \Vert$ as it is a natural choice in light of the distribution $P_0$ {(no learning of the score function is used in this case and so $\ell=0$ throughout this section)}.
We simulate $n=200$ data points from $P_0$ to be the calibration data set and let the test data consist of $m=50$ data points. The test data on average consists of $m(1-\pi)$ data points sampled from $P_0$ while the remaining data points are sampled from $P_1$.

\subsection{Hyperparameters of the conformal data contamination tests}
We compare, in Figure~\ref{fig:simple_power_comparison}, the power estimated using $10000$ simulations when $\alpha = 0.05$, $\mu_1 = 4$, $\pi=0.7$, $\pi_{\rm th}=0.5$, $n=200$, and $m=50$, depending on the hyperparameters of the conformal data contamination tests, for the two cases $\mu_1 = 2$ and $\mu_1 = 4$. We can observe that for appropriate hyperparameter choices, the Storey and Quantile conformal data contamination tests can achieve the highest power. Moreover, when the outliers are easily separated from the inliers ($\mu_1=4$), $\lambda$ should be chosen as a very small value and the choice of $i_0$ should also be sufficiently small, meanwhile for $\mu_1=2$, $\lambda$ and $i_0$ should be chosen moderately. This highlights a weakness of the Storey and Quantile tests as they are sensitive to the hyperparameter choice. Note also that the Fisher test is relatively powerful when the outliers are easily separated, however it is very weak when this is not the case. This is a sensible observation when considering the function used in the Fisher test statistic as it emphasizes the information in very small p-values, but when the outliers are not easily separated, the p-values for the outliers will not be very small and so is less emphasized.
\begin{figure}
    \centering
    \begin{minipage}{0.49\linewidth}
        \centering
        \subfloat[$\mu_1 = {4}$.]{\includegraphics[width=1\linewidth, page=2]{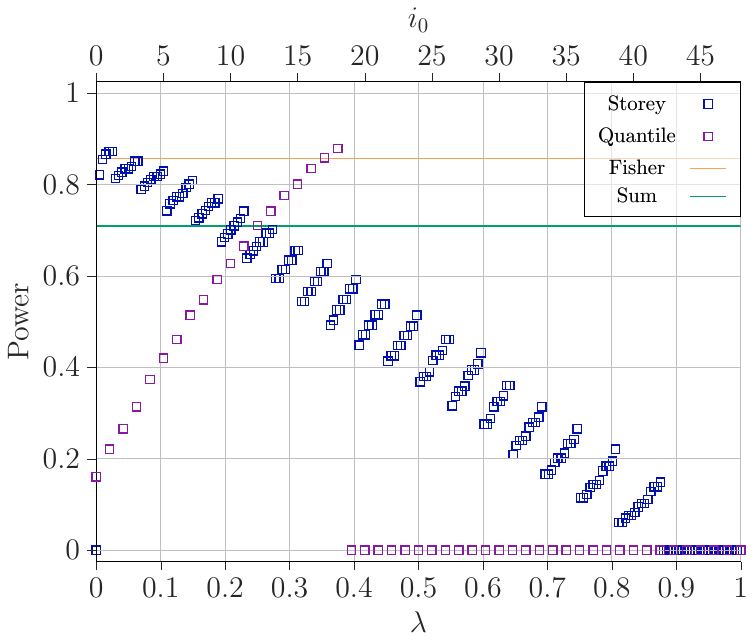}\label{subfig:power_comparison_mu12}}
    \end{minipage}~
    \begin{minipage}{0.49\linewidth}
        \centering
        \subfloat[$\mu_1 = {2}$.]{\includegraphics[width=1\linewidth, page=3]{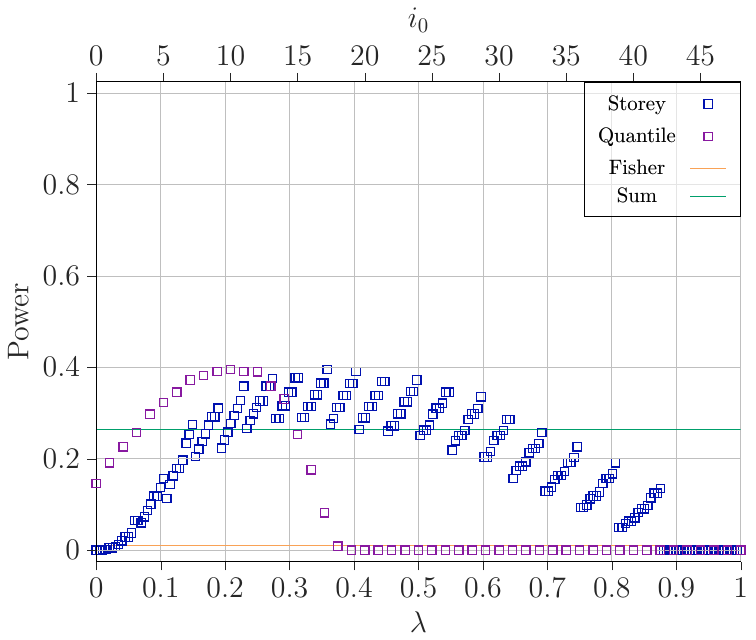}\label{subfig:power_comparison_mu14}}
    \end{minipage}
    \caption{Comparison of power of the different tests for varying hyperparameters.}
    \label{fig:simple_power_comparison}
\end{figure}

\subsection{Special case of {classical} two-sample testing}
In the particular case of $\pi_{\rm th} = 0$, the problem reverts to the classical two-sample testing problem. This can be solved using classical techniques such as Kolmogorov-Smirnov and Cramér-von Mises, but also permutation tests as in \cite{Konstantinou2024:Power}, as well as the combination tests of \cite{Bates2023:Testing}. Notably, the Fisher conformal data contamination test proposed in this work coincides with the Fisher combination test of \cite{Bates2023:Testing} for $\pi_{\rm th} = 0$.

The setup is the following: we observe two samples of data $\mathcal{D}_{0} = \{X_i\}_{i=1}^n$ and $\mathcal{D}_{1} = \{X_i\}_{i=n+1}^m$. Now, many different two-sample tests can be constructed by defining a summary statistic $\hat{s}(X)$, and then comparing the empirical distributions of the summary statistic on the two samples: $\hat{F}_0 = \hat{F}(\hat{s}(\mathcal{D}_{0}))$ and $\hat{F}_1 = \hat{F}(\hat{s}(\mathcal{D}_{1}))$, where $\hat{s}(\mathcal{D}_{1}) = \{\hat{s}(X_{n+i})\}_{i=1}^m$ and $\hat{F}$ computes the empirical \gls{cdf}. Now, a test statistic can be constructed from these two empirical distributions $T(\hat{F}_0, \hat{F}_1)$.

\paragraph{Kolmogorov-Smirnov test:}
A classic test statistic is the Kolmogorov-Smirnov statistic
\begin{equation*}
    T_{\rm KS}(\hat{F}_0, \hat{F}_1) = {\rm sup}_x |\hat{F}_0(x) - \hat{F}_1(x)|,
\end{equation*}
for which asymptotically as $n,m\to \infty$
\begin{equation*}
    \mathbb{P}\Bigg(T_{\rm KS}(\hat{F}_0, \hat{F}_1) > \sqrt{-{\rm ln}(\alpha/2)\frac{n+m}{2nm}}\Bigg) \leq \alpha,
\end{equation*}
with significance level $\alpha \in (0, 1)$.

\paragraph{Permutation tests:}
Permutation tests are based on constructing random permutations by randomly assigning data points from $\mathcal{D}_{0}$ and $\mathcal{D}_{1}$ to new splits $\mathcal{D}_{i,0}$ and $\mathcal{D}_{i,1}$ for $i=1,\dots I$ where $I$ is the number of permutations. The idea is that under the null, the test statistic $T(\hat{F}_0, \hat{F}_1)$ is exchangeable with $T_i = T(\hat{F}(\hat{s}(\mathcal{D}_{i, 0})), \hat{F}(\hat{s}(\mathcal{D}_{i, 1})))$ for $i=1,\dots I$ \citep{Konstantinou2024:Power}. This allows for computing a p-value based on the empirical distribution of the test statistic among the permutations, allowing for flexibility in the choice of test statistic $T$, however, at a computational cost. We consider as examples test statistic using the Kolmogorov-Smirnov difference, the $L_2$-norm, and a test statistic using \gls{ocsvm}. For the \gls{ocsvm}, a number of training permutations are made yielding $T_1,\dots,T_B$, followed by a number of calibration permutations $T_{B+1}, \dots, T_{I}$. Then, the \gls{ocsvm} is fitted to $T_1,\dots,T_B$, and the output scores of the \gls{ocsvm} on $T_{B+1}, \dots, T_I$ are used a calibration data with the \gls{ocsvm} score of $T(\hat{F}_0, \hat{F}_1)$ being the observed test statistic.

\paragraph{Numerical power comparison:}
A power comparison across $2000$ simulations of different two-sample tests for varying contamination factors when $\alpha=0.05$, $n=200$, $m=100$, $\mu_1 = 4$, $\lambda = \lfloor n/32 \rfloor/(n+1)$, and $i_0=5$, is shown in Figure~\ref{fig:two_sample}. We have used $200$ permutations for fitting OCSVM (in Perm. OCSVM), and $500$ permutations for calibration (in the permutation tests). Figure~\ref{fig:two_sample} indicates that the weakest tests are the Kolmogorov-Smirnov tests and the Sum conformal data contamination test, and the other five tests achieve relatively similar power.
We note that it intuitively makes sense that the Fisher combination test is powerful here as this test statistic emphasizes the contribution of very small p-values more so than the other test statistics, which is a desirable property when looking for a minimum of just one outlier. On the other hand, the Sum conformal data contamination test is relatively weak in this setting. These properties of these two tests were also discussed in the context of Figure~\ref{fig:simple_power_comparison}. Our main take-away is that the proposed conformal data contamination tests are competitive with state-of-the-art tests.
\begin{figure}
    \centering
    \includegraphics[width=0.5\linewidth, page=4]{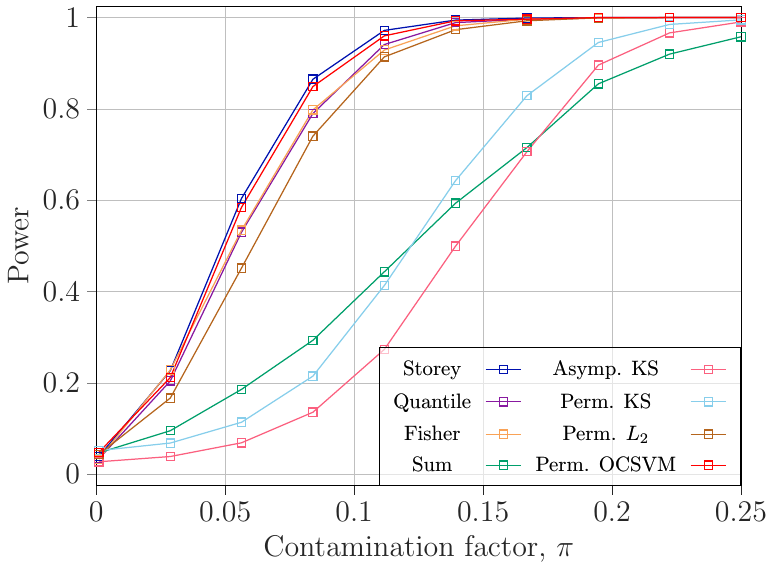}
    \caption{Comparison of the power of different two-sample tests for varying contamination factors.}
    \label{fig:two_sample}
\end{figure}

\subsection{Ablation study: test and calibration data size}
In this section we study the performance of the conformal data contamination tests for varying parameter settings. Specifically, the focus will be on varying contamination factors, $\pi$, calibration data sizes, $n-\ell$
% {(we are still using the score function $\hat{s}(X) = -\Vert X \Vert$, hence, no learning of the score is required and so $\ell=0$)}
, and test data sizes, $m$. We set $\lambda = \lfloor {(n-\ell)}/12 \rfloor/({n-\ell}+1)$, $i_0 = \lfloor m / 1.5\rfloor$, $\alpha=0.05$, $\pi_{\rm th} = 0.1$, and $\mu_1=4$.

Consider an ablation study for the contamination factor, $\pi$, and the test data size, $m$, with $n = 100$. We show in Figure~\ref{subfig:ablation_Gaussian_pi_m} the power estimated using $1000$ simulations for the conformal data contamination tests. We observe that the power relatively quickly goes to $1$ as the contamination factor increases, and there is a general tendency that the power increases with $m$, which is also to be expected. In this simulation study, the Storey conformal data contamination test achieves the highest power.

The influence of the calibration data size, $n-\ell$, with $m = 100$ and $\pi=0.3$, is visualized in Figure~\ref{subfig:ablation_Gaussian_n}. Here we show the power estimated using $10000$ simulations for the conformal data contamination tests. The Storey and Quantile conformal data contamination tests have the highest power, achieving a power near $1$ for $n-\ell \geq 20$. This highlights that powerful conformal data contamination tests are possible even for extremely modest calibration data sizes.
\begin{figure}[t]
    \centering
    \begin{minipage}{0.49\linewidth}
        \centering
        \subfloat[]{\includegraphics[width=1\linewidth, page=5]{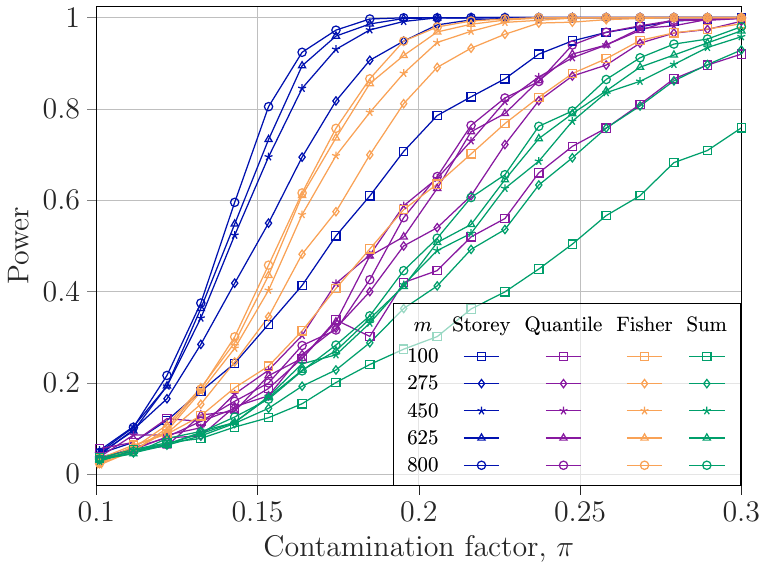}\label{subfig:ablation_Gaussian_pi_m}}
    \end{minipage}~
    \begin{minipage}{0.49\linewidth}
        \centering
        \subfloat[]{\includegraphics[width=1\linewidth, page=6]{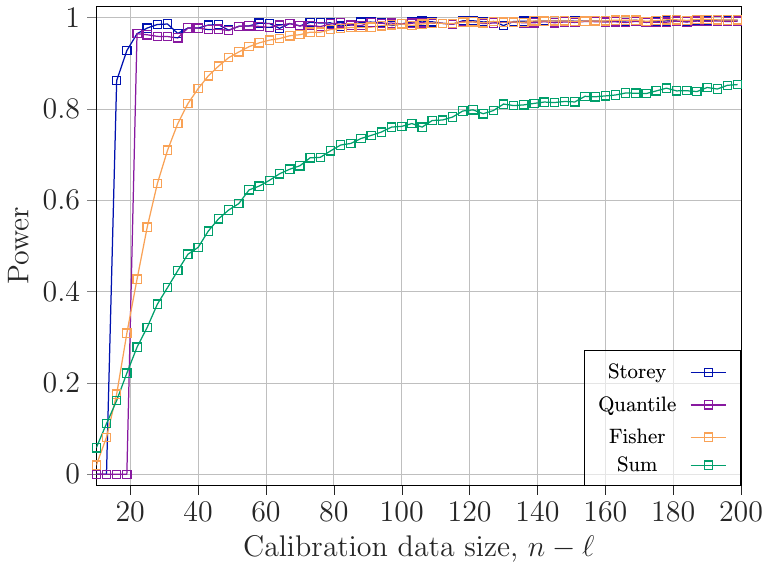}\label{subfig:ablation_Gaussian_n}}
    \end{minipage}
    \caption{Ablation study for varying contamination factors, $\pi$, calibration data sizes, $n-\ell$, and test data sizes, $m$.}
    \label{fig:ablation_Gaussian}
\end{figure}

With the insights gained through these numerical experiments we can highlight an advantage of the conformal data contamination tests. Conformal outlier detection relies on having a calibration data size, $n-\ell$, relatively large compared to the test data size, $m$, see \cite{Mary2022:Semi}. Meanwhile, conformal data contamination tests only improve with increasing $m$, and so are more convenient in settings with small $n-\ell$ and large $m$.

\subsection{Multiple testing with conformal data contamination p-values}
Consider the same scenario as before except now $K=20$ agents are sending test data for which $K_0=10$ of them have $\pi_0=\pi_{\rm th}$ and the remaining $K-K_0=10$ have $\pi_1 > \pi_{\rm th}$. We use Storey's \gls{bh} procedure on the sequence of conformal data contamination p-values with Storey's hyperparameter $\gamma=0.5$.

In Figure~\ref{fig:Gaussian_fdr_tdr_curves} we show \gls{fdr} and \gls{tdr} curves estimated by $2000$ simulations for the four different conformal data contamination tests presented in Section~\ref{sec:results}, when $\alpha = 0.05$, $\mu_1 = 4$, $n=200$, $m=100$, $\pi_{\rm th}=0.2$, $\pi_1=0.3$, $\lambda = \lfloor {(n-\ell)}/32 \rfloor/({n-\ell}+1)$, and $i_0 = \lfloor m/1.5\rfloor$. We notice from Figure~\ref{subfig:Gaussian_fdr} that all the methods are conservative. In terms of power Figure~\ref{subfig:Gaussian_tdr} shows that the Storey conformal data contamination test achieves the highest power, with comparable power between Quantile and Fisher test, and the Sum test tends to have the lowest power. We emphasize that this does not mean that in general the Sum test is the worst alternative among those considered here. Which of the proposed tests performs the best depends on many factors, herein $\pi_{\rm th}$ and the conformal score.
\begin{figure}[t]
    \centering
    \begin{minipage}{0.49\linewidth}
        \centering
        \subfloat[FDR.]{\includegraphics[width=1\linewidth, page=7]{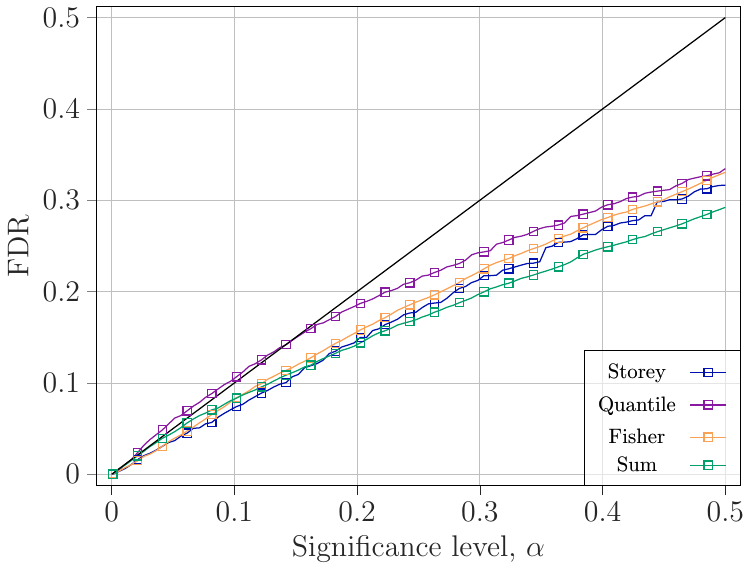}\label{subfig:Gaussian_fdr}}
    \end{minipage}~
    \begin{minipage}{0.49\linewidth}
        \centering
        \subfloat[TDR.]{\includegraphics[width=1\linewidth, page=8]{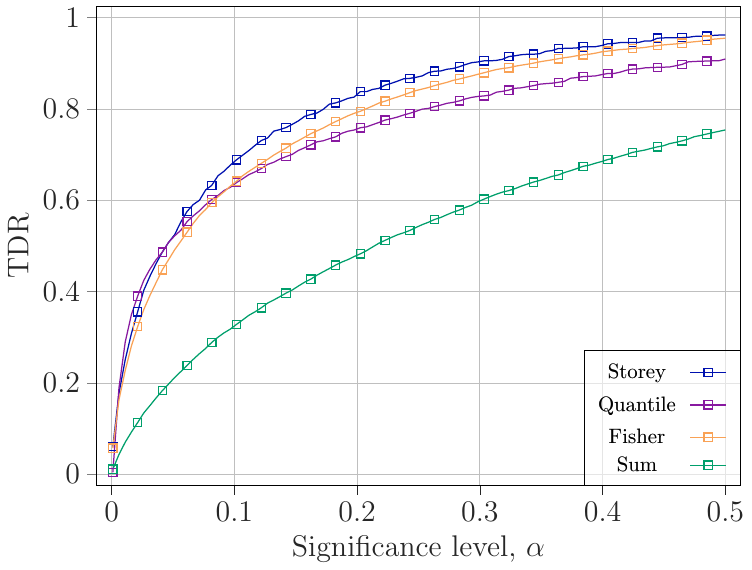}\label{subfig:Gaussian_tdr}}
    \end{minipage}
    \caption{Comparison of FDR and TDR estimated of the different tests.}
    \label{fig:Gaussian_fdr_tdr_curves}
\end{figure}

\section{Additional numerical experiments: Real datasets}\label{sec:real_datasets}
In this section, we provide extensive numerical results on a variety of data contamination scenarios with real image classification datasets. We consider several conformal scores and evaluate both the outlier detection capability, the performance of the proposed conformal data contamination tests, and the performance on the relevant classification task for the proposed data {acquisition} procedure as well as relevant baselines.

% \subsection{{Data contamination scenarios}}
In the following we list the data contamination scenarios.
\begin{enumerate}[label=({S\arabic*})]
    \item Mixture of retinal fundus images (color) from RetinaMNIST \citep{Liu2022:RetinaMNIST} and EyePACS \citep{Dugas2015:EyePACS,Gulshan2016:EyePACS}, as also considered in Section~\ref{sec:numerical_experiments}. The classification task is to distinguish between five degrees of severity of diabetic retinopathy ranging from \emph{no diabetic retinopathy} to \emph{proliferate retinopathy}. \label{enum:S1}
    \item Mixture of optical coherence tomography (OCT) images (grayscale) taken from University of California San Diego found in OCTMNIST \citep{Kermany2018:OCTMNIST} and OCT images from Noor Eye Hospital (NEH) \citep{Sotoudeh2022:NEH}. The classification task is to distinguish \emph{choroidal neovascularization} from \emph{normal retina}. \label{enum:S2}
    \item Mixture of microscopic blood smear images (color) from BloodMNIST \citep{Acevedo2020:BloodMNIST} and \citep{Bodzas2023:WBC}. The classification task to to distinguish between the labels \emph{basophile}, \emph{eosinophile}, \emph{lymphocyte}, and \emph{monocyte}. \label{enum:S3}
    \item Mixture of images (grayscale) of lower- and uppercase handwritten letters \emph{A}, \emph{B}, and \emph{C} from FEMNIST \citep{Caldas2019:LEAF}. \label{enum:S4}
    \item Mixture of noise-free and noisy images (grayscale) of handwritten digits $1$, $4$, and $7$ from MNIST \citep{Lecun1998:Gradient}. We sample the feature noise as a Gaussian vector with zero mean and covariance kernel $\varsigma^2 {\rm exp}(-\varrho\Vert p_i - p_j\Vert)$ for pixel positions $p_i = i$ and $p_j = j$ for the $i$-th row and $j$-column in the pixel image, and add it to the image. After adding noise, each image is normalized to the interval $[0, 1]$. We set $\varsigma = 0.3$ and $\varrho = 0.1$. \label{enum:S5}
    \item Mixture of images (grayscale) of handwritten digits $1$, $4$, and $7$ from MNIST with and without noisy labels. \label{enum:S6}
    \item {Mixture of images of dogs, deer, and cats from CIFAR-10 and CIFAR-10C contaminated with brightness \citep{Hendrycks2019:Benchmarking}.} \label{enum:S7}
    \item {Mixture of images of dogs, deer, and cats from CIFAR-10 and CIFAR-10C contaminated with snow \citep{Hendrycks2019:Benchmarking}.} \label{enum:S8}
    \item {Mixture of images of classes 'n01641577', 'n04562935', 'n02233338', 'n09246464', and 'n03992509' from Tiny-ImageNet and Tiny-Imagenet-C contaminated with defocus blur \citep{Hendrycks2019:Benchmarking}.} \label{enum:S9}
\end{enumerate}
{The images in (S1), (S2), and (S3) are cropped and down sampled in the same way to $28\times 28$ ($28\times28\times3$ for color images) pixels. All images are pixel-wise normalized to the interval $[0, 1]$.}

In all cases we sample randomly $n$ in-distribution datapoints for the $0$-th data agent. Meanwhile, the $k$-th data agent observes $2m$ datapoints and the number of \gls{ood} datapoints is sampled from ${\rm B}_{\pi_k}^{2m}$. We sample the contamination factors \gls{iid} according to a standard uniform distribution, such that approximately half of the test data are outliers. For all the numerical experiments we set $n=100$, $\ell=60$, $m=m_{r0}=m_{r1}=40$, $K = 10$, $\lambda=\lfloor {(n-\ell)}/8\rfloor/({n-\ell}+1)$, $i_0 = \lfloor m/3 \rfloor$, and $\gamma=0.5$. Here we chose $n$ relatively small as this is the regime where data sharing is most interesting, and we decided on a 60-40 split between data for fitting conformal scores and calibration data, as we found that having enough data for fitting the conformal score is often critical. The hyperparameters $\lambda$ and $i_0$ were chosen in light of the insights gained in the ablation study in Section~\ref{sec:Gaussian_data}. The hyperparameter for Storey's \gls{bh} procedure, $\gamma$, was chosen somewhat arbitrarily at the intermediate value $0.5$.

% \subsection{{Conformal scores}}
We consider different conformal scores:
\begin{enumerate}[label=({C\arabic*}),leftmargin=1.1cm]
    \item \Gls{ocsvm} \citep{Scholkopf2001:Estimating}. \label{enum:C1}
    \item Class-wise score based on \gls{ocsvm}. \label{enum:C6}
    \item \Gls{if} \citep{Liu2008:IF}. \label{enum:C2}
    \item Class-wise score based on \gls{if} . \label{enum:C7}
    \item \Gls{ae} \citep{Hawkins2002:Outlier}. \label{enum:C3}
    \item Class-wise score based on \gls{ae}. \label{enum:C8}
    \item \Gls{pu} classification using \gls{lr} \citep{Marandon2024:Adaptive}. \label{enum:C4}
    \item Class-wise score based on \gls{pu} classification using \gls{lr}. \label{enum:C9}
    \item \Gls{pu} classification using \gls{svc} \citep{Marandon2024:Adaptive}. \label{enum:C5}
    \item Class-wise score based on \gls{pu} classification using \gls{svc}. \label{enum:C10}
    \item {\Gls{pu} classification using \gls{lr} on principal components explaining at least $80~\%$ of the variance.}\label{enum:C11}%
    \item {Class-wise score based on \gls{pu} classification using \gls{lr} on principal components explaining at least $80~\%$ of the variance.} \label{enum:C12}
    \item {\Gls{pu} classification using a Residual Network with four layers (ResNet-4) \citep{He2016:Residual}.}\label{enum:C13}%
    \item {Class-wise score based on \gls{pu} classification using a ResNet-4.} \label{enum:C14}
\end{enumerate}
% {We consider different conformal scores: (C1) \gls{ocsvm} \citep{Scholkopf2001:Estimating}, (C2) \gls{if} \citep{Liu2008:IF}, (C3) autoencoder \citep{Hawkins2002:Outlier}, (C4) \gls{pu} classification using \gls{lr} \citep{Marandon2024:Adaptive}, (C5) \gls{pu} classification using \gls{svc}, (C6) class-wise score based on \gls{ocsvm}, (C7) class-wise score based on \gls{if}, (C8) class-wise score based on autoencoder, (C9) class-wise score based on \gls{pu} classification using \gls{lr}, (C10) class-wise score based on \gls{pu} classification using \gls{svc}.
%
% Class-wise scores are defined as
% \begin{equation*}
%     \begin{split}
%         &\hat{s}((\bs{X}, Y), (Z_1, \dots, Z_\ell), (Z_{\ell+1}, \dots, Z_{n+m})) \\
%         &\quad = \sum_{i \in \mathcal{Y}} \mathbbm{1}[Y = i] ~ \hat{s}_i(\bs{X}, (\bs{X}_j : Y_j = i, j \in \{1,\dots,\ell\}), (\bs{X}_j : Y_j = i, j \in \{\ell+1,\cdots n+m\})),
%     \end{split}
% \end{equation*}
% where $\mathcal{Y}$ denotes the set of labels, $(\bs{X}_j : Y_j = i, j\in\{1,\dots,\ell\})$ denotes the subset of $\bs{X}_1, \dots, \bs{X}_\ell$ for which $Y_j = i$, $j\in\{1,\dots,\ell\}$. Here $\hat{s}_i(\cdot, (\bs{X}_j : Y_j = i, j\in\{1,\dots,\ell\}))$ denotes the score function fitted to $(\bs{X}_j : Y_j = i, j\in\{1,\dots,\ell\})$, and in the case of adaptive scores also $(\bs{X}_j : Y_j = i, j\in\{\ell+1,\dots,n+m\})$.
%
Class-wise scores are defined as
% \begin{equation*}
%     {\hat{s}((X, Y), (Z_1, \dots, Z_\ell), (Z_{\ell+1}^1, \dots, Z_{n+m}^K)) = \sum_{i \in \mathcal{Y}} \mathbbm{1}[Y = i] ~ \hat{s}_i(X, X_{1:\ell}^{Y=i}, X_{\ell+1:n+Km}^{Y=i}),}
% \end{equation*}
% {where $\mathcal{Y}$ denotes the set of labels, $X_{1:\ell}^{Y=i} = (X_j : Y_j = i, j\in[\ell])$ denotes the subset of $X_1, \dots, X_\ell$ for which $Y_j = i$ for $j\in[\ell]$, and likewise $X_{\ell+1:n+Km}^{Y=i} = (X_{\ell+j} : Y_{\ell+j} = i, j\in[n-\ell])\cup(X_{j}^k : Y_{j}^k = i, j\in[m], k\in[K])$.}
\begin{equation*}
    {\hat{s}((X, Y), (Z_1, \dots, Z_\ell), (Z_{\ell+1},\dots,Z_{n}, Z_1^1, \dots, Z_{m}^K)) = \sum_{i \in \mathcal{Y}} \mathbbm{1}[Y = i] ~ \hat{s}_i(X, X_{\rm train}^{Y=i}, X_{\rm cal + test}^{Y=i}),}
\end{equation*}
{where $\mathcal{Y}$ denotes the set of labels, $X_{\rm train}^{Y=i} = (X_j : Y_j = i, j\in[\ell])$ denotes the subset of $X_1, \dots, X_\ell$ for which $Y_j = i$ for $j\in[\ell]$, and likewise $X_{\rm cal+test}^{Y=i} = (X_{\ell+j} : Y_{\ell+j} = i, j\in[n-\ell])\cup(X_{j}^k : Y_{j}^k = i, j\in[m], k\in[K])$.}
% Specifically, $\hat{s}_i(\bs{X}, \bs{X}_{1:\ell}^{Y=i}, \bs{X}_{\ell+1:n+m}^{Y=i})$ denotes the predicted probability that $(\bs{X}, i)$ has label $1$, i.e., is an inlier, using \gls{lr} fitted to $\bs{X}_{1:\ell}^{Y=i}$ with labels $1$ and $\bs{X}_{\ell+1:n+m}^{Y=i}$ with labels $-1$.

Scores (C1)-(C4) and (C7)-(C12) are implemented using the standard implementations in \texttt{scikit-learn}, the scores (C5) and (C6) are implemented in \texttt{tensorflow}, and the scores (C13) and (C14) are implemented in \texttt{pytorch}. {Details for the implementation of the score functions are given in Section~\ref{sec:score_details}.}
% Details are available in the software repository.

% {For image classification, we use \gls{lr}, \gls{svc}, \gls{gbc}, and \gls{mlp}, all based on standard implementations in \textit{scikit-learn}.}

% We consider the following baselines: (i) \emph{oracle 1:} the complete oracle knows the true contamination factors and exactly which datapoints are contaminated; (ii) \emph{oracle 2:} the partial oracle knows the true contamination factors but not which datapoints are contaminated; and (iii) \emph{random:} we randomly select which agents to collaborate with after the first round.

\subsection{{Marginal power and type I error probability study}}
{To compare the different conformal scores on the various data contamination scenarios we will consider the marginal power of the conformal outlier detection and conformal data contamination tests, as well as the type I error probability for the conformal data contamination tests, through the empirical \glspl{cdf} of the conformal p-values and the conformal data contamination p-values, respectively. Particularly, we will use the \gls{auc} to summarize the \glspl{cdf} to scalars.}

For conformal outlier detection, ${\rm AUC}^{\rm cod} = 1/({n-\ell}+2) \sum_{i=1}^{{n-\ell}+1} \hat{F}_{\rm outlier}(i/({n-\ell}+1))$ where $\hat{F}_{\rm outlier}$ is the empirical \gls{cdf} of the conformal p-value among outliers. ${\rm AUC}^{\rm cod}$ ranges from $0$ to $1$ with $0.5$ corresponding to random guessing and $1$ corresponding to perfect detection.

{For conformal data contamination tests, ${\rm AUC}^{cdct}_{(\pi, \pi_{\rm th})} = \int_0^1 \hat{F}_{(\pi, \pi_{\rm th})}(x) {\rm d}x$ where $\hat{F}_{(\pi, \pi_{\rm th})}$ is the empirical \gls{cdf} of the conformal data contamination test p-value for data contamination threshold $\pi_{\rm th}$ and when the data contamination level is $\pi$. For comparing the power we are interested in $\pi > \pi_{\rm th}$ and a larger ${\rm AUC}^{cdct}_{(\pi, \pi_{\rm th})}$ the better, while for numerically evaluating the conservativeness we are interested in $\pi = \pi_{\rm th}$ which should be less than but also close to $0.5$.}

\begin{figure}[t]
    \centering
    \begin{minipage}{0.32\linewidth}
        \centering
        \subfloat[Scenario (S1).]{\includegraphics[width=1\linewidth, page=9]{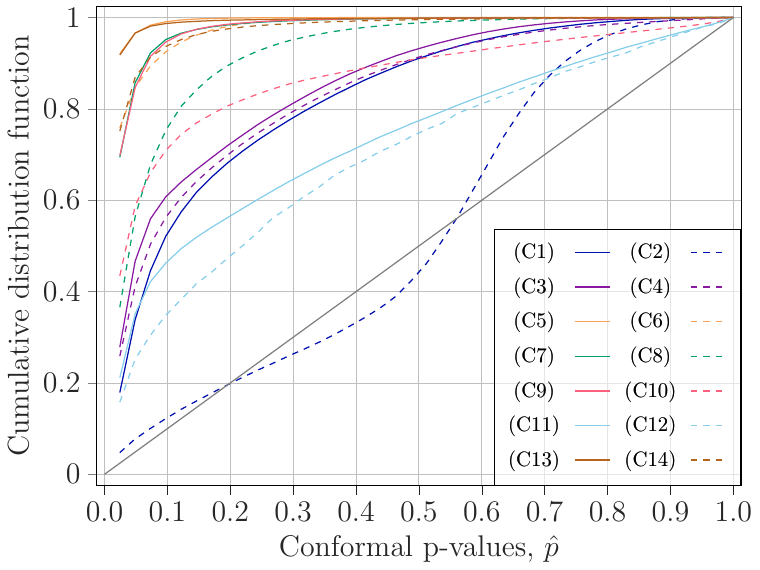}\label{subfig:COD_cdf_retina}}
    \end{minipage}~
    \begin{minipage}{0.32\linewidth}
        \centering
        \subfloat[Scenario (S2).]{\includegraphics[width=1\linewidth, page=10]{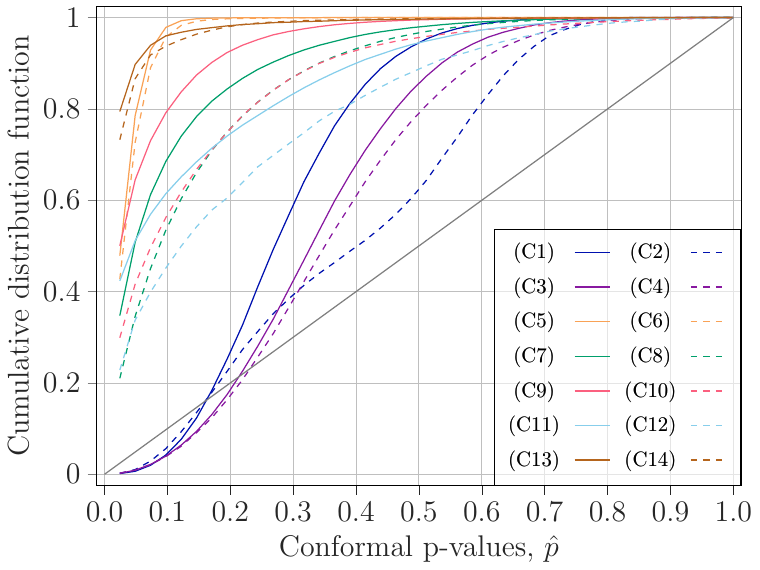}\label{subfig:COD_cdf_retinalOCT}}
    \end{minipage}~
    \begin{minipage}{0.32\linewidth}
        \centering
        \subfloat[Scenario (S3).]{\includegraphics[width=1\linewidth, page=11]{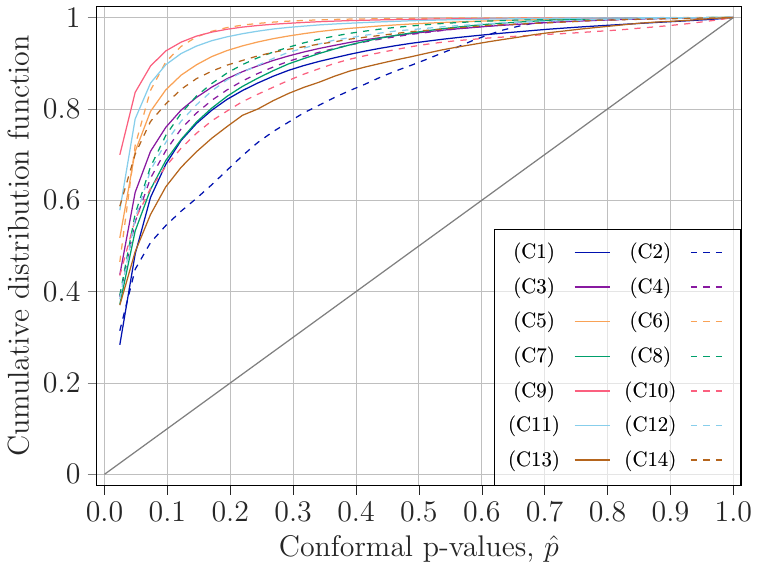}\label{subfig:COD_cdf_WBC}}
    \end{minipage}\\
    \begin{minipage}{0.32\linewidth}
        \centering
        \subfloat[Scenario (S4).]{\includegraphics[width=1\linewidth, page=12]{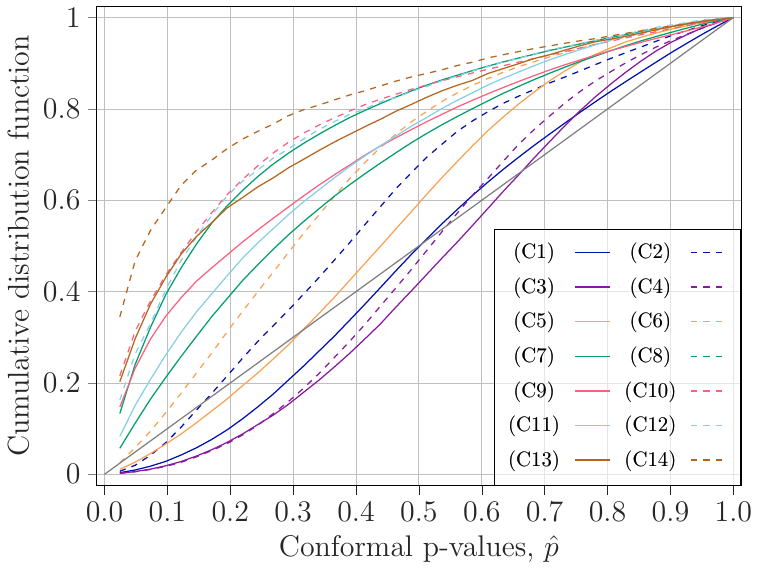}\label{subfig:COD_cdf_FEMNIST}}
    \end{minipage}~
    \begin{minipage}{0.32\linewidth}
        \centering
        \subfloat[Scenario (S5).]{\includegraphics[width=1\linewidth, page=13]{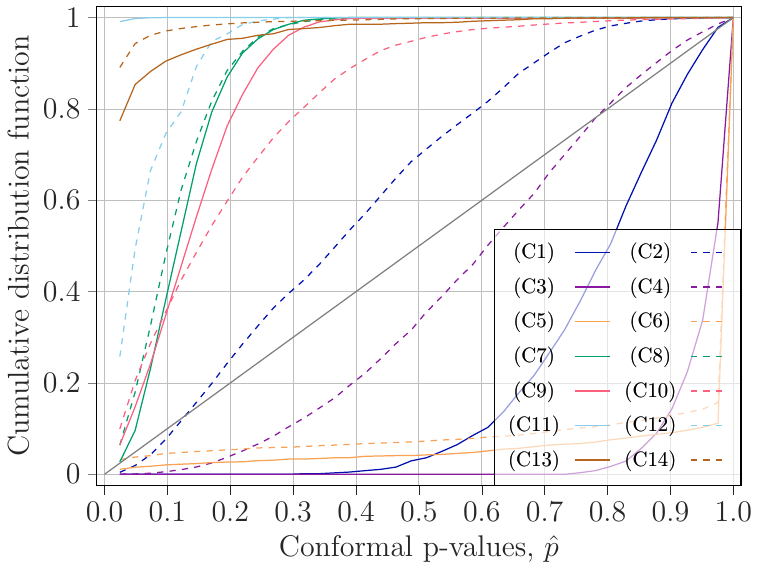}\label{subfig:COD_cdf_MNISTfn}}
    \end{minipage}~
    \begin{minipage}{0.32\linewidth}
        \centering
        \subfloat[Scenario (S6).]{\includegraphics[width=1\linewidth, page=14]{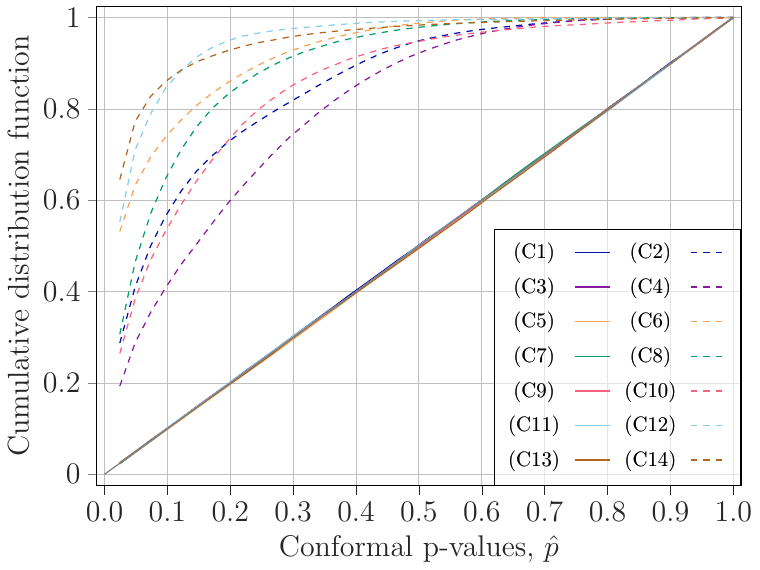}\label{subfig:COD_cdf_MNISTln}}
    \end{minipage}\\
    \begin{minipage}{0.32\linewidth}
        \centering
        \subfloat[{Scenario (S7).}]{\includegraphics[width=1\linewidth, page=35]{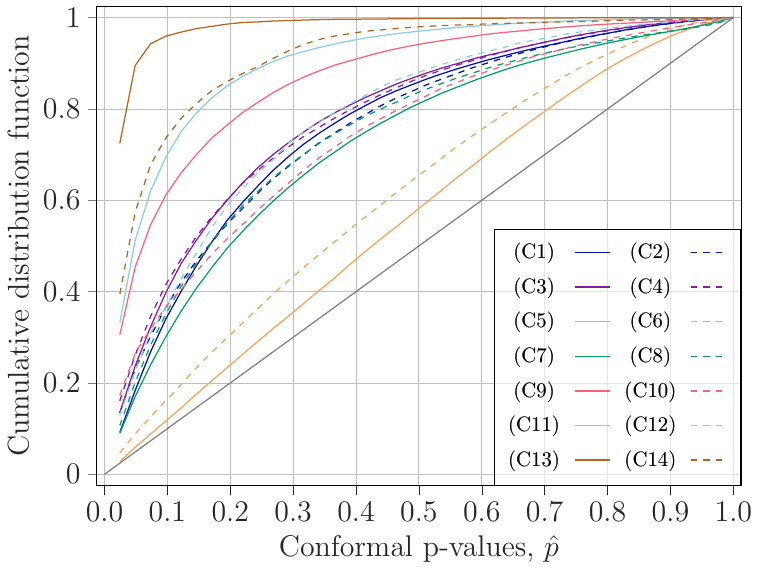}\label{subfig:COD_cdf_cifar10brightness}}
    \end{minipage}~
    \begin{minipage}{0.32\linewidth}
        \centering
        \subfloat[{Scenario (S8).}]{\includegraphics[width=1\linewidth, page=36]{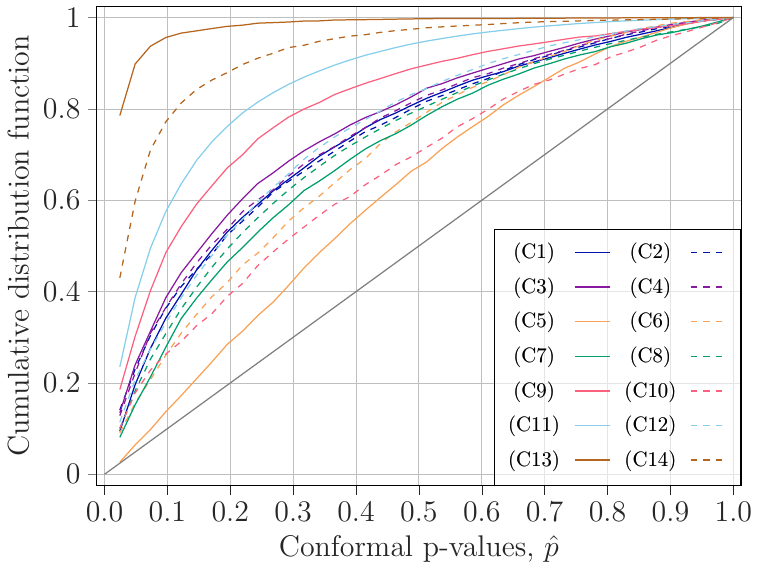}\label{subfig:COD_cdf_cifar10snow}}
    \end{minipage}~
    \begin{minipage}{0.32\linewidth}
        \centering
        \subfloat[{Scenario (S9).}]{\includegraphics[width=1\linewidth, page=38]{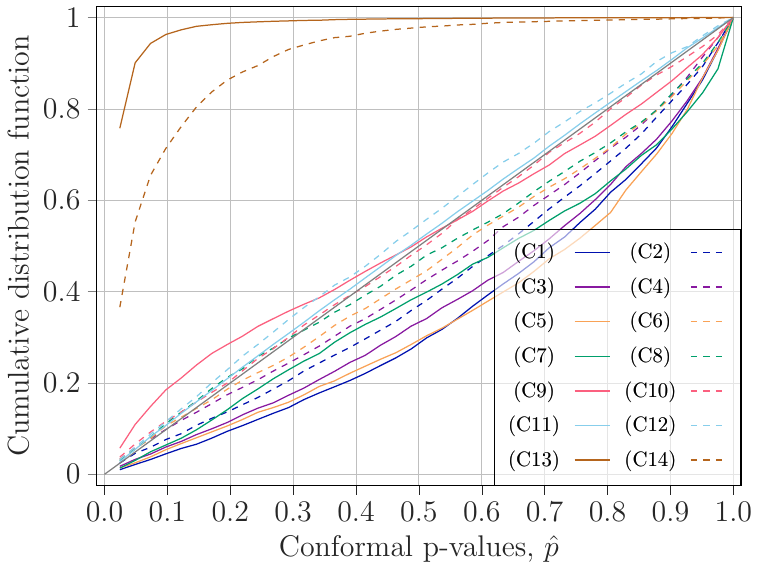}\label{subfig:COD_cdf_imagenetdefocusblur}}
    \end{minipage}
    \caption{{\emph{Power study:} Empirical CDFs of ${\rm AUC}^{\rm cod}$ for the different conformal scores and data contamination scenarios.}}
    \label{fig:COD_cdf}
\end{figure}

\setlength{\tabcolsep}{0.5em} % for the horizontal padding
\begin{table*}[t]
    \centering
    \caption{{\emph{Power study:} AUC for conformal outlier detection computed from $2000$ simulations. The best conformal score for each data contamination scenario is highlighted with boldface.}}
    % \vspace{0.5cm}
    % \scriptsize
    \label{tab:bigAUC}
    \begin{tabular}{@{}l|rrrrrrrrr@{}}
        \toprule
                & (S1)          & (S2)          & (S3)          & (S4)          & (S5)          & (S6)          & (S7)          & (S8)          & (S9)\\\midrule
 (C1)           & $0.8164$      & $0.7110$      & $0.8694$      & $0.4819$      & $0.2246$      & $0.5011$      & $0.7531$      & $0.7279$      & $0.3536$ \\
 (C2)           & $0.5343$      & $0.6047$      & $0.8240$      & $0.5852$      & $0.6158$      & $0.8451$      & $0.7504$      & $0.7292$      & $0.4128$ \\
 (C3)           & $0.8439$      & $0.6582$      & $0.9003$      & $0.4539$      & $0.0586$      & $0.4996$      & $0.7738$      & $0.7501$      & $0.3774$ \\
 (C4)           & $0.8261$      & $0.6333$      & $0.8894$      & $0.4773$      & $0.4077$      & $0.7979$      & $0.7734$      & $0.7361$      & $0.4426$ \\
 (C5)           & $\bs{0.9726}$ & $0.9562$      & $0.9287$      & $0.5554$      & $0.0712$      & $0.4984$      & $0.5594$      & $0.6116$      & $0.3514$ \\
 (C6)           & $0.9581$      & $0.9515$      & $0.9450$      & $0.6569$      & $0.0997$      & $0.9088$      & $0.6092$      & $0.6913$      & $0.4542$ \\
 (C7)           & $0.9584$      & $0.8921$      & $0.8863$      & $0.6580$      & $0.8676$      & $0.4999$      & $0.7164$      & $0.6948$      & $0.4046$ \\
 (C8)           & $0.9105$      & $0.8536$      & $0.9058$      & $0.7572$      & $0.8792$      & $0.8850$      & $0.7409$      & $0.7124$      & $0.4704$ \\
 (C9)           & $0.9580$      & $0.9259$      & $\bs{0.9537}$ & $0.6979$      & $0.8525$      & $0.4992$      & $0.8543$      & $0.7999$      & $0.5160$ \\
 (C10)          & $0.8571$      & $0.8549$      & $0.8645$      & $0.7678$      & $0.8019$      & $0.8438$      & $0.7345$      & $0.6530$      & $0.4999$ \\
 (C11)          & $0.7283$      & $0.8568$      & $0.9435$      & $0.6903$      & $\bs{0.9759}$ & $0.4995$      & $0.8896$      & $0.8502$      & $0.5110$ \\
 (C12)          & $0.6904$      & $0.7917$      & $0.8982$      & $0.7633$      & $0.9223$      & $\bs{0.9357}$ & $0.7725$      & $0.7401$      & $0.5326$ \\
 (C13)          & $0.9709$      & $\bs{0.9613}$ & $0.8478$      & $0.7502$      & $0.9479$      & $0.4978$      & $\bs{0.9608}$ & $\bs{0.9610}$ & $\bs{0.9621}$ \\
 (C14)          & $0.9565$      & $0.9578$      & $0.9156$      & $\bs{0.8145}$ & $0.9666$      & $0.9331$      & $0.9003$      & $0.9063$      & $0.8976$ \\
        \bottomrule
    \end{tabular}
\end{table*}

We report in Table~\ref{tab:bigAUC} ${\rm AUC}^{\rm cod}$ for the considered data contamination scenarios and for each of the different conformal scores, and show the full \glspl{cdf} in Figure~\ref{fig:COD_cdf}. Immediately, we can notice that the adaptive scores using the \gls{pu} classification techniques of \cite{Marandon2024:Adaptive} consistently perform well. For all the medical data scenarios (S1), (S2), and (S3), AdaDetect using \gls{svc} or ResNet-4, i.e.~(C9) or (C13), as well as the \gls{ae}, i.e.~(C5), perform very well.
% For all the medical data scenarios (S1), (S2), and (S3), AdaDetect with \gls{svc}, i.e. (C9), as well as the \gls{ae}, i.e.~(C5), perform very well.
For scenarios (S4), (S5), and (S6), class-wise scores with AdaDetect generally perform {well}, i.e.~(C8), (C10), (C12), and (C14), consistently yielding accurate outlier detection. Meanwhile, the most effective score for the CIFAR-10 scenarios (S7) and (S8) is AdaDetect using ResNet-4, i.e.~(C13), with \gls{lr} on the principal components, i.e.~(C11), not far behind. Finally, the only conformal score among those considered that effectively detects outliers in the Tiny-ImageNet scenario (S9) is AdaDetect using ResNet-4, i.e.~(C13) and (C14).

These results may suggest that AdaDetect using a Residual Network is a good choice for image data, but such a neural network approach is of higher complexity than for instance AdaDetect with \gls{lr} and can also suffer from overfitting during the training phase if the training algorithm is not well specified, leading to catastrophic failure of outlier detection. For a given data application, it is crucial to choose a high performing conformal score which may require some domain knowledge.
% The most challenging data contamination scenario here is (S4), i.e., the FEMNIST data scenario, meanwhile the apparently easiest is (S1), i.e., the retinal fundus image data scenario.

While these results for varying data contamination scenarios with different conformal scores are not exhaustive, they give an indication of the achieveable outlier detection performance, and provide some guidelines for choosing conformal scores.
% In the following, we will limit our attention to AdaDetect with \gls{svc}, i.e.~(C9), and class-wise AdaDetect with \gls{lr}, i.e.~(C8), as these scores perform consistently wellv{, and we will intermediately restrict our attention to scenarios (S1)-(S6).}
In the following, we limit our attention to class-wise AdaDetect with \gls{lr}, i.e.~(C8), as this score is simple and perform consistently well for scenarios (S1)-(S6){, which we intermediately restrict our attention to.}

{In Table~\ref{tab:AUC_CDCT_Power_v2}, we display ${\rm AUC}^{\rm cdct}_{(\pi, \pi_{\rm th})}$, for $\pi > \pi_{\rm th}$ thereby indicating the power of the tests, for the four different proposed conformal data contamination tests for varying values of the data contamination threshold, $\pi_{\rm th}$, when the data contamination level is $\pi=0.3$. The \gls{auc} decreases for increasing data contamination thresholds, which is to be expected as distinguishing between smaller differences in data contamination is naturally harder. We notice that in almost all cases, the Storey and Quantile tests, which are valid non-asymptotically, perform better than the Fisher and Sum tests, which are only valid asymptotically. Further, we observe that the Storey test usually has the highest \gls{auc} for small data contamination thresholds, whereas the Quantile test is comparatively better for larger data contamination thresholds. A similar observation holds when comparing the Fisher and Sum tests, with the Fisher test performing better at small data contamination thresholds, and vice versa, as to be expected.}

\setlength{\tabcolsep}{0.5em} % for the horizontal padding
\begin{table*}[t]
    \centering
    \caption{{\emph{Power study:} AUC for conformal data contamination tests computed from $2000$ simulations when $\pi = 0.3$. The best test statistic for each data contamination scenario is highlighted with boldface.}}
    % \vspace{0.5cm}
    % \scriptsize
    \label{tab:AUC_CDCT_Power_v2}
    \begin{tabular}{@{}ll|rrr|rrr|rrr@{}}
        \toprule
                                &           & \multicolumn{3}{c|}{(S1)} & \multicolumn{3}{c|}{(S2)} & \multicolumn{3}{c}{(S3)} \\
                                &           & $\pi_{\rm th}\!=\!0$ & $\pi_{\rm th}\!=\!0.1$ & $\pi_{\rm th}\!=\!0.2$
                                            & $\pi_{\rm th}\!=\!0$ & $\pi_{\rm th}\!=\!0.1$ & $\pi_{\rm th}\!=\!0.2$
                                            & $\pi_{\rm th}\!=\!0$ & $\pi_{\rm th}\!=\!0.1$ & $\pi_{\rm th}\!=\!0.2$ \\\midrule
        \multirow{4}{*}{(C8)}   & Storey   & $\bs{0.9514}$  & $\bs{0.8233}$ & 0.5867        & $\bs{0.8607}$ & 0.6470        & 0.3857        & $\bs{0.9462}$ & $\bs{0.8108}$ & 0.5720 \\
                                & Quantile & 0.8055         & 0.7309        & $\bs{0.6279}$ & 0.7969        & $\bs{0.7185}$ & $\bs{0.6110}$ & 0.7998        & 0.7244        & $\bs{0.6207}$ \\
                                & Fisher   & 0.9289         & 0.7351        & 0.4456        & 0.8474        & 0.5762        & 0.2815        & 0.9245        & 0.7321        & 0.4446 \\
                                & Sum      & 0.8987         & 0.7634        & 0.5632        & 0.8575        & 0.6952        & 0.4775        & 0.8959        & 0.7589        & 0.5558 \vspace{1pt}\\
        \toprule
                                &           & \multicolumn{3}{c|}{(S4)} & \multicolumn{3}{c|}{(S5)} & \multicolumn{3}{c}{(S6)} \\
                                &           & $\pi_{\rm th}\!=\!0$ & $\pi_{\rm th}\!=\!0.1$ & $\pi_{\rm th}\!=\!0.2$
                                            & $\pi_{\rm th}\!=\!0$ & $\pi_{\rm th}\!=\!0.1$ & $\pi_{\rm th}\!=\!0.2$
                                            & $\pi_{\rm th}\!=\!0$ & $\pi_{\rm th}\!=\!0.1$ & $\pi_{\rm th}\!=\!0.2$ \\\midrule
        \multirow{4}{*}{(C8)}   & Storey   & $\bs{0.7866}$  & 0.5165        & 0.2594        & 0.7636        & 0.5989        & 0.4263        & $\bs{0.9140}$ & $\bs{0.7367}$ & 0.4780 \\
                                & Quantile & 0.7206         & $\bs{0.6316}$ & $\bs{0.5175}$ & 0.8170        & $\bs{0.7451}$ & $\bs{0.6445}$ & 0.7910        & 0.7140        & $\bs{0.6091}$ \\
                                & Fisher   & 0.7514         & 0.4292        & 0.1704        & 0.8247        & 0.5416        & 0.2606        & 0.8954        & 0.6594        & 0.3585 \\
                                & Sum      & 0.7668         & 0.5674        & 0.3485        & $\bs{0.8786}$ & 0.7295        & 0.5191        & 0.8759        & 0.7245        & 0.5123 \\
        \bottomrule
    \end{tabular}
\end{table*}

\setlength{\tabcolsep}{0.5em} % for the horizontal padding
\begin{table*}[t]
    \centering
    \caption{{\emph{Type I error probability study:} AUC for conformal data contamination tests computed from $2000$ simulations when $\pi = \pi_{\rm th}$.}}
    % \vspace{0.5cm}
    % \scriptsize
    \label{tab:AUC_CDCT_Error}
    \begin{tabular}{@{}ll|rrr|rrr|rrr@{}}
        \toprule
                                &           & \multicolumn{3}{c|}{(S1)} & \multicolumn{3}{c|}{(S2)} & \multicolumn{3}{c}{(S3)} \\
                                &           & $\pi_{\rm th}\!=\!0$ & $\pi_{\rm th}\!=\!0.1$ & $\pi_{\rm th}\!=\!0.2$
                                            & $\pi_{\rm th}\!=\!0$ & $\pi_{\rm th}\!=\!0.1$ & $\pi_{\rm th}\!=\!0.2$
                                            & $\pi_{\rm th}\!=\!0$ & $\pi_{\rm th}\!=\!0.1$ & $\pi_{\rm th}\!=\!0.2$ \\
                                &           & $\pi\!=\!0$ & $\pi\!=\!0.1$ & $\pi\!=\!0.2$
                                            & $\pi\!=\!0$ & $\pi\!=\!0.1$ & $\pi\!=\!0.2$
                                            & $\pi\!=\!0$ & $\pi\!=\!0.1$ & $\pi\!=\!0.2$ \\\midrule
        \multirow{4}{*}{(C8)}   & Storey   & 0.4906 & 0.4165 & 0.3537 & 0.4990 & 0.3358 & 0.2402 & 0.4934 & 0.4069 & 0.3517 \\
                                & Quantile & 0.4796 & 0.4932 & 0.4908 & 0.5035 & 0.5033 & 0.4832 & 0.5150 & 0.4987 & 0.4933 \\
                                & Fisher   & 0.3843 & 0.3112 & 0.2299 & 0.4110 & 0.2537 & 0.1565 & 0.4106 & 0.3033 & 0.2383 \\
                                & Sum      & 0.4184 & 0.4082 & 0.3591 & 0.4519 & 0.3829 & 0.3168 & 0.4545 & 0.4021 & 0.3669 \vspace{1pt}\\
        \toprule
                                &           & \multicolumn{3}{c|}{(S4)} & \multicolumn{3}{c|}{(S5)} & \multicolumn{3}{c}{(S6)} \\
                                &           & $\pi_{\rm th}\!=\!0$ & $\pi_{\rm th}\!=\!0.1$ & $\pi_{\rm th}\!=\!0.2$
                                            & $\pi_{\rm th}\!=\!0$ & $\pi_{\rm th}\!=\!0.1$ & $\pi_{\rm th}\!=\!0.2$
                                            & $\pi_{\rm th}\!=\!0$ & $\pi_{\rm th}\!=\!0.1$ & $\pi_{\rm th}\!=\!0.2$ \\
                                &           & $\pi\!=\!0$ & $\pi\!=\!0.1$ & $\pi\!=\!0.2$
                                            & $\pi\!=\!0$ & $\pi\!=\!0.1$ & $\pi\!=\!0.2$
                                            & $\pi\!=\!0$ & $\pi\!=\!0.1$ & $\pi\!=\!0.2$ \\\midrule
        \multirow{4}{*}{(C8)}   & Storey   & 0.5065 & 0.3045 & 0.1676 & 0.5011 & 0.3524 & 0.2746 & 0.5084 & 0.3888 & 0.3078 \\
                                & Quantile & 0.5068 & 0.4746 & 0.4374 & 0.5043 & 0.5022 & 0.4904 & 0.4996 & 0.5033 & 0.4930 \\
                                & Fisher   & 0.4194 & 0.2168 & 0.1029 & 0.4078 & 0.2427 & 0.1412 & 0.4117 & 0.2900 & 0.2049 \\
                                & Sum      & 0.4570 & 0.3441 & 0.2490 & 0.4492 & 0.3906 & 0.3333 & 0.4460 & 0.3997 & 0.3496 \\
        \bottomrule
    \end{tabular}
\end{table*}

{In Table~\ref{tab:AUC_CDCT_Error}, we display ${\rm AUC}^{\rm cdct}_{(\pi, \pi_{\rm th})}$, for $\pi=\pi_{\rm th}$ thereby indicating the probability of false rejection, again for all the conformal data contamination tests for varying values of $\pi=\pi_{\rm th}$. As the developed theory promises, in all cases the \gls{auc} is less than (or approximately) $0.5$. These results are interesting to see the amount of excess conservativeness in the conformal data contamination tests, noting that an \gls{auc} of (approximately) $0.5$ means there is no excess conservativeness. Generally, the Quantile test shows the least amount of excess conservativeness with \gls{auc} approximately $0.5$ even when $\pi=\pi_{\rm th}=0.2$ for most scenarios. Any excess conservativeness can in these results be attributed to the inequalities in Eqs.~\eqref{eq:pi0sup} and \eqref{eq:fisher_ineq}, in which an outlier with a conformal p-value above $1/({n-\ell}+1)$ contributes to a lack of tightness in the bound.}

\subsection{{Multiple testing with conformal data contamination p-values}}

\setlength{\tabcolsep}{0.5em} % for the horizontal padding
\begin{table*}[t]
    \centering
    \caption{{${\rm TDR}$ computed from $2000$ simulations at significance level $\alpha=0.1$ using Storey's BH procedure with parameter $\gamma=0.5$. Data contamination factors are simulated as $\pi_k\sim {\rm Uniform}([0,1])$. The best test statistic for each data contamination scenario is highlighted with boldface.}}
    % \vspace{0.5cm}
    % \scriptsize
    \label{tab:bigTDR}
    \begin{tabular}{@{}ll|rrr|rrr|rrr@{}}
        \toprule
                                &           & \multicolumn{3}{c|}{(S1)} & \multicolumn{3}{c|}{(S2)} & \multicolumn{3}{c}{(S3)} \\
                                &           & $\pi_{\rm th}\!=\!0$ & $\pi_{\rm th}\!=\!0.1$ & $\pi_{\rm th}\!=\!0.2$
                                            & $\pi_{\rm th}\!=\!0$ & $\pi_{\rm th}\!=\!0.1$ & $\pi_{\rm th}\!=\!0.2$
                                            & $\pi_{\rm th}\!=\!0$ & $\pi_{\rm th}\!=\!0.1$ & $\pi_{\rm th}\!=\!0.2$ \\\midrule
        \multirow{4}{*}{(C8)}   & Storey    & $\bs{0.9128}$ & $\bs{0.8147}$ & 0.7033        & 0.8014        & 0.5995        & 0.4145        & $\bs{0.9122}$ & $\bs{0.8056}$ & 0.6838 \\
                                & Quantile  & 0.8104        & 0.7851        & $\bs{0.7537}$ & 0.7965        & $\bs{0.7548}$ & $\bs{0.6784}$ & 0.8105        & 0.7816        & $\bs{0.7407}$ \\
                                & Fisher    & 0.8890        & 0.7576        & 0.6077        & 0.8076        & 0.5566        & 0.3298        & 0.8903        & 0.7539        & 0.6032 \\
                                & Sum       & 0.8778        & 0.7994        & 0.7142        & $\bs{0.8400}$ & 0.7229        & 0.5821        & 0.8777        & 0.7970        & 0.7046 \vspace{1pt}\\
        \toprule
                                &           & \multicolumn{3}{c|}{(S4)} & \multicolumn{3}{c|}{(S5)} & \multicolumn{3}{c}{(S6)} \\
                                &           & $\pi_{\rm th}\!=\!0$ & $\pi_{\rm th}\!=\!0.1$ & $\pi_{\rm th}\!=\!0.2$
                                            & $\pi_{\rm th}\!=\!0$ & $\pi_{\rm th}\!=\!0.1$ & $\pi_{\rm th}\!=\!0.2$
                                            & $\pi_{\rm th}\!=\!0$ & $\pi_{\rm th}\!=\!0.1$ & $\pi_{\rm th}\!=\!0.2$ \\\midrule
        \multirow{4}{*}{(C8)}   & Storey    & 0.7028        & 0.3953        & 0.1939        & 0.6864        & 0.5398        & 0.4718        & $\bs{0.8784}$ & 0.7345        & 0.5799 \\
                                & Quantile  & 0.6945        & $\bs{0.5850}$ & $\bs{0.4308}$ & 0.8213        & $\bs{0.8025}$ & $\bs{0.7791}$ & 0.8068        & $\bs{0.7770}$ & $\bs{0.7232}$ \\
                                & Fisher    & 0.6867        & 0.3234        & 0.1142        & 0.7725        & 0.4868        & 0.2892        & 0.8671        & 0.6939        & 0.5033 \\
                                & Sum       & $\bs{0.7468}$ & 0.5241        & 0.3013        & $\bs{0.8550}$ & 0.7569        & 0.6297        & 0.8675        & 0.7708        & 0.6634 \\
        \bottomrule
    \end{tabular}
\end{table*}

\setlength{\tabcolsep}{0.5em} % for the horizontal padding
\begin{table*}[t]
    \centering
    \caption{{${\rm FDR}$ computed from $2000$ simulations at significance level $\alpha=0.1$ using the BH procedure. For $K_0 = K/2$ of the data agents $\pi_k = \pi_{\rm th}$ while for the remaining data agents $\pi_k = 1 - \pi_{\rm th}$. Theoretically, the FDR is upper bounded by $\alpha K_0 / K = 0.05$.}}
    % \vspace{0.5cm}
    % \scriptsize
    \label{tab:bigFDR_BH_fixed_balanced}
    \begin{tabular}{@{}ll|rrr|rrr|rrr@{}}
        \toprule
                                &           & \multicolumn{3}{c|}{(S1)} & \multicolumn{3}{c|}{(S2)} & \multicolumn{3}{c}{(S3)} \\
                                &           & $\pi_{\rm th}\!=\!0$ & $\pi_{\rm th}\!=\!0.1$ & $\pi_{\rm th}\!=\!0.2$
                                            & $\pi_{\rm th}\!=\!0$ & $\pi_{\rm th}\!=\!0.1$ & $\pi_{\rm th}\!=\!0.2$
                                            & $\pi_{\rm th}\!=\!0$ & $\pi_{\rm th}\!=\!0.1$ & $\pi_{\rm th}\!=\!0.2$ \\\midrule
        \multirow{4}{*}{(C8)}   & Storey   & 0.0363 & 0.0246 & 0.0251 & 0.0379 & 0.0197 & 0.0115 & 0.0414 & 0.0207 & 0.0234 \\
                                & Quantile & 0.0410 & 0.0391 & 0.0361 & 0.0447 & 0.0401 & 0.0324 & 0.0416 & 0.0403 & 0.0374 \\
                                & Fisher   & 0.0164 & 0.0093 & 0.0067 & 0.0168 & 0.0066 & 0.0029 & 0.0193 & 0.0095 & 0.0052 \\
                                & Sum      & 0.0309 & 0.0228 & 0.0157 & 0.0329 & 0.0171 & 0.0116 & 0.0318 & 0.0230 & 0.0166 \vspace{1pt}\\
        \toprule
                                &           & \multicolumn{3}{c|}{(S4)} & \multicolumn{3}{c|}{(S5)} & \multicolumn{3}{c}{(S6)} \\
                                &           & $\pi_{\rm th}\!=\!0$ & $\pi_{\rm th}\!=\!0.1$ & $\pi_{\rm th}\!=\!0.2$
                                            & $\pi_{\rm th}\!=\!0$ & $\pi_{\rm th}\!=\!0.1$ & $\pi_{\rm th}\!=\!0.2$
                                            & $\pi_{\rm th}\!=\!0$ & $\pi_{\rm th}\!=\!0.1$ & $\pi_{\rm th}\!=\!0.2$ \\\midrule
        \multirow{4}{*}{(C8)}   & Storey   & 0.0416 & 0.0141 & 0.0050 & 0.0390 & 0.0287 & 0.0259 & 0.0418 & 0.0222 & 0.0145 \\
                                & Quantile & 0.0447 & 0.0309 & 0.0184 & 0.0454 & 0.0424 & 0.0369 & 0.0472 & 0.0422 & 0.0317 \\
                                & Fisher   & 0.0182 & 0.0026 & 0.0008 & 0.0170 & 0.0064 & 0.0022 & 0.0179 & 0.0081 & 0.0022 \\
                                & Sum      & 0.0332 & 0.0141 & 0.0055 & 0.0313 & 0.0206 & 0.0118 & 0.0344 & 0.0225 & 0.0105 \\
        \bottomrule
    \end{tabular}
\end{table*}

\setlength{\tabcolsep}{0.5em} % for the horizontal padding
\begin{table*}[ht]
    \centering
    \caption{{${\rm FDR}$ computed from $2000$ simulations at significance level $\alpha=0.1$ using Storey's BH procedure with parameter $\gamma=0.5$. For $K_0 = K/2$ of the data agents $\pi_k = \pi_{\rm th}$ while for the remaining data agents $\pi_k = 1 - \pi_{\rm th}$.}}
    % \vspace{0.5cm}
    % \scriptsize
    \label{tab:bigFDR_fixed_balanced}
    \begin{tabular}{@{}ll|rrr|rrr|rrr@{}}
        \toprule
                                &           & \multicolumn{3}{c|}{(S1)} & \multicolumn{3}{c|}{(S2)} & \multicolumn{3}{c}{(S3)} \\
                                &           & $\pi_{\rm th}\!=\!0$ & $\pi_{\rm th}\!=\!0.1$ & $\pi_{\rm th}\!=\!0.2$
                                            & $\pi_{\rm th}\!=\!0$ & $\pi_{\rm th}\!=\!0.1$ & $\pi_{\rm th}\!=\!0.2$
                                            & $\pi_{\rm th}\!=\!0$ & $\pi_{\rm th}\!=\!0.1$ & $\pi_{\rm th}\!=\!0.2$ \\\midrule
        \multirow{4}{*}{(C8)}   & Storey   & 0.1064 & 0.0708 & 0.0489 & 0.1105 & 0.0513 & 0.0230 & 0.1159 & 0.0639 & 0.0461 \\
                                & Quantile & 0.1120 & 0.1074 & 0.1050 & 0.1198 & 0.1051 & 0.0984 & 0.1128 & 0.1089 & 0.0975 \\
                                & Fisher   & 0.0641 & 0.0281 & 0.0153 & 0.0643 & 0.0191 & 0.0074 & 0.0653 & 0.0293 & 0.0164 \\
                                & Sum      & 0.0890 & 0.0658 & 0.0458 & 0.0965 & 0.0589 & 0.0347 & 0.0940 & 0.0632 & 0.0483 \vspace{1pt}\\
        \toprule
                                &           & \multicolumn{3}{c|}{(S4)} & \multicolumn{3}{c|}{(S5)} & \multicolumn{3}{c}{(S6)} \\
                                &           & $\pi_{\rm th}\!=\!0$ & $\pi_{\rm th}\!=\!0.1$ & $\pi_{\rm th}\!=\!0.2$
                                            & $\pi_{\rm th}\!=\!0$ & $\pi_{\rm th}\!=\!0.1$ & $\pi_{\rm th}\!=\!0.2$
                                            & $\pi_{\rm th}\!=\!0$ & $\pi_{\rm th}\!=\!0.1$ & $\pi_{\rm th}\!=\!0.2$ \\\midrule
        \multirow{4}{*}{(C8)}   & Storey   & 0.1151 & 0.0341 & 0.0100 & 0.1143 & 0.0736 & 0.0575 & 0.1093 & 0.0625 & 0.0294 \\
                                & Quantile & 0.1169 & 0.0932 & 0.0641 & 0.1223 & 0.1094 & 0.1065 & 0.1207 & 0.1095 & 0.0969 \\
                                & Fisher   & 0.0648 & 0.0089 & 0.0018 & 0.0634 & 0.0196 & 0.0059 & 0.0654 & 0.0283 & 0.0067 \\
                                & Sum      & 0.0967 & 0.0424 & 0.0158 & 0.0941 & 0.0632 & 0.0392 & 0.0947 & 0.0694 & 0.0374 \\
        \bottomrule
    \end{tabular}
\end{table*}

{Now, we will investigate the \gls{tdr} and \gls{fdr} when doing multiple conformal data contamination tests. The \gls{tdr} when using Storey's \gls{bh} procedure at significance level $\alpha=0.1$ and for Storey's hyperparameter $\gamma=0.5$ is shown in Table~\ref{tab:bigTDR}. Mostly, the results follow that observed in Table~\ref{tab:AUC_CDCT_Power_v2} for the marginal power, i.e., the Storey test statistic tends to be the most powerful, but for some cases for larger data contamination thresholds, the Quantile test statistic is preferable. This tells us that the marginal power of the test is a good indication of the power of multiple testing.}
% {In Table~\ref{tab:bigFDR}, we show the \gls{fdr}, again when using Storey's \gls{bh} procedure at level $\alpha=0.1$ and for Storey's hyperparameter $\gamma=0.5$. We observe here that the \gls{fdr} is always substantially below the specified level $\alpha=0.1$, however, much of this excess conservativeness can be attributed to having test points in the interior of the null hypotheses, i.e., $\pi_k < \pi_{\rm th}$ (as opposed to $\pi_k = \pi_{\rm th}$).}
% \input{figures/tables/bigFDR_table_discrete_uniform}

To finally validate Theorem~\ref{thm:Storey_PRDS} on these real datasets, we show in Table~\ref{tab:bigFDR_BH_fixed_balanced} the \gls{fdr} when using the \gls{bh} procedure at level $\alpha=0.1$ when at the boundary of the null hypothesis. Specifically, we let $K_0=K/2$ of the data agents have $\pi_k=\pi_{\rm th}$ and the others have $\pi_k = 1-\pi_{\rm th}$ (this maintains a balance of approximately half the data points being outliers). Hence, here we do not simulate $\pi_k \sim {\rm Uniform}([0,1])$ as we do not allow any $\pi_k < \pi_{\rm th}$ which would inherently cause excess conservativeness. By Theorem~\ref{thm:Storey_PRDS}, the \gls{fdr} should in this case be upper bounded by $\alpha K_0/K = 0.05$, which is also observed in Table~\ref{tab:bigFDR_BH_fixed_balanced} for all tested cases. However, we can note that some excess conservativeness arises when applying the \gls{bh} procedure to the conformal data contamination p-values.
%, which is particularly apparent when $\pi_{\rm th} = 0$ for the Storey and Quantile tests (here the multiplicity correction is the only reason for conservativeness).
Once again, the Quantile test is generally the least conservative and in some of the tested scenarios maintain an \gls{fdr} close to the $0.05$ level even for $\pi_{\rm th} = 0.2$. On the other hand, the Fisher test tends to be the most conservative out of the four proposed conformal data contamination tests.

Finally, we consider, in Table~\ref{tab:bigFDR_fixed_balanced}, the \gls{fdr} when using Storey's \gls{bh} procedure.
% To get an idea of the conservativeness of the tests when at the boundary of the null hypotheses we consider in Table~\ref{tab:bigFDR_fixed_balanced} the \gls{fdr} when using Storey's \gls{bh} procedure.
Here, you see that the tests are close to the nominal level $\alpha=0.1$, even for $\pi_{\rm th} = 0.2$, especially the Quantile tests. Some conservativeness is induced from the multiplicity correction, as was reported above, however, using the adaptive procedure of Storey's \gls{bh} tends to do the opposite. Recall that we have only shown \gls{fdr} control with the \gls{bh} procedure in Theorem~\ref{thm:Storey_PRDS}, not with Storey's \gls{bh} procedure. It will be an interesting line of further work to determine if any adaptive \gls{bh} procedure can maintain \gls{fdr} control, and also to find a non-trivial lower bound on the \gls{fdr} when applying the \gls{bh} procedure to the conformal data contamination p-values.

\subsection{Collaborative data sharing}\label{subsecSupp:data_sharing}
In this section, we numerically evaluate the performance of the proposed data sharing procedure in terms of classification accuracy. This is done for data contamination scenarios (S1)-(S8) with the class-wise score based on \gls{pu} classification using \gls{lr}, i.e.~(C8), and to solve the classification task we use the well-studied, simple, and robust method in \gls{svc}.

\begin{figure}[t]
    \centering
    \begin{minipage}{0.32\linewidth}
        \centering
        \subfloat[Scenario (S1). Accuracy without data sharing: $47.77~\%$.]{\includegraphics[width=1\linewidth, page=17]{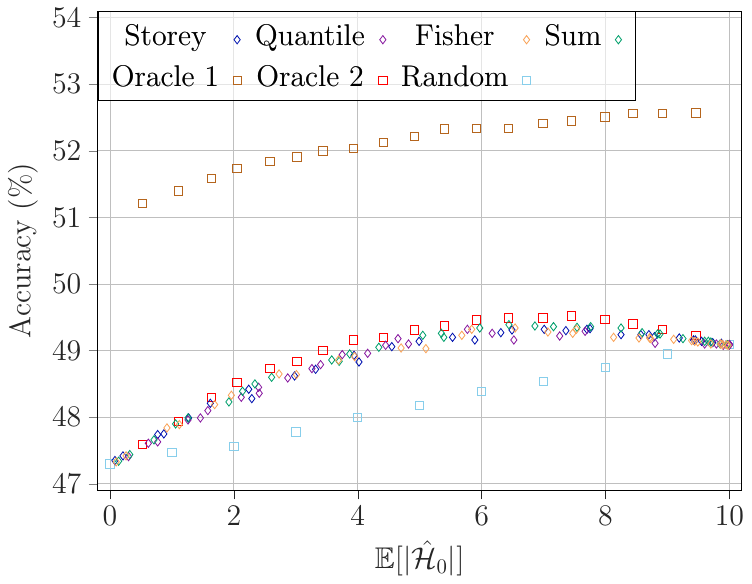}\label{subfigSupp:retina_labelAdaDetect_LR_SVC}}
    \end{minipage}~
    \begin{minipage}{0.32\linewidth}
        \centering
        \subfloat[{Scenario (S2). Accuracy without data sharing: $59.50~\%$.}]{\includegraphics[width=1\linewidth, page=18]{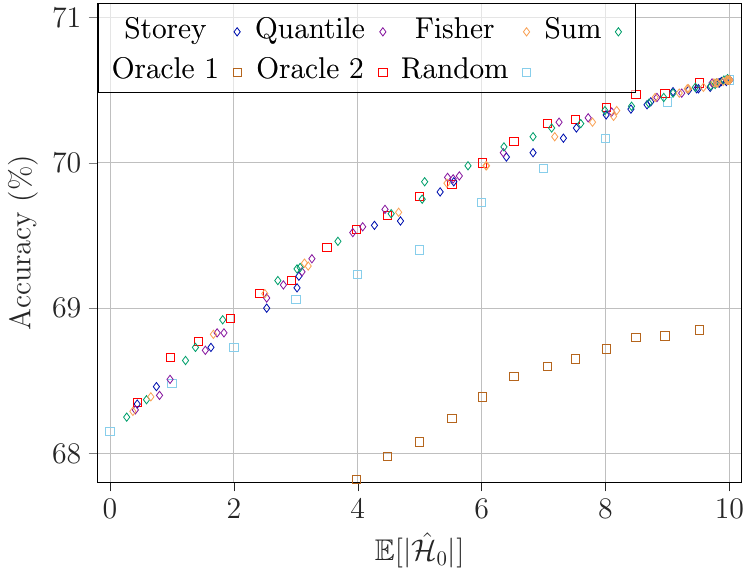}\label{subfigSupp:retinalOCT_partial_labelAdaDetect_LR_SVC}}
    \end{minipage}~
    \begin{minipage}{0.32\linewidth}
        \centering
        \subfloat[{Scenario (S3). Accuracy without data sharing: $69.47~\%$.}]{\includegraphics[width=1\linewidth, page=19]{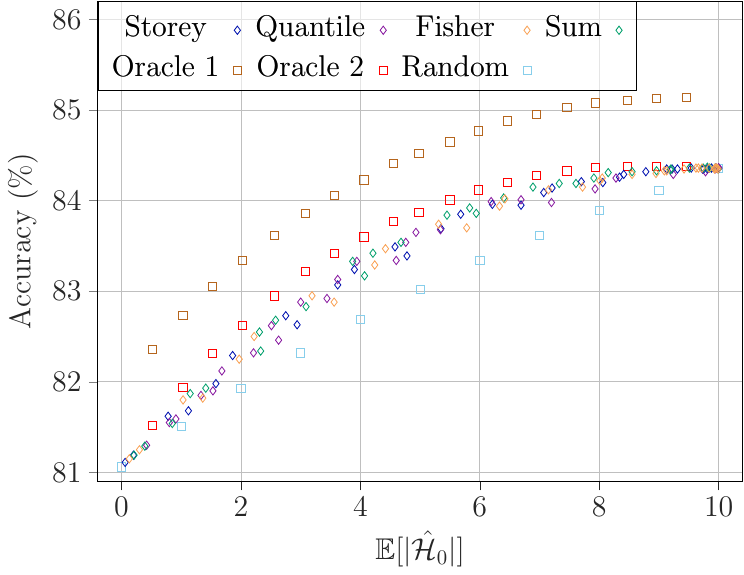}\label{subfigSupp:WBC_partial_labelAdaDetect_LR_SVC}}
    \end{minipage}\\
    \begin{minipage}{0.32\linewidth}
        \centering
        \subfloat[{Scenario (S4). Accuracy without data sharing: $77.76~\%$.}]{\includegraphics[width=1\linewidth, page=20]{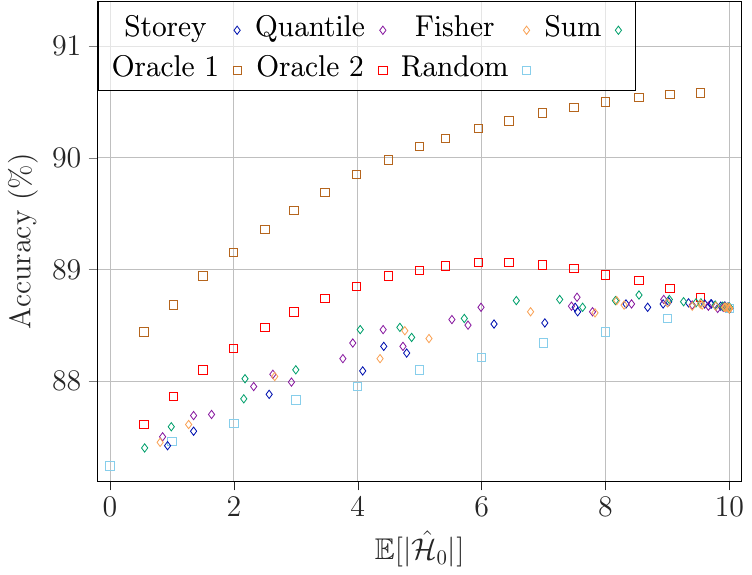}\label{subfigSupp:FEMNIST_partial_labelAdaDetect_LR_SVC}}
    \end{minipage}~
    \begin{minipage}{0.32\linewidth}
        \centering
        \subfloat[Scenario (S5). Accuracy without data sharing: $90.75~\%$.]{\includegraphics[width=1\linewidth, page=21]{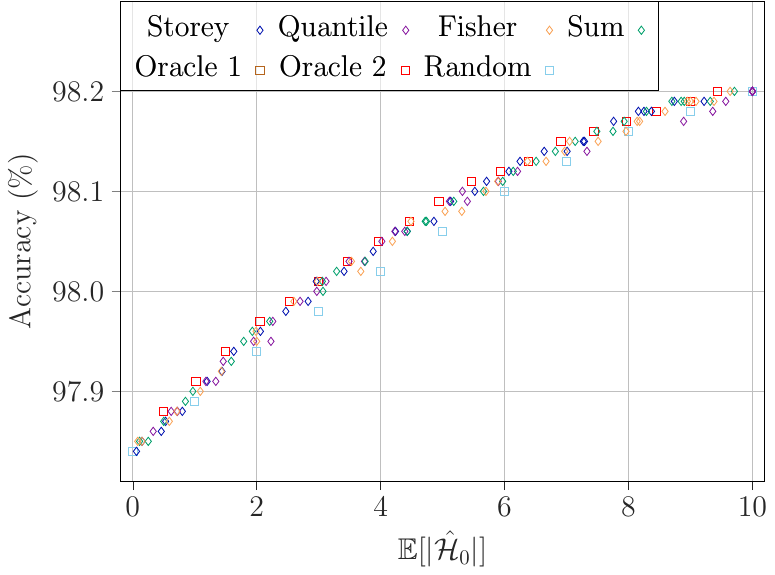}\label{subfigSupp:MNIST_partial_fn_labelAdaDetect_LR_SVC}}
    \end{minipage}~
    \begin{minipage}{0.32\linewidth}
        \centering
        \subfloat[Scenario (S6). Accuracy without data sharing: $90.75~\%$.]{\includegraphics[width=1\linewidth, page=22]{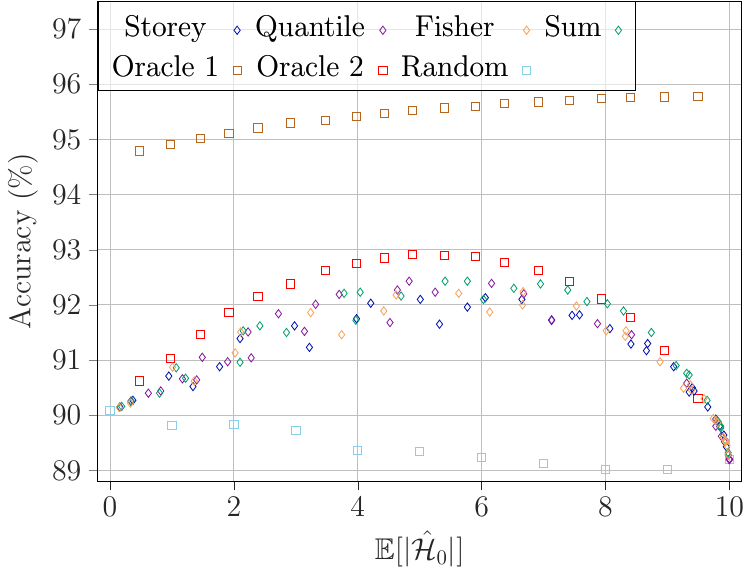}\label{subfigSupp:MNIST_partial_ln_labelAdaDetect_LR_SVC}}
    \end{minipage}\\
    \begin{minipage}{0.32\linewidth}
        \centering
        \subfloat[{Scenario (S7). Accuracy without data sharing: $46.90~\%$.}]{\includegraphics[width=1\linewidth, page=29]{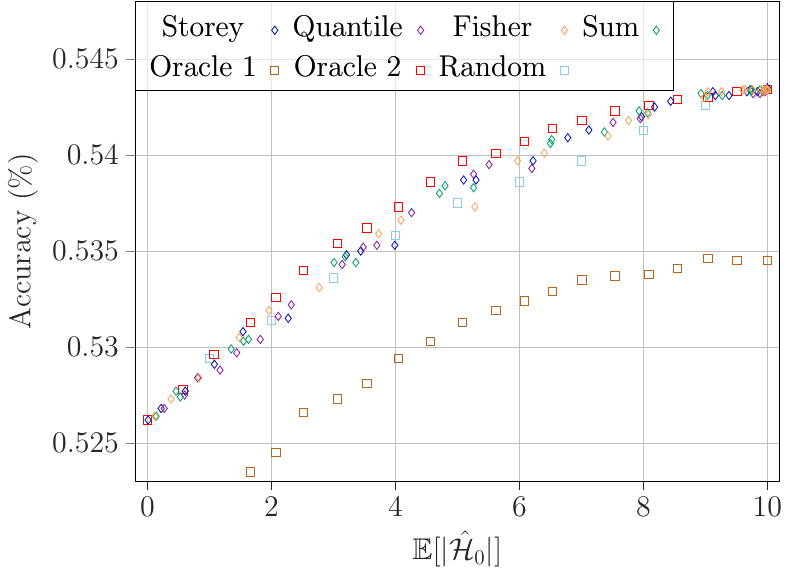}\label{subfigSupp:cifar10_brightness_AdaDetect_PCALR_SVC}}
    \end{minipage}\hspace{20pt}
    \begin{minipage}{0.32\linewidth}
        \centering
        \subfloat[{Scenario (S8). Accuracy without data sharing: $46.90~\%$.}]{\includegraphics[width=1\linewidth, page=31]{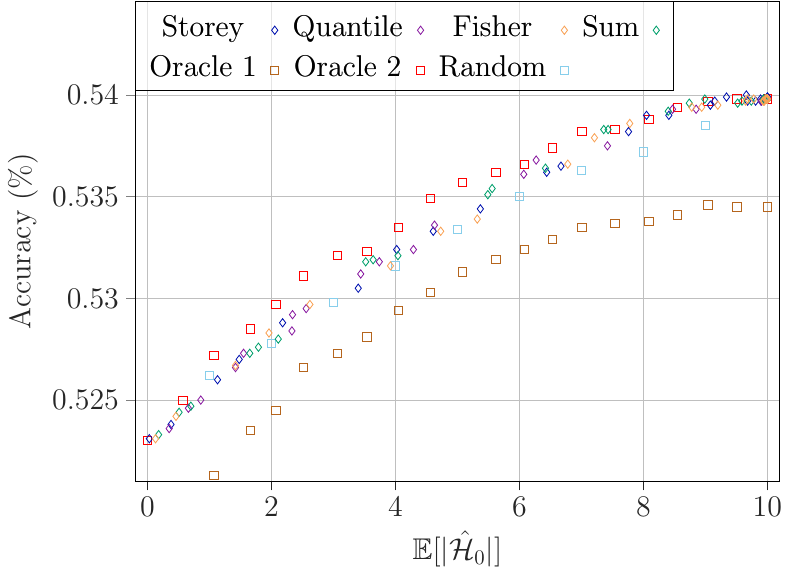}\label{subfigSupp:cifar10_snow_AdaDetect_PCALR_SVC}}
    \end{minipage}
    \caption{Plot of the model accuracy against the average number of collaborating data agents.}
    \label{figSupp:Accuracy}
\end{figure}

We show in Figure~\ref{figSupp:Accuracy} the relation between the average number of collaborating data agents, $\mathbb{E}[|\hat{\mathcal{H}}_0|]$, and the average classification accuracy across $500$ data simulations for each of the {eight} data contamination scenarios, with various hyperparameter settings $(\pi_{\rm th}, \alpha) \in \{0, 0.1, 0.2, 0.3, 0.4, 0.5, 0.6, 0.7\} \times \{0.05, 0.2, 0.5, 0.7\}$. No data subset selection is used meaning that all acquired datapoints are used for model fitting, since the focus of this work is on the proposed conformal data contamination tests.

We observe that in scenarios (S1), (S3), (S4), and (S6) the oracles and the proposed methods outperform the \emph{random} baseline meaning that in those scenarios a substantial boost in accuracy can be achieved by carefully selecting data agents for collaboration. Also, in these scenarios \emph{Oracle 1} is clearly outperforming \emph{Oracle 2} showing that there is further room for improvement by incorporating data subset selection. Meanwhile, in scenarios {(S2), (S5), (S7), and (S8), i.e., the scenarios with optical coherence tomography images, feature noise on MNIST, and feature noise on CIFAR-10}, it appears that the contaminated data is (nearly) as useful for subsequent classification model training as the inliers. In all the considered scenarios, the proposed data sharing procedures outperform the case of no data sharing, emphasizing the value of collaboration. Moreover, the proposed data sharing procedures reaches the performance of \emph{oracle 2} in most cases, showing the efficiency of the proposed methodology. A notable exception is scenario (S4), i.e., the FEMNIST data contamination scenario, and we also observed in Table~\ref{tab:bigAUC} {that outlier detection was difficult in this scenario}, hence, showing that in this case it would be beneficial to further improve the conformal score.

These results are limited to small scale problems with small dataset sizes, few collaborators, and few image classes. In the following, we study numerically problems of larger scale.

\subsection{{Scalability to larger problems}}\label{subsecSupp:scalability}
{Here, we explore the performance of the data acquisition protocol for problems of larger scale. We consider three scenarios: full MNIST with label noise, FEMNIST with lower- and uppercase handwritten letters \emph{a} through \emph{j}, and full CIFAR-10 mixed with CIFAR-10C contaminated with brightness. All three scenarios have $10$ classes, and we set $\ell = 300$, $n-\ell = 250$, and $m=50$. Further, for MNIST and FEMNIST we set the number of data agents to be $K=100$, while for the CIFAR-10 scenario $K=50$ due to limited amount of total data. As previously, we use \gls{svc} as the classifier, $f^*$. For the MNIST and FEMNIST data we use the conformal score (C8) and for CIFAR-10 we use conformal score (C11). Results are only shown for the proposed methods with hyperparameters $(\pi_{\rm th}, \alpha) \in {\{0, 0.1, 0.2, 0.3, 0.4, 0.5, 0.6, 0.7\} \times \{0.01, 0.05, 0.2, 0.5, 0.7\}}$ when averaged across {$20$} scenario simulations for MNIST and FEMNIST scenarios, and {$100$} for the CIFAR-10 scenario. Note that the variability in accuracy across data simulations is reduced in large scale problems.}

The results are shows in Figure~\ref{fig:AccuracyHuge}. {In the MNIST and FEMNIST} scenarios, substantial improvements in accuracy is achieved as compared to the \emph{random} baseline, and in most cases the proposed methods are capable of (nearly) reaching the performance of \emph{oracle 2}, as in the small scale problems previously considered. Generally, the performance gains of the proposed protocols, as compared to the \emph{random} baseline, tends to be more pronounced here than in the small scale problems. This verifies that the proposed data acquisition protocol scales well to problems of larger scale which is an essential property for practical deployment. {The accuracy improvements over the \emph{random} baseline is lesser in the CIFAR-10 scenario as we encounter a case where out-of-distribution data still proves helpful to improve learning.}

\begin{figure}[t]
    \centering
    \begin{minipage}{0.32\linewidth}
        \centering
        \subfloat[{FEMNIST. Accuracy without data sharing: {$65.80~\%$}.}]{\includegraphics[width=1\linewidth, page=34]{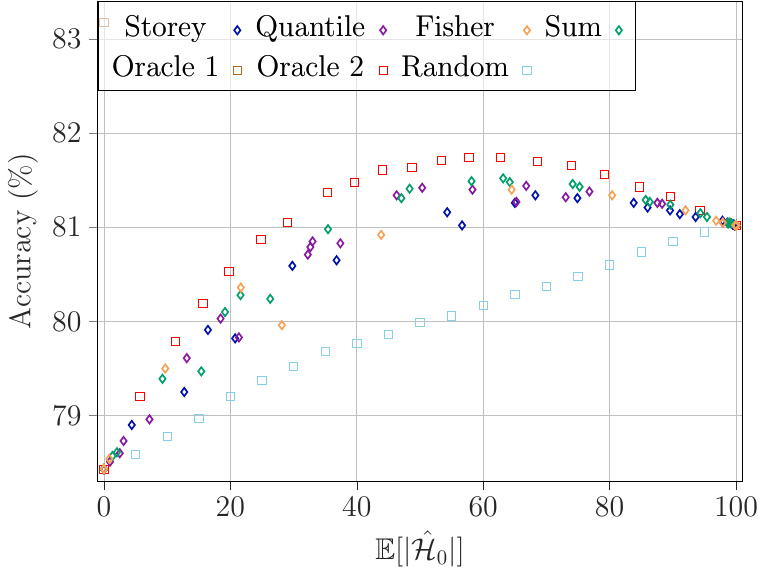}\label{subfigSupp:FEMNIST_partial_HUGE_labelAdaDetect_LR_SVC}}
    \end{minipage}~
    \begin{minipage}{0.32\linewidth}
        \centering
        \subfloat[{MNIST label noise. Accuracy without data sharing: {$84.05~\%$}.}]{\includegraphics[width=1\linewidth, page=33]{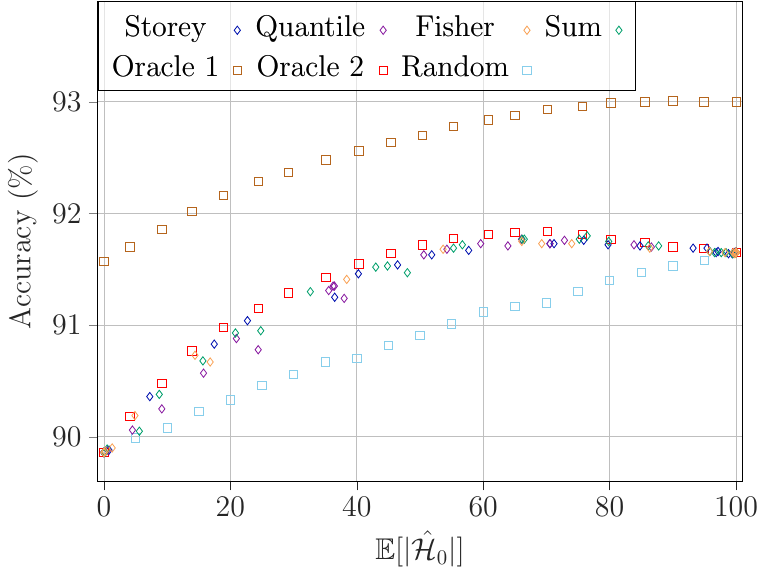}\label{subfigSupp:MNIST_partial_ln_HUGE_labelAdaDetect_LR_SVC}}
    \end{minipage}~
    \begin{minipage}{0.32\linewidth}
        \centering
        \subfloat[{CIFAR-10C brightness. Accuracy without data sharing: {$32.84~\%$}.}]{\includegraphics[width=1\linewidth, page=37]{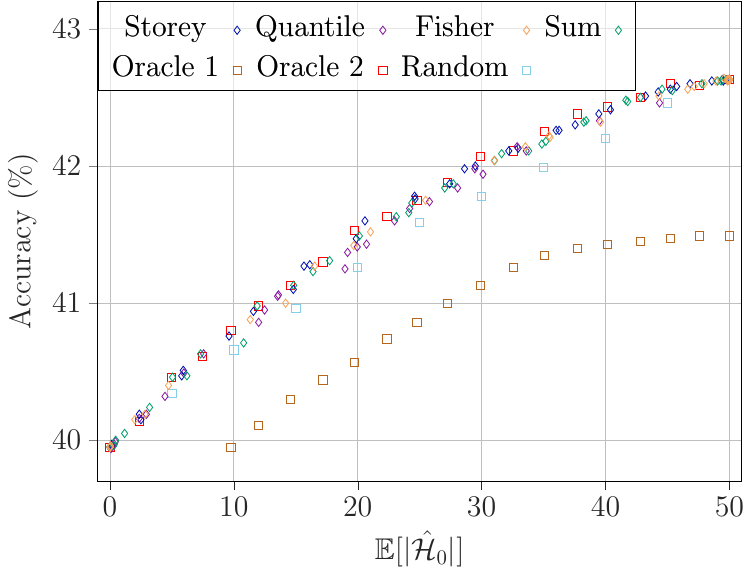}\label{subfigSupp:cifar10_brightness_HUGE_AdaDetect_PCALR_SVC}}
    \end{minipage}
    \caption{{Plot of the model accuracy against the average number of collaborating data agents.}}
    \label{fig:AccuracyHuge}
\end{figure}

\subsection{Data-driven hyperparameter selection}\label{sec:hyperparameter_selection}
For the proposed collaborative data sharing procedure there are a number of hyperparameters, see Section~\ref{sec:additional_details} for an overview. In this section, we will consider a data-driven approach to selecting key hyperparameters. We observed in Figure~\ref{figSupp:Accuracy} for some of the data contamination scenarios a concave relation between the accuracy and $\mathbb{E}[|\hat{\mathcal{H}}_0|]$, and so in practice we may be interested in selecting the optimal hyperparameters in terms of accuracy.

The data-driven hyperparameter selection we propose here can be summarized as: (i) order the $K$ data agents in terms of contamination and determine $\hat{\mathcal{H}}_0$; (ii) fit the classification model on the data $\{Z_i\}_{i=1}^{\ell} \bigcup \{Z_i^k : i\in[m], k \in \hat{\mathcal{H}}_0\}$; (iii) evaluate the accuracy on the data $\{Z_i\}_{i=\ell+1}^n$ denoted $L_{\rm val}(\pi_{\rm th}, \alpha)$ noting that $\hat{\mathcal{H}}_0$ depends on the hyperparameters $\pi_{\rm th}$ and $\alpha$; (iv) set the hyperparameter as $(\hat{\pi}_{\rm th}, \hat{\alpha}) = \argmax_{(\pi_{\rm th}, \alpha) \in G} L_{\rm val}(\pi_{\rm th}, \alpha)$ where $G$ is a grid of hyperparameter values.

We compare with the \emph{random} and \emph{oracle 2} baselines: for the \emph{random} baseline the hyperparameter is $K_{\rm budget}$, and for baseline \emph{oracle 2} the hyperparameter is $\pi_{\rm th}$.

In Table~\ref{tab:ProposedAccuracyCV}, we show the mean and standard deviation of the accuracy $(\%)$ across $500$ data simulations when using the data-driven hyperparameter selection. In all cases, we consider the settings as described in Section~\ref{sec:real_datasets} and use the conformal score (C9). First, we notice that for scenarios (S2)-(S5) all the methods are close in terms of accuracy, since in these scenarios the general rule is that more data is almost always better, as has also been noted above. On the other hand, for scenarios (S1) and (S6) notable gains in accuracy can be achieved with the proposed method.
We further observe that we are not quite able to reach the peak accuracies observed in Figure~\ref{fig:AccuracyMain} with the proposed methods, however, it is quite close. For this reason, we deem that the data-driven hyperparameter selection paves the way for a practical deployment of the procedure, and note that the optimization of other hyperparameters could also be done with the same approach. Notice also that for the \emph{random} baseline, the mean accuracy surpasses the highest accuracy observed in Figure~\ref{fig:AccuracyMain}, due to the implicit dependency on the contamination factors in choosing $K_{\rm budget}$.

This data-driven hyperparameter selection still lacks attention to the cost of collaboration. For instance, in scenarios (S2)-(S5), where more data is almost always better, the optimal choice is often to collaborate with all the data agents. In a specific application, the loss function could be defined to take the cost of collaboration into account.

% \setlength{\tabcolsep}{0.5em} % for the horizontal padding
% \begin{table*}
%     \centering
%     \caption{Mean (standard deviation) of accuracy $(\%)$ across $500$ data simulations with data-driven hyperparameter selection.}
%     \vspace{0.5cm}
%     \scriptsize
%     \label{tab:ProposedAccuracyCV}
%     \begin{tabular}{@{}ll|llllll@{}}
%         \toprule
%          & {Oracle 2} & No sharing & Random & Storey & Quantile & Fisher & Sum \\ \midrule
%  (S1) & 48.674 (2.909) & 47.774 (2.899) & 47.957 (2.806) & $\mathbf{48.578~(2.968)}$ & 48.863 (2.896) & 48.815 (2.905) & 48.249 (2.887) \\
%  (S2) & 69.574 (3.502) & 59.502 (5.419) & $\mathbf{69.547~(3.494)}$ & 69.276 (3.669) & 69.158 (3.714) & 69.245 (3.751) & 68.670 (3.743) \\
%  (S3) & 83.912 (1.664) & 69.474 (4.390) & $\mathbf{83.724~(1.783)}$ & 82.689 (2.497) & 82.940 (2.355) & 82.855 (2.435) & 82.108 (2.295) \\
%  (S4) & 88.716 (1.432) & 77.658 (3.114) & 88.085 (1.475) & $\mathbf{88.380~(1.550)}$ & 88.352 (1.583) & 88.378 (1.633) & 88.348 (1.662) \\
%  (S5) & 98.037 (0.285) & 90.75 (2.03) & 97.978 (0.292) & 98.005 (0.318) & 98.019 (0.292) & $\mathbf{98.027~(0.293)}$ & 97.941 (0.297) \\
%  (S6) & 92.399 (4.009) & 90.753 (2.025) & 90.856 (5.094) & 91.706 (4.582) & 91.862 (4.670) & $\mathbf{91.935~(4.357)}$ & 91.567 (4.868) \\
%         \bottomrule
%     \end{tabular}
% \end{table*}

\setlength{\tabcolsep}{0.5em} % for the horizontal padding
\begin{table*}
    \centering
    \caption{Mean (standard deviation) of accuracy $(\%)$ across $500$ data simulations with data-driven hyperparameter selection.}
    \vspace{5pt}
    % \scriptsize
    \label{tab:ProposedAccuracyCV}
    \begin{tabular}{@{}ll|llllll@{}}
        \toprule
         & Oracle 2 & No sharing & Random & Storey & Quantile & Fisher & Sum \\ \midrule
 (S1) & 48.64 (3.00) & 47.77 (2.90) & 47.96 (2.98) & 48.55 (3.02) & 48.77 (3.07) & 48.56 (3.00) & 48.25 (3.06) \\
 (S2) & 69.49 (3.22) & 59.50 (5.42) & 69.35 (3.27) & 69.48 (3.34) & 69.20 (3.33) & 69.36 (3.27) & 69.00 (3.34) \\
 (S3) & 83.96 (1.59) & 69.47 (4.39) & 83.75 (1.73) & 83.03 (2.37) & 83.08 (2.33) & 83.09 (2.34) & 82.42 (2.27) \\
 (S4) & 88.78 (1.47) & 77.66 (3.11) & 88.22 (1.53) & 88.50 (1.56) & 88.44 (1.62) & 88.46 (1.62) & 88.42 (1.65) \\
 (S5) & 98.02 (0.28) & 90.75 (2.03) & 97.97 (0.29) & 97.99 (0.30) & 98.01 (0.28) & 98.01 (0.30) & 97.91 (0.31) \\
 (S6) & 92.40 (4.01) & 90.75 (2.03) & 90.86 (5.09) & 91.71 (4.58) & 91.86 (4.67) & 91.93 (4.36) & 91.57 (4.87) \\
        \bottomrule
    \end{tabular}
\end{table*}

\section{{Implementation details of conformal scores}}\label{sec:score_details}
{Here we provide additional implementation details for the conformal score functions considered in Section~\ref{sec:real_datasets}.}

{\textbf{Scores (C1)-(C4):} The scores (C1)-(C4) are based on the one-class classification methods, one-class support vector machines and isolation forest, and uses the standard implementation in \texttt{scikit-learn}. For scores (C1) and (C3), the one-class classifier is fitted to the data $X_1,\dots,X_\ell$. The calibration and test scores are then evaluated using the scoring function of the fitted model. For the class-wise scores (C2) and (C4), $|\mathcal{Y}|$ one-class classifiers are fitted, one to each dataset $X^{Y=i}_{\rm train}$ for $i\in\mathcal{Y}$, and again calibration and test scores are evaluated using the score function.}

{\textbf{Scores (C5)-(C6):} The scores (C5) and (C6) are also one-class classification methods, but autoencoder based. The autoencoder consists of an encoder neural network and a decoder neural network, both having a single layer. The dimension of the latent is $64$ with a \texttt{ReLu} activation in the hidden layer, and a \texttt{sigmoid} activation in the output layer. The autoencoder is trained using $100$ epochs with the standard implementation of the \texttt{adam} algorithm in \texttt{Keras}, employing the mean squared error loss. The calibration and test scores are evaluated as the reconstruction error using the fitted autoencoder. For the score (C5), a single autoencoder is trained on the data $X_1,\dots,X_\ell$, while for the class-wise score (C6), $|\mathcal{Y}|$ autoencoders are fitted, one to each dataset $X^{Y=i}_{\rm train}$ for $i\in\mathcal{Y}$.}

{\textbf{Scores (C7)-(C14):} The scores (C7)-(C14) are based on positive-unlabeled classification \citep{Marandon2024:Adaptive} using either logistic regression (C7)-(C8), support vector machines (C9)-(C10), logistic regression on principal components explaining at least $80~\%$ of the variance, or a residual network. Logistic regression and support vector machines implementation uses the standard in \texttt{scikit-learn} with balanced class weight. Dimensionality reduction to principal components is also implemented using the standard in \texttt{scikit-learn}, and is fitted to the training data, i.e., $X_1,\dots,X_\ell$.}

{The residual network consists of two convolutional layers, followed by a residual block with two convolutional layers, and finally a fully connected output layer. The first convolutional layer has $32$ output channels, uses a kernel of size $3\times3$, stride of $2$, and padding of $1$, followed by batch normalization with $0.9$ momentum and \texttt{ReLu} activation. The second convolutional layer has $64$ output channels, uses a kernel of size $3\times3$, stride of $2$, and padding of $1$, again followed by batch normalization, \texttt{ReLu} activation, and max pooling with kernel size of $2\times2$ and a stride of $2$. The residual block uses a kernel of size $3\times 3$ and no stride. This is followed by average pooling into a single scalar for each of the $64$ channels. Finally, the output layer maps with a fully connected layer to a $2$-dimensional output with a \texttt{sigmoid} activation. We use the cross entropy loss and $L_2$ weight decay with parameter $0.0001$. The model is optimized using the \texttt{adam} algorithm with learning rate of $0.0001$. The batch size is set to $32$ and we use $10$ training epochs. The residual network is implemented in \texttt{PyTorch}.}

{For scores (C7), (C9), (C11), and (C13), the binary classifier is fitted to the dataset
\begin{equation*}
    \{(X_j, 1)\}_{j=1}^{\ell}\bigcup \{(X_j, 0)\}_{j=\ell+1}^{n}\bigcup \{(X_j^k, 0) : j\in[m_{r0}],k\in[K]\}.
\end{equation*}
The calibration and test scores are then evaluated as the estimated probability of belonging to the positive class $1$. For the class-wise scores (C8), (C10), (C12), and (C14), $|\mathcal{Y}|$ binary classifiers are fitted, one to each dataset
\begin{equation*}
    \{(X_j, 1):Y_j=i,j\in[\ell]\}\bigcup \{(X_j, 0):Y_j=i,j\in\ell+[n-\ell]\}\bigcup \{(X_j^k, 0) : Y_j^k=i,j\in[m_{r0}],k\in[K]\}.
\end{equation*}
for $i\in[|\mathcal{Y}|]$. Again calibration and test scores are evaluated as the estimated probability of belonging to the positive class $1$.}

% \section{Use of Large Language Models}
% Large language models have been used for the purpose of text editing.
%%%%%%%%%%%%%%%%%%%%%%%%%%%%%%%%%%%%%%%%%%%%%%%%%%%%%%%%%%%%%%%%%%%%%%%%%%%%%%%

% \putbib[bib]
% \end{bibunit}
\bibliography{bib.bib}

\end{document}